\documentclass[letterpaper,journal]{IEEEtran}
\usepackage{amsmath,amsfonts}
\usepackage{algorithmic}
\usepackage{algorithm}
\usepackage{array}
\usepackage[caption=false,font=normalsize,labelfont=sf,textfont=sf]{subfig}
\usepackage{textcomp}
\usepackage{url}
\usepackage{booktabs}
\usepackage{amsfonts}
\usepackage{nicefrac}
\usepackage{microtype}
\usepackage{graphicx}
\usepackage{color}
\usepackage{colortbl}
\usepackage{dsfont}
\usepackage{amsmath,amssymb}
\usepackage{multirow}
\usepackage{mathtools}
\usepackage{amsthm}
\newtheorem{theorem}{Theorem}
\newtheorem{definition}{Definition}

\newtheorem{assumption}{Assumption}
\newtheorem{lemma}{Lemma}

\usepackage{stfloats}
\usepackage{placeins}
\usepackage{verbatim}
\usepackage{cite}
\usepackage{soul}
\usepackage{xcolor}
\usepackage[colorlinks=true,linkcolor=blue,citecolor=blue,urlcolor=blue]{hyperref}
\usepackage{cleveref}
\usepackage[skins]{tcolorbox}
\usepackage{orcidlink}

\definecolor{algbg}{gray}{0.95}
\newtcolorbox[auto counter, Crefname={Algorithm}{Algorithms}, crefname={Algorithm}{Algorithms}]{styledalg}[2][]{
  enhanced,
  sharp corners,
  colback=algbg,
  colframe=black,
  boxrule=0.4pt,
  top=8pt, bottom=6pt, left=6pt, right=6pt,
  fonttitle=\bfseries\small,
  coltitle=black,
  colbacktitle=white,
  attach boxed title to top left={xshift=4pt, yshift=-3mm},
  boxed title style={sharp corners, boxrule=0.4pt, size=small},
  title={Algorithm~\thetcbcounter: #2},
  #1
}
\definecolor{inc}{RGB}{0, 150, 0}
\definecolor{dec}{RGB}{200, 0, 0}
\definecolor{demogray}{gray}{0.9}

\newcommand{\cinc}[1]{\textcolor{inc}{\scriptsize{+#1}}} 
\newcommand{\cdec}[1]{\textcolor{dec}{\scriptsize{-#1}}}

\begin{document}

\title{Generalized Category Discovery under Domain Shifts: From Vision to Vision-Language Models}

\author{Hongjun Wang~\orcidlink{0000-0001-5269-5471}, Po Hu~\orcidlink{0009-0004-4390-0620}, Kai Han~\orcidlink{0000-0002-7995-9999}
\IEEEcompsocitemizethanks{\IEEEcompsocthanksitem Hongjun Wang, Po Hu and Kai Han are with School of Computing and Data Science, The University of Hong Kong. Emails: \{hjwang, stephenhu\}@connect.hku.hk, kaihanx@hku.hk. \protect\\
\IEEEcompsocthanksitem Corresponding author: Kai Han}}

\maketitle

\begin{abstract}
Generalized Category Discovery (GCD) aims to categorize unlabelled instances from both known and unknown classes by transferring knowledge from labelled data of known classes. Existing methods assume all data comes from a single domain, yet real-world unlabelled data often exhibits domain shifts alongside semantic shifts. We study GCD under domain shifts and propose three frameworks that adapt foundation models, ranging from self-supervised vision models to vision-language models. (i) HiLo disentangles domain and semantic features through multi-level feature extraction and mutual information minimization, combined with PatchMix augmentation and curriculum sampling. (ii) HLPrompt extends HiLo with semantic-aware spatial prompt tuning to suppress background and domain noise. (iii) VLPrompt leverages vision-language models via factorized textual prompts and cross-modal consistency regularization. The three methods share core design principles while operating on different foundation backbones, making them suitable for different deployment scenarios. Extensive experiments on synthetic corruptions and real-world multi-domain shifts demonstrate consistent improvements over strong baselines. Project page: \url{https://visual-ai.github.io/hilo/}
\end{abstract}


\section{Introduction}
\label{sec:intro}

\IEEEPARstart{C}{ategory} Discovery (CD)~\cite{he2025category} aims to automatically partition unlabelled data into meaningful groups by leveraging knowledge from a set of labelled categories.
The task was initially studied as novel category discovery (NCD)~\cite{han2019learning}, where unlabelled data contains only novel categories, and was later relaxed to generalized category discovery (GCD)~\cite{vaze2022generalized}, where unlabelled data contains both seen and unseen categories, better reflecting practical open-world scenarios.

\IEEEpubidadjcol

However, most existing GCD methods~\cite{vaze2022generalized,wen2022parametric,zhang2022promptcal,pu2023dynamic,he2025category} assume that all data originates from the same domain.
This assumption is frequently violated in real-world deployments where data exhibits low-level covariate shifts arising from diverse sensors, environmental conditions, or data acquisition processes.
Consider autonomous vehicles operating across varying weather conditions, medical imaging systems with data from different acquisition devices, or web-scale visual recognition systems aggregating images from heterogeneous sources.
In these scenarios, unlabelled data not only contains novel semantic categories but also exhibits domain shifts, that is, systematic variations in low-level visual statistics such as lighting, color distribution, texture patterns, or image quality.

The confluence of semantic and domain shifts poses a fundamental challenge: the model must remain \textit{sensitive to semantic differences} while being \textit{robust to domain variations} (see~\Cref{fig:teaser}).

\begin{figure}[!t]
\centering
\includegraphics[width=\columnwidth]{\detokenize{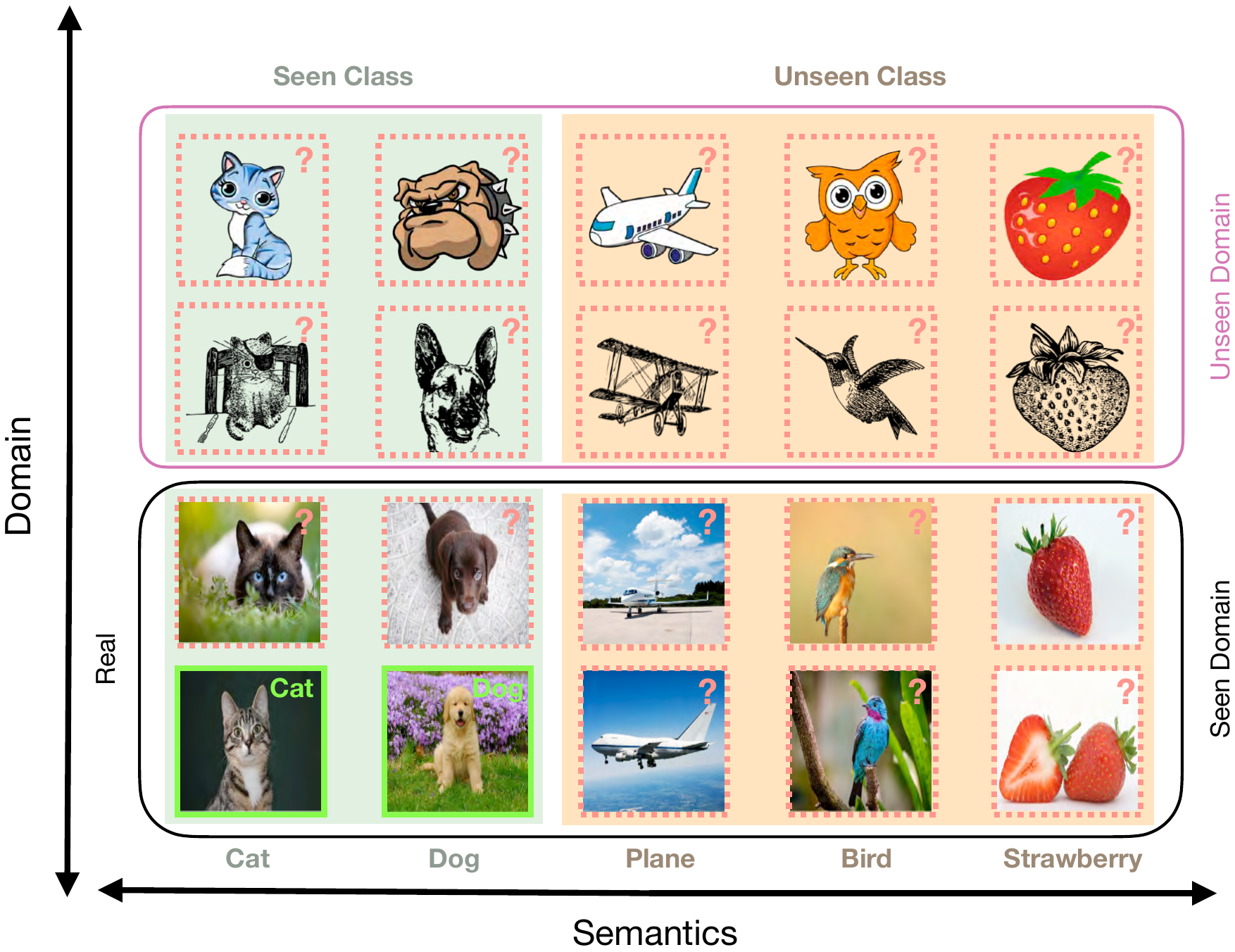}}
\caption{\small{We study Generalized Category Discovery with domain shifts, where a model must categorize unlabelled instances that may come from both seen and unseen categories, and also from seen and novel domains. In the example, the model is given labels only for the images in green boxes and must categorize all remaining unlabelled images, including those from different domains (top rows) and novel categories (rightmost columns).}}
\label{fig:teaser}
\end{figure}

One might consider directly borrowing techniques from domain adaptation or domain generalization. However, domain adaptation methods~\cite{ganin2016domain,long2017deep,zhu2023patch} focus on learning domain-invariant representations but assume fixed, known label spaces. Domain generalization methods~\cite{zhou2021domain,shi2022fish} aim to improve cross-domain robustness but require labelled data from multiple source domains and do not address category discovery.
Neither paradigm directly addresses the unique challenges of GCD with domain shifts, namely, discovering and clustering novel categories in unlabelled data from unseen domains without any labelled examples of those domains or categories.

The difficulty is exacerbated by the fact that domain and semantic information can be entangled in complex ways.
For instance, a domain classifier might exploit spurious correlations between domains and classes (e.g., if certain classes appear predominantly in specific domains), leading to features that are neither truly domain-invariant nor purely semantic.
Moreover, standard domain adaptation objectives like domain adversarial training~\cite{ganin2016domain} may inadvertently suppress semantic information that is correlated with domain characteristics, making it harder to discover fine-grained novel categories.

GCD with domain shifts presents three coupled challenges: (1) \emph{semantic discovery}, i.e., clustering novel categories in $\mathcal{D}_u$ without labels, (2) \emph{domain robustness}, i.e., generalizing from the labelled domain(s) to new domains, and (3) \emph{domain-semantic entanglement}, i.e., preventing domain cues from polluting semantic representations.
To address these challenges, we propose a suite of three frameworks grounded in foundation models, covering both self-supervised vision models and vision-language models~\cite{caron2021emerging,oquab2023dinov2,simeoni2025dinov3,radford2021learning}.

Firstly, we introduce \textbf{HiLo}, a pure-vision framework that extracts high-level semantic features and low-level domain features from different depths of a vision transformer, and minimizes the mutual information (MI) between them to explicitly disentangle domain and semantic factors. HiLo further incorporates PatchMix augmentation~\cite{zhu2023patch} in the embedding space to bridge domain gaps and curriculum sampling to stabilize optimization by gradually introducing harder new-domain samples.

Secondly, we propose \textbf{HLPrompt}, which enhances the pure-vision path from the perspective of prompt-based adaptation. It introduces semantic-aware spatial prompt tuning based on Normalized Cut (NCut)~\cite{shi2000normalized} to selectively inject learnable prompts into foreground object regions and suppress background and domain noise through an alternating optimization scheme.

Thirdly, we propose \textbf{VLPrompt}, which studies the same problem from the perspective of vision-language foundation models. It leverages textual semantic information as an additional source of structure for category discovery through factorized textual prompts and cross-modal consistency regularization, while inheriting and extending HiLo's core components to the cross-modal setting.

These three methods share core design principles while operating on different foundation backbones. HiLo provides a strong pure-vision formulation for disentangling domain and semantic factors. HLPrompt strengthens visual adaptation through prompt tuning on foreground regions. VLPrompt further incorporates cross-modal semantic priors when vision-language supervision is available.
Note that a preliminary version of HiLo was presented in~\cite{wang2024hilo}; HLPrompt and VLPrompt are newly introduced in this extended work. 

In this work, our main contributions are:
\begin{itemize}
\item We formalize GCD with domain shifts and establish benchmarks on both synthetic corruptions (SSB-C) and real-world multi-domain shifts (DomainNet-GCD).
\item We present HiLo, which disentangles domain and semantic features via multi-level feature extraction and MI minimization, enhanced with PatchMix and curriculum sampling for robust category discovery under domain shifts.
\item We develop HLPrompt, extending HiLo with semantic-aware spatial prompt tuning and alternating optimization to better suppress domain noise.
\item We develop VLPrompt, which leverages pretrained vision-language models (VLMs) for domain-robust GCD through factorized textual prompts and cross-modal consistency regularization, while inheriting and extending HiLo's core components to the cross-modal setting.
\end{itemize}
Section~\ref{sec:related} reviews related work. Section~\ref{sec:problem} illustrates the problem. Sections~\ref{sec:hilo},~\ref{sec:hlprompt}, and~\ref{sec:vlprompt} present HiLo, HLPrompt, and VLPrompt, respectively. Section~\ref{sec:experiments} provides experimental evaluation. Section~\ref{sec:conclusion} discusses broader implications and concludes the paper.

\section{Related Work}
\label{sec:related}

Our work lies at the intersection of several research areas. We review the most relevant literature below.

\noindent\textbf{Category discovery.} \ 

The problem of category discovery was firstly studied as Novel Category Discovery (NCD), which was initially formalized in~\cite{han2019learning}. NCD aims to discovery novel categories in the unlabelled data automatically by effectively transferring knowledge from labelled data of seen categories.
Subsequent works have achieved promising progress through rank statistics~\cite{han2020automatically,han2021autonovel}, contrastive learning~\cite{zhong2021neighborhood}, unified objectives~\cite{fini2021unified,jia21joint,zhong2021openmix}, knowledge distillation~\cite{gu2023classrelation}, dual ranking with mutual distillation~\cite{zhao2021novel}, discovery without forgetting~\cite{joseph2022novel}, and more. However, NCD
assumes that unlabelled data contains only novel classes, which is often violated in practice.

The task was later relaxed as GCD~\cite{vaze2022generalized}, allowing unlabelled data to contain both known and novel classes.
SimGCD~\cite{wen2022parametric}  introduced a simple parametric baseline with learnable prototypes. Recent advances have explored diverse directions including semi-supervised mixture models~\cite{zhao2023learning}, cross-instance positive relations~\cite{hao2023cipr}, hierarchical~\cite{he2025seal} and hyperbolic~\cite{liu2025hypcd} formulations, dynamic concept learning~\cite{pu2023dynamic}, debiased learning~\cite{liu2025debgcd}, information-theoretic approaches~\cite{rastegar2024learn}, and part-level correspondence priors~\cite{cendra2025partco}. ~\cite{he2025category} provides a comprehensive overview on category discovery.
Existing GCD methods are primarily studied in a single-domain setting.
To the best of our knowledge, our work is the first to systematically study GCD under domain shifts and develop principled solutions with theoretical backups.

\noindent\textbf{Foundation models and prompt-based adaptation for category discovery.} \ 
Recent category discovery methods increasingly rely on pretrained foundation models because transferable representations reduce the dependence on task-specific supervision. In the pure-vision setting, self-supervised ViTs such as the DINO family~\cite{caron2021emerging, oquab2023dinov2, simeoni2025dinov3} provide strong semantic features for category discovery. In the vision-language setting, CLIP~\cite{radford2021learning} introduces semantic priors through cross-modal alignment, motivating VLM-based GCD methods such as CLIP-GCD~\cite{ouldnoughi2023clipgcd}, CPT~\cite{yang2025cpt}, and GET~\cite{wang2025get}.
Prompt-based adaptation has become a practical way to specialize these foundation models for downstream discovery tasks. On the textual side, CoOp~\cite{zhou2022learning} and CoCoOp~\cite{zhou2022conditional} learn continuous context tokens for CLIP adaptation. On the visual side, VPT~\cite{jia2022visual} prepends learnable tokens to a frozen ViT, and SPTNet~\cite{wang2024sptnet} proposes spatial prompt tuning with alternating optimization for GCD. PromptCAL~\cite{zhang2022promptcal} and PromptCCD~\cite{cendra2024promptccd} further explore prompt-based strategies for novel and continual category discovery, respectively.
However, most existing foundation-model-based GCD methods focus on the standard single-domain setting and do not explicitly address the additional entanglement introduced by \emph{unseen domain shifts} in the unlabelled pool.
Very recently, FREE~\cite{feng2025free} studies GCD under domain shift from a frequency-domain perspective.
In contrast, our work develops a unified suite of frameworks that explicitly disentangle domain and semantics, adapting both DINO-pretrained vision backbones and the CLIP vision-language model to this challenge.

\noindent\textbf{Domain adaptation.} \ 
Unsupervised Domain Adaptation (UDA) transfers knowledge from a labelled source domain to an unlabelled target domain under a \emph{fixed} label space. Representative approaches include moment matching (e.g., MMD/CORAL)~\cite{tzeng2014deep,long2015learning,long2017deep,peng2017open,sun2016deep}, adversarial domain-invariant learning~\cite{ganin2016domain,tzeng2017adversarial,tang2020discriminative,long2018conditional,lee2019sliced}, and self-training with target pseudo-labels~\cite{zou2018unsupervised,li2020model,na2021fixbi}, with recent variants exploiting distributional distances or intermediate-domain augmentation such as PatchMix~\cite{chen2022reusing,du2021cross,zhu2023patch,jin2020minimum}. {Domain Generalization (DG)} targets robustness to unseen domains without target data, typically via meta-learning, invariant representation learning, data augmentation, ensembling, or causal formulations~\cite{li2018learning,muandet2013domain,zhou2021domain,cha2021swad,mahajan2021domain,tobin2017domain}. Both UDA and DG assume \emph{known, fixed label spaces}, whereas our setting must \emph{discover novel categories} in unlabelled data that also comes from unseen/new domains, without any labelled samples from those domains or identities of novel classes.

\section{Problem Formulation and Analysis}
\label{sec:problem}

We now formally define the problem of Generalized Category Discovery with Domain Shifts and analyze its complexity. Concretely, we consider a setting in which a model is given access to labelled data from a source domain. It is further given access to a pool of unlabelled data, where images may come from either the \textit{source domain or new domains}, and whose categories may come from the \textit{labelled classes or from new ones} (see~\Cref{fig:teaser}).

\subsection{Problem Definition}

Let $\mathcal{X}$ denote the input space (e.g., images), $\mathcal{Y}$ the label space, and $\Omega$ the domain space.
A domain $\omega \in \Omega$ is characterized by a conditional distribution $P(\mathbf{x}|y, \omega)$ over inputs given labels, where two domains differ if these conditionals differ for some class $y$.
We are given a labelled dataset $\mathcal{D}_{\ell} = \{(\mathbf{x}_i, y_i)\}_{i=1}^{N_\ell}$ from a fixed labelled domain $\omega^a \in \Omega^{a}$ with labels $y_i \in \mathcal{Y}_{\text{base}} \subset \mathcal{Y}$, together with an unlabelled dataset $\mathcal{D}_{u} = \{\mathbf{x}_i\}_{i=1}^{N_u}$ where images come from both the labelled domain set $\Omega^{a}$ and a disjoint set of new domains $\Omega^{b}$ ($\Omega^{a} \cap \Omega^{b} = \emptyset$), with labels from both $\mathcal{Y}_{\text{base}}$ and novel categories $\mathcal{Y}_{\text{novel}}$ ($\mathcal{Y}_{\text{base}} \cap \mathcal{Y}_{\text{novel}} = \emptyset$).
The goal is to learn a classifier $f: \mathcal{X} \rightarrow \mathcal{Y}_{\text{base}} \cup \mathcal{Y}_{\text{novel}}$ that minimizes the expected classification error on unlabelled data:
\begin{equation}
\min_f \mathbb{E}_{(\mathbf{x}, y, \omega) \sim P_{u}}[\mathds{1}(f(\mathbf{x}) \neq y)]
\end{equation}
where $P_u$ denotes the joint distribution over unlabelled samples, their true labels, and their domains. Formal definitions and assumptions are provided in Appendix~\ref{app:formal_defs}.

\subsection{Key Challenges}

The problem presents three intertwined challenges that must be addressed simultaneously.
First, regarding \emph{semantic discovery}, the model must identify and cluster novel categories from $\mathcal{Y}_{\text{novel}}$ without any labelled examples, relying on transferable knowledge from base categories $\mathcal{Y}_{\text{base}}$.
Second, concerning \emph{domain robustness}, the model must generalize from the labelled domain $\Omega^{a}$ to new domains $\Omega^{b}$ without labelled data from these new domains, a scenario where standard domain adaptation techniques are not directly applicable.
Third, for \emph{feature disentanglement}, the model must learn representations that are invariant to domain-specific variations while preserving discriminative power for semantic differences, which is particularly challenging when both domain and semantic factors vary simultaneously in the unlabelled data.

We make standard assumptions for theoretical analysis---domain-class independence (covariate shift), semantic transferability, and sufficient data diversity---and provide the formal statements and proofs in Appendix~\ref{app:formal_defs} and~\ref{app:proof_complexity}.

\begin{figure*}[!t]
\centering
\includegraphics[width=0.95\textwidth]{\detokenize{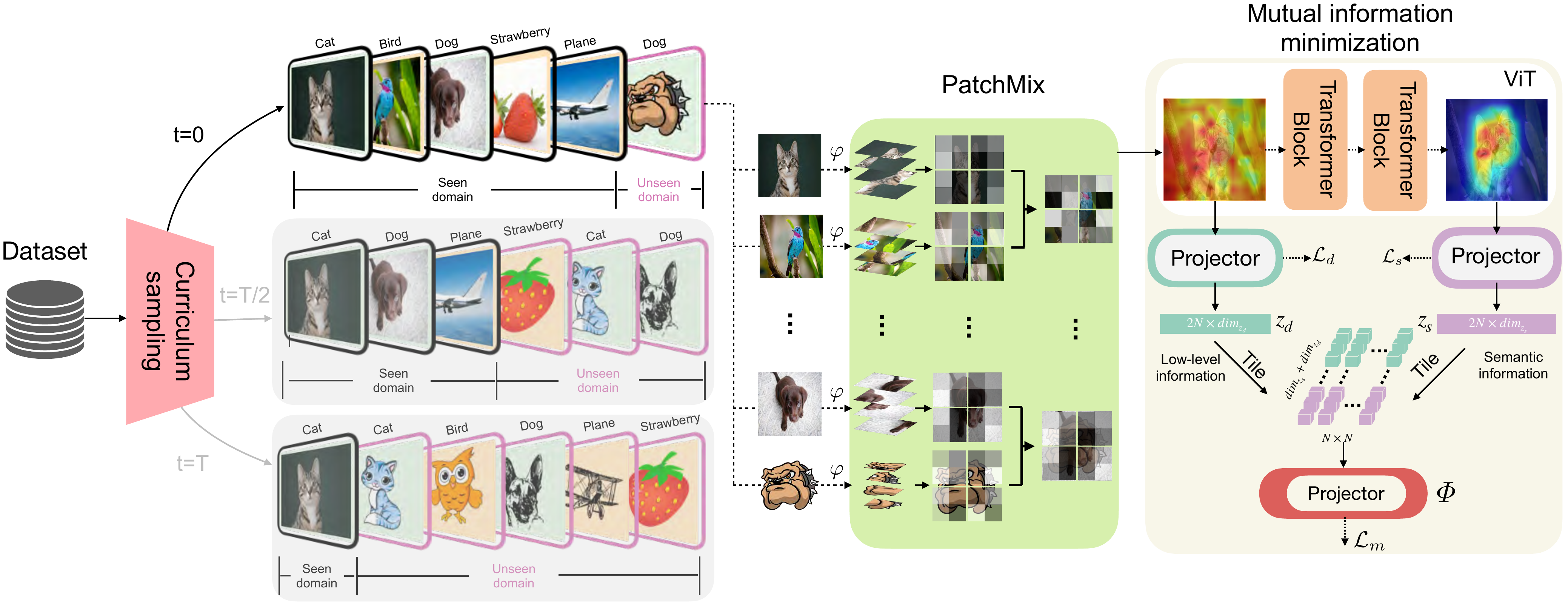}}
\caption{\small{
Overview of HiLo framwork
(Section~\ref{sec:hilo}).
Samples are drawn through our proposed curriculum sampling approach, considering the difficulty of each sample. Labelled and unlabelled samples are paired and augmented through PatchMix which we subtly adapt in the embedding space for contrastive learning for GCD. The mixed-up embeddings are then processed by our network with a high-level (for semantic) and low-level (for domain) feature design, allowing for the domain-semantic disentangled feature learning via mutual information minimization. The full HLPrompt and VLPrompt pipelines are illustrated in~\Cref{fig:pipeline}.}}
\label{fig:hilo_pipeline}
\end{figure*}

\noindent\textbf{Method roadmap.}
We propose a suite of three frameworks for GCD with domain shifts, all grounded in foundation models, sharing core design principles.
We first introduce \textbf{HiLo} (Section~\ref{sec:hilo}), a pure-vision framework built on a vision foundation model that establishes the shared core principles of domain--semantic feature disentanglement, PatchMix augmentation, and curriculum sampling. In our implementation, HiLo uses a DINO-pretrained ViT~\cite{caron2021emerging}.
We then introduce \textbf{HLPrompt} (Section~\ref{sec:hlprompt}), which studies how prompt-based adaptation of vision foundation models can complement HiLo through NCut~\cite{shi2000normalized}-guided semantic-aware spatial prompting. In our implementation, HLPrompt uses a DINOv3-pretrained ViT~\cite{simeoni2025dinov3}.
Finally, we present \textbf{VLPrompt} (Section~\ref{sec:vlprompt}), which studies how textual semantics in vision-language foundation models can further support category discovery through factorized textual prompts and cross-modal consistency regularization. In our implementation, VLPrompt uses CLIP~\cite{radford2021learning}.
An overview of the full pipeline is shown in~\Cref{fig:pipeline}.

\section{HiLo: Feature Disentanglement for GCD with Domain Shifts}
\label{sec:hilo}

We begin with HiLo, a pure-vision framework built on a pretrained vision foundation model. The motivation is that self-supervised vision transformers provide transferable representations whose depth-wise hierarchy naturally separates low-level domain cues from high-level semantic abstractions. In our implementation, HiLo uses a DINO-pretrained Vision Transformer (ViT)~\cite{caron2021emerging}. HiLo extracts domain and semantic features from different depths of the pretrained ViT and minimizes their mutual information to achieve explicit disentanglement, while further incorporating PatchMix augmentation to create intermediate-domain representations that bridge the gap between labelled and new domains, and curriculum sampling to gradually introduce harder new-domain samples during training. The components introduced here, namely MI minimization, PatchMix, and curriculum sampling, will be reused and extended in both HLPrompt (Section~\ref{sec:hlprompt}) and VLPrompt (Section~\ref{sec:vlprompt}). An overview of the HiLo pipeline is provided in~\Cref{fig:hilo_pipeline}.

\subsection{GCD Backbone}

A representative end-to-end baseline for GCD integrates two primary losses for representation learning and parametric classification. The first is a contrastive loss $\mathcal{L}^{rep}$ based on InfoNCE~\cite{oord2018representation} for the representation learning of the feature backbone. The second is a cross-entropy loss $\mathcal{L}^{cls}$ for training a cosine classification head~\cite{gidaris2018dynamic}, utilizing different image views as pseudo-labels for one another. Following GCD~\cite{vaze2022generalized}, the backbone is a DINO-pretrained~\cite{caron2021emerging} ViT containing $m$ Transformer layers; the self-supervised pretraining provides strong transferable features that generalize to novel categories without task-specific supervision. Let $\mathcal{T}$ be the feature extractor consisting of these $m$ layers and $\psi$ be a projection head. For an input image $\mathbf{x}$, a $\ell_2$-normalised feature can be obtained by $\mathbf{h} = \mathcal{\psi}(\mathcal{T}(\phi(\mathbf{x})))$, where $\phi$ is a standard embedding layer before the multi-head attention layers in the ViT model. The representation loss is
\begin{equation}
\label{eq:rep_loss}
\begin{aligned}
\mathcal{L}^{rep}(\mathbf{x})
&= -\frac{1}{|\mathcal{P}(\mathbf{x})|}
\sum_{\mathbf{h}^+ \in \mathcal{P}(\mathbf{x})}
\log \sigma(\mathbf{h} \cdot \mathbf{h}^+; \tau),
\end{aligned}
\end{equation}
where $\sigma(\cdot; \tau)$ is the softmax operation with a temperature $\tau$ for scaling and $\mathcal{P}(\mathbf{x})$ denotes the positive feature set for each $\mathbf{x}$. Suppose we sample a batch $\mathcal{B}$, which contains labelled images and unlabelled images, denoted as $\mathcal{B}^l$ and $\mathcal{B}^u$, respectively. For each $\mathbf{x} \in \mathcal{B}$ (either a labelled or unlabelled image), $\mathcal{P}(\mathbf{x})$ contains only the feature of a different view of the same image. For each $\mathbf{x} \in \mathcal{B}^l$, an additional $\mathcal{P}(\mathbf{x})$ including features of other images from the same class and the feature of a different view of the same image is also used for supervised contrastive learning. Likewise, the classification loss can be written as
\begin{equation}
\label{eq:cls_loss}
\begin{aligned}
\mathcal{L}^{cls}(\mathbf{x})
&= - \sum_{\mathbf{w} \in \mathbf{W}} \mathbf{q}\,
\log \sigma(\hat{\mathbf{h}} \cdot \mathbf{w}; \tau),
\end{aligned}
\end{equation}
where $\mathbf{W}$ is a set of prototypes and each vector $\mathbf{w}$ in $\mathbf{W}$ represents a $\ell_2$-normalised learnable class prototype. $\hat{\mathbf{h}}$ is the $\ell_2$-normalised vector of $\mathcal{T}(\phi(\mathbf{x}))$. For each $\mathbf{x} \in \mathcal{B}$, $\mathbf{q}$ is the pseudo-label from a sharpened prediction of a different view of the same image. For each $\mathbf{x} \in \mathcal{B}^l$, an additional $\mathbf{q}$ as the one-hot ground-truth vector is also used for supervised learning. Let $\mathcal{L}^{r,c}$ be the summation of $\mathcal{L}^{rep}$ and $\mathcal{L}^{cls}$ for simplification, the overall loss can then be written as:
\begin{equation}
\label{simgcd_loss}
\begin{aligned}
\mathcal{L}_{\text{SimGCD}}
&= \lambda \sum_{\mathbf{x} \in \mathcal{B}} \mathcal{L}^{r,c}(\mathbf{x})
 + (1 - \lambda) \sum_{\mathbf{x} \in \mathcal{B}^l} \mathcal{L}^{r,c}(\mathbf{x})
 + \epsilon \Delta,
\end{aligned}
\end{equation}
where $\mathcal{B}^l$ denotes the subset of labelled samples in the current mini-batch, and $\Delta$ is an entropy maximization term to prevent pseudo-label collapse~\cite{assran2022masked}. Finally, $\lambda$ and $\epsilon$ are hyperparameters.

\subsection{Domain and Semantic Features}

Let $\mathcal{T}_\ell$ denote the operation of the first $\ell$ transformer blocks (after the embedding layer $\phi$).
For an input image $\mathbf{x}$, we extract:
\begin{equation}
\begin{aligned}
\mathbf{r}_{\text{dom}} = \mathcal{T}_1(\phi(\mathbf{x})), \quad
\mathbf{r}_{\text{sem}} = \mathcal{T}_L(\phi(\mathbf{x}))
\end{aligned},
\end{equation}
where $\mathbf{r}_{\text{dom}} \in \mathbb{R}^{d_1}$ captures low-level, domain-specific features from the first block, and $\mathbf{r}_{\text{sem}} \in \mathbb{R}^{d_L}$ captures high-level semantic features from the last block.

These representations are then projected through separate MLPs and $\ell_2$-normalized:
\begin{equation}
\begin{aligned}
\mathbf{h}_{\text{dom}} = \psi_{\text{dom}}(\mathbf{r}_{\text{dom}}),\quad
\mathbf{h}_{\text{sem}} = \psi_{\text{sem}}(\mathbf{r}_{\text{sem}})
\end{aligned},
\end{equation}
where $\psi_{\text{dom}}, \psi_{\text{sem}}: \mathbb{R}^{d} \rightarrow \mathbb{R}^{d'}$ are projection heads (3-layer MLPs).

This depth-wise hierarchy is well motivated by information-theoretic principles. Since the representation pipeline forms a Markov chain across layers, the data processing inequality implies that information about raw input statistics (including domain-specific patterns) is progressively discarded with depth. Meanwhile, end-to-end semantic supervision drives later layers to retain label-relevant abstractions. This naturally induces a hierarchy where early layers preserve more domain cues and later layers encode more semantics. A formal statement and proof are provided in Appendix~\ref{app:proof_hierarchy}.

\subsection{Mutual Information Minimization}

The core of HiLo is encouraging independence between domain and semantic features by minimizing their mutual information. We aim to minimize:
\begin{equation}
\begin{aligned}
I(\mathbf{h}_{\text{dom}}; \mathbf{h}_{\text{sem}})
&= H(\mathbf{h}_{\text{dom}}) - H(\mathbf{h}_{\text{dom}} \mid \mathbf{h}_{\text{sem}}),
\end{aligned}
\end{equation}
where $H(\cdot)$ denotes entropy.

Direct computation is intractable for high-dimensional continuous variables. We employ the Jensen-Shannon (JS) divergence-based estimator from Deep InfoMax~\cite{hjelm2018learning}:

\begin{equation}
\label{eq:mi_estimator}
\begin{aligned}
\hat{I}(\mathbf{h}_{\text{dom}}; \mathbf{h}_{\text{sem}})
= &\max_D \Big(
\mathbb{E}_{(\mathbf{h}_d, \mathbf{h}_s) \sim P_{\text{joint}}}
\big[-\log(1 + e^{-D(\mathbf{h}_d, \mathbf{h}_s)})\big] \\
&+ \mathbb{E}_{\mathbf{h}_d \sim P_d, \mathbf{h}_s \sim P_s}
\big[-\log(1 + e^{D(\mathbf{h}_d, \mathbf{h}_s)})\big]
\Big),
\end{aligned}
\end{equation}
where $D: \mathbb{R}^{d'} \times \mathbb{R}^{d'} \rightarrow \mathbb{R}$ is a discriminator network (MLP with output dimension 1), and marginal distributions are obtained by shuffling.

The MI minimization loss is:
\begin{equation}
\label{eq:mi_loss}
\mathcal{L}_{\text{MI}} = \hat{I}(\mathbf{h}_{\text{dom}}; \mathbf{h}_{\text{sem}})
\end{equation}

Minimizing MI between domain and semantic features leads to provably tighter generalization bounds under domain shift. Intuitively, when the semantic feature distribution becomes independent of the domain variable, the domain shift penalty in the generalization bound vanishes. We formalize this in Theorem~\ref{thm:mi_bound} (Appendix~\ref{app:proof_mi_bound}), which shows that the risk on any new domain $\omega^b$ is bounded by the empirical risk plus a standard complexity term and a domain shift term $\lambda(\delta) = O(\sqrt{\delta})$, where $\delta$ controls the KL divergence between domain-conditional feature distributions and their mixture. Lower MI directly reduces $\delta$, thus tightening the bound.

\subsection{PatchMix Augmentation for GCD}

For images $\mathbf{x}, \mathbf{x}'$ (typically a labelled-domain sample and an unlabelled sample from a potentially new domain), we mix their patch embeddings:
\begin{equation}
\label{eq:patchmix_op}
\tilde{\phi}(\mathbf{x})_j = \beta_j \odot \phi(\mathbf{x})_j + (1-\beta_j) \odot \phi(\mathbf{x}')_j
\end{equation}
where $\phi(\mathbf{x})_j \in \mathbb{R}^d$ is the $j$-th patch embedding (for $j = 1, \ldots, P$), $\beta_j \sim \text{Beta}$ is independently sampled for each patch, and $\odot$ denotes element-wise multiplication.

The mixed embeddings $\tilde\phi(\mathbf{x})=\{\tilde{\phi}(\mathbf{x})_j\}_{j=1}^P$ are then processed through transformer blocks to obtain semantic features $\mathbf{\tilde h}=\psi_{\text{sem}}(\mathcal{T}(\tilde\phi(\mathbf{x})))$ for contrastive learning.

For semantic features from mixed embeddings, we use a modified contrastive loss Eq.~\ref{eq:rep_loss}:
\begin{equation}
\label{eq:patchmix_rep_loss}
\begin{aligned}
\mathcal{L}^{rep}_{\text{PM}}(\mathbf{x})
&= -\frac{1}{|\mathcal{P}(\mathbf{x})|}
\sum_{\tilde{\mathbf{h}}^+ \in \mathcal{P}(\mathbf{x})}
\alpha \cdot \log \sigma(\tilde{\mathbf{h}}\cdot\tilde{\mathbf{h}}^+; \tau),
\end{aligned}
\end{equation}
where $\mathbf{h}_{\text{sem}}^+$ denotes positive pairs (same class), $\sigma(\cdot;\tau)$ is softmax operation with temperature $\tau$, and $\alpha = \frac{\mathbf{\beta}\cdot s}{\mathbf{\beta}\cdot s+(1-\mathbf{\beta})\cdot s'}$ with $\mathbf{\beta}=\sum_{j=1}^P\beta_j$, and $s$ and $s'$ are vectors storing the attention scores for all the patches for $\mathbf{x}$ and $\mathbf{x}'$ respectively~\cite{chen2022transmix,zhu2023patch}. The $\alpha$ weighting reflects the proportion of semantic content from the anchor image in the mixed representation.

We also use a modified classification loss Eq.~\ref{eq:cls_loss}:
\begin{equation}
\label{eq:patchmix_cls_loss}
\begin{aligned}
\mathcal{L}^{cls}_{PM}(\mathbf{x})
&= - \sum_{\mathbf{w} \in \mathbf{W'}} \tilde{\mathbf{q}}\,
\log \sigma(\tilde{\mathbf{h}} \cdot \mathbf{w}; \tau),
\end{aligned}
\end{equation}
where $\mathbf{\tilde q}=\alpha\cdot\mathbf{q}+\frac{1-\alpha}{|\mathcal{Y}_{novel}|}\cdot\mathbf{1}$~\cite{szegedy2016rethinking} with $\mathbf{q}$ being the one-hot ground-truth label of $\mathbf{x}$ for labelled $\mathbf{x}$ and the pseudo-label from the prediction of another mix-up view for unlabelled $\mathbf{x}$. We also need to learn another set of prototypes $\mathbf{W}'$.

The resulting PatchMix objective is:
\begin{equation}
\label{eq:patchmix_loss}
\mathcal{L}_{PM}=\mathcal{L}^{rep}_{PM}+\mathcal{L}^{cls}_{PM}
\end{equation}

PatchMix training creates intermediate representations that bridge the domain gap between the labelled-domain distribution and new-domain distributions. When mixing samples $\mathbf{x}$ and $\mathbf{x}'$ from different domains but the same semantic class, the mixed representation $\tilde{\mathbf{h}}$ forms a convex combination in feature space. Training on these mixed samples encourages the classifier to maintain smooth decision boundaries along interpolation paths between domains, which empirically improves cross-domain generalization. This regularization effect is particularly beneficial when the mixing coefficient $\beta$ is sampled from a Beta distribution, creating diverse intermediate representations that cover the space between the labelled domain and new domains.

\subsection{Curriculum Learning}

Curriculum sampling~\cite{bengio2015scheduled} is an effective technique that can enhance generalization by gradually increasing the difficulty of the training data, which is also a natural fit for GCD with domain shifts. We first emphasize labelled-domain-like samples to learn stable semantic representations from labelled supervision, and then introduce more new-domain-like samples as an additional challenge in later training stages. Practically, we design a sampling weight $p_{cs}(\mathbf{x}|t)$ for each sample $\mathbf{x}$ at training epoch $t$ based on separating labelled-domain-like and new-domain-like samples in the domain-feature space. Since domain labels are generally unavailable for unlabelled samples, we run semi-supervised $k$-means on domain features extracted using the DINO-pretrained backbone. Let the resulting clusters be $\hat{\mathcal{D}}_a$ and $\hat{\mathcal{D}}_b$, which correspond to domains $\Omega^{a}$ and $\Omega^{b}$ respectively, and $\mathcal{D}_u = \hat{\mathcal{D}}_a \cup \hat{\mathcal{D}}_b$. We then define the sampling probability weight $p_{cs}(\mathbf{x}|t)$ as follows:

\begin{equation}
    p_{cs}(\mathbf{x}|t) = \begin{cases} 
      \hfill 1, & \mathbf{x} \in \mathcal{D}_l \\
      \hfill \frac{|\mathcal{D}_l|}{|\hat{\mathcal{D}}_a|}, & \mathbf{x} \in \hat{\mathcal{D}}_a, \\
      r_0 + (r' - r_0)\mathbf{1}(t > t'), & \mathbf{x} \in \hat{\mathcal{D}}_b
   \end{cases}
   \label{eq:sampling_prob}
\end{equation}
\noindent where $\mathbf{1}(\cdot)$ is an indicator function, $t'$ is a constant epoch number after which we increase the portion of samples from new domains, $r_0$ and $r'$ are constant probabilities for new-domain samples in the earlier ($< t'$) and later ($> t'$) stages, and $t$ indicates the current training time step. In our formulation, (1) if $\mathbf{x}$ is a labelled sample, its $p_{cs}(\mathbf{x}|t)$ is set to 1, without any discount. (2) If $\mathbf{x}$ is an unlabelled sample in $\hat{\mathcal{D}}_a$ (predicted as from the labelled domain), $p_{cs}(\mathbf{x}|t)$ is set to $\frac{|\mathcal{D}_l|}{|\hat{\mathcal{D}}_a|}$, proportional to the ratio of labelled and unlabelled samples from the labelled domain, following the sampling strategy in conventional GCD~\cite{vaze2022generalized}. (3) If $\mathbf{x}$ is an unlabelled sample in $\hat{\mathcal{D}}_b$ (predicted as from a new domain), its $p_{cs}(\mathbf{x}|t)$ increases along with training after epoch $t'$.
\par\noindent

Under standard stochastic approximation conditions, curriculum-weighted SGD converges to a stationary point, and the warmup phase improves optimization stability by first learning robust semantic features before introducing harder new-domain samples. A theoretical convergence analysis is provided in Appendix~\ref{app:proof_curriculum} (Theorem~\ref{thm:curriculum}).

\subsection{Overall Training Objective and Algorithm}
The complete HiLo objective combines domain and semantic clustering, MI minimization, PatchMix, and curriculum sampling:
\begin{equation}
\label{eq:hilo_full}
\begin{aligned}
\mathcal{L}_{\text{HiLo}}
&= \mathcal{L}_{\text{s}} + \epsilon_{s} \Delta_{s}
 + \mathcal{L}_{\text{d}} + \epsilon_{d} \Delta_{d}
 + \mathcal{L}_{\text{MI}} + \mathcal{L}_{\text{PM}}
\end{aligned}
\end{equation}

The clustering terms $\mathcal{L}_{\text{d}}$ and $\mathcal{L}_{\text{s}}$ combine supervised contrastive loss on labelled data with self-supervised contrastive loss on all data using low-level domain features $\mathbf{h}_{\text{dom}}$ and high-level semantic features $\mathbf{h}_{\text{sem}}$ respectively, along with a parametric classification loss using learnable prototypes. For the domain classification loss, since domain labels are generally unavailable for unlabelled samples, we run a semi-supervised $k$-means on domain features extracted from the backbone to obtain pseudo-labels. $\Delta$ is the mean entropy maximization term to prevent pseudo-label collapse~\cite{assran2022masked} for both domain and semantic aspects, where $\epsilon_{s}$ and $\epsilon_{d}$ are balancing factors.
The mutual information term $\mathcal{L}_{\text{MI}}$ corresponds to the MI minimization loss defined in Eq.~\ref{eq:mi_loss}.
The PatchMix term $\mathcal{L}_{\text{PM}}$ applies the modified loss in Eq.~\ref{eq:patchmix_loss} on mixed embeddings. Samples are drawn according to curriculum weights $p_{cs}(\mathbf{x}|t)$.

HiLo establishes the core framework for GCD with domain shifts and serves as the foundation for the two extensions introduced here. The next two sections extend HiLo in two directions: HLPrompt adds input-level adaptation via spatial prompt tuning to suppress background domain noise (Section~\ref{sec:hlprompt}), and VLPrompt extends the framework to vision-language models for richer cross-modal semantic structure (Section~\ref{sec:vlprompt}).

\section{HLPrompt: HiLo with Semantic-Aware Spatial Prompt Tuning}
\label{sec:hlprompt}

HiLo operates entirely in feature space, disentangling domain and semantic representations after the image has already been encoded. However, domain-specific noise---particularly from backgrounds carrying strong domain cues such as sky textures, lighting conditions, or artistic styles---enters the representation \emph{before} feature-level mechanisms can act on it.

\textbf{Key idea.}
HLPrompt complements HiLo's feature-level disentanglement with \emph{input-level} adaptation that focuses the model on foreground object regions while suppressing background domain noise.
The core challenge is how to identify foreground regions reliably across different domains without any segmentation labels.
We address this by leveraging Normalized Cut (NCut)~\cite{shi2000normalized} on patch affinity graphs from early ViT layers, exploiting the fact that self-supervised DINO attention naturally concentrates on foreground objects~\cite{caron2021emerging}.

\textbf{Approach.} Concretely, HLPrompt introduces two mechanisms on top of HiLo: (i)~\emph{semantic-aware spatial prompt tuning} that injects learnable prompts selectively into foreground patches identified by NCut, and (ii)~a \emph{two-stage alternating optimization} that decouples prompt learning from model parameter updates for stable convergence.
All of HiLo's core components---MI minimization, PatchMix, and curriculum sampling---are retained; the training objective remains $\mathcal{L}_{\text{HiLo}}$ (Eq.~\ref{eq:hilo_full}).
HLPrompt uses a DINOv3-pretrained ViT~\cite{simeoni2025dinov3}, whose self-supervised attention yields reliable NCut masks even under domain shifts.
See~\Cref{fig:pipeline} for an overview.

\begin{figure*}[!t]
\centering
\includegraphics[width=\textwidth]{\detokenize{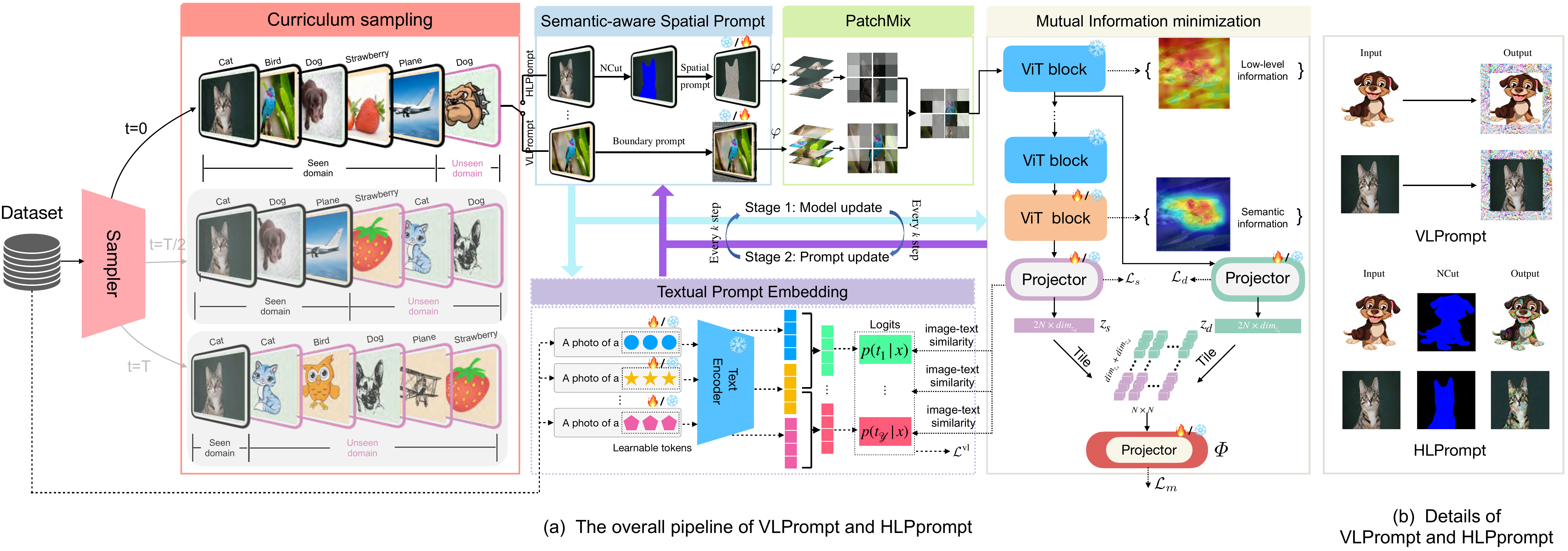}}
\caption{\small{\textbf{Overview of HLPrompt and VLPrompt.} We summarize both the pure-vision path (HLPrompt; building on HiLo) and the vision-language path (VLPrompt) in this figure. The core HiLo components shared by both paths are detailed in~\Cref{fig:hilo_pipeline}.
\textbf{(1) HLPrompt.} A ViT backbone extracts low-level \emph{domain} features from early layers and high-level \emph{semantic} features from late layers. HLPrompt enforces domain--semantic disentanglement via mutual information minimization, improves cross-domain interpolation via PatchMix, and adopts curriculum sampling $p_{cs}(\mathbf{x}\mid t)$ to emphasize labelled-domain-like samples early and new-domain-like samples later. We further integrate NCut~\cite{shi2000normalized}-based semantic-aware spatial prompt tuning to inject visual prompts only into foreground object regions, and optimize prompts and model parameters using a two-stage alternating scheme.
\textbf{(2) VLPrompt.} Starting from a pretrained vision-language model (VLM; e.g., CLIP), VLPrompt learns factorized textual prompts (shared context + category embeddings) and aligns vision/text through a symmetric vision-language objective (Eq.~\ref{eq:vl_loss}) together with semi-supervised classification, MI minimization, and cross-modal PatchMix. VLPrompt uses boundary-based spatial prompt tuning and a two-stage alternating optimization between the visual prompt $\mathbf{Q}_s$ and model/text prompt parameters.
\textbf{Dashed components indicate the optional text branch:} when adopting the pure-vision mechanism, the pipeline ignores the text encoder/prompt pathway; when adopting the vision-language mechanism, the dashed pathway is activated to leverage cross-modal alignment.}}
\label{fig:pipeline}
\end{figure*}

\subsection{Semantic-Aware Spatial Prompt Tuning with NCut}
\label{subsec:ncut_spt}

\noindent
Standard visual prompting methods~\cite{jia2022visual,bahng2022exploring} apply prompts uniformly to all image regions, wasting capacity on backgrounds that carry domain-specific noise. Building upon our prior work on spatial prompt tuning~\cite{wang2024sptnet}, HLPrompt extends this idea by leveraging NCut~\cite{shi2000normalized} to identify foreground regions and applying prompts \emph{only} there.
Given an input image $\mathbf{x}$, we construct a patch affinity graph from intermediate ViT features and apply NCut to obtain a foreground mask $\mathbf{M}\in\{0,1\}^P$
(full details in Appendix~\ref{app:ncut_details}).

\subsubsection{Selective Spatial Prompt Application}

With the foreground mask $\mathbf{M}$, we define spatial prompts $\mathbf{Q}_s = \{\mathbf{q}_1, \mathbf{q}_2, \ldots, \mathbf{q}_P\}$ that are applied selectively:
\begin{equation}
\tilde{\mathbf{x}}_j = \mathbf{x}_j + M_j \cdot \mathbf{q}_j
\end{equation}
where $\mathbf{x}_j$ denotes the $j$-th patch embedding, $M_j \in \{0, 1\}$ is the foreground indicator, and $\mathbf{q}_j$ is the learnable spatial prompt for patch $j$.

This selective application concentrates prompt capacity on semantically relevant regions and avoids background regions that carry domain-specific cues, improving both parameter efficiency and domain robustness.
In practice, we use a single shared foreground prompt $\mathbf{q}_{\text{fg}} \in \mathbb{R}^d$ modulated by the mask: $\mathbf{q}_j = M_j \cdot \mathbf{q}_{\text{fg}}$.

\subsubsection{Two-Stage Alternating Optimization}

Following SPTNet~\cite{wang2024sptnet}, we decouple prompt and model optimization via alternation.
\textbf{Stage~1} freezes the encoder and heads and updates only spatial prompts $\mathbf{Q}_s$;
\textbf{Stage~2} freezes prompts and updates the encoder, projection heads, and discriminator using $\mathcal{L}_{\text{HiLo}}$ (Eq.~\ref{eq:hilo_full}).
The two stages alternate every $k$ iterations:
\begin{equation}
\label{eq:hilo_iter}
\begin{aligned}
\textbf{Stage 1:}\quad
& \mathbf{Q}_s^{(t+1)}
\leftarrow \mathbf{Q}_s^{(t)} - \eta_p \nabla_{\mathbf{Q}_s} \mathcal{L}_{\text{HiLo}} \\
\textbf{Stage 2:}\quad
& \boldsymbol{\Theta}_{\mathcal{T}}^{(t+1)}
\leftarrow \boldsymbol{\Theta}_{\mathcal{T}}^{(t)} - s\eta \nabla \mathcal{L}_{\text{HiLo}} \\
& \{\boldsymbol{\Theta}_{\psi}, \boldsymbol{\Theta}_{D}\}^{(t+1)}
\leftarrow \{\boldsymbol{\Theta}_{\psi}, \boldsymbol{\Theta}_{D}\}^{(t)} - \eta \nabla \mathcal{L}_{\text{HiLo}}
\end{aligned}
\end{equation}
where $\eta_p$ and $\eta$ are learning rates for spatial prompts and model parameters respectively, and $s$ is a scaling factor that reduces the learning rate for the encoder backbone $\mathcal{T}$.
This decomposition avoids interference between prompt and model gradients, leading to more stable convergence.

\noindent During inference, the same NCut-based foreground detection is applied to new images.
The detailed HLPrompt training procedure is provided in~\Cref{alg:hilo} (Appendix~\ref{app:algorithms}).

HLPrompt is dominated by the ViT forward and backward passes, while prompt injection and shallow heads add comparatively small overhead. We defer a detailed complexity breakdown (including NCut mask estimation) to Appendix~\ref{app:complexity}.

\noindent\textbf{Summary.}
HLPrompt extends HiLo with NCut-guided foreground prompting and alternating optimization. It retains HiLo's loss $\mathcal{L}_{\text{HiLo}}$ but changes the optimization landscape: prompts act as learned domain-invariant input transformations that complement HiLo's feature-level disentanglement, and the alternating scheme (Eq.~\ref{eq:hilo_iter}) stabilizes convergence by decoupling prompt and model updates. Together, HiLo and HLPrompt form a complete pure-vision path. The next section turns to vision-language foundation models for an additional source of semantic structure.

\section{VLPrompt: Vision-Language Prompt Tuning for GCD with Domain Shifts}
\label{sec:vlprompt}

HiLo and HLPrompt operate in a pure-vision setting, relying solely on visual representations. However, vision-language foundation models (VLMs) offer a complementary advantage: their aligned visual and textual representations provide richer semantic structure that can be leveraged for category discovery, and large-scale cross-modal pretraining endows them with broader robustness to visual style variations.

\textbf{Key idea.}
VLPrompt extends HiLo's principles to the vision-language setting. The core challenge is twofold: (i)~standard VLM prompts like ``a photo of a [CLASS]'' cannot handle \emph{unnamed} novel categories that must be discovered, and (ii)~domain shifts such as photos-to-sketches can disrupt the pretrained vision-language alignment.
We address (i) by introducing \emph{factorized textual prompts} that decompose each category representation into shared context embeddings (encoding task-level information) and category-specific embeddings (initialized randomly for novel categories), enabling discovery without predefined class vocabularies.
We address (ii) through \emph{cross-modal consistency regularization} that maintains vision-language alignment under domain-interpolating perturbations.

\textbf{Approach.}
VLPrompt inherits HiLo's core components---MI minimization, PatchMix, and curriculum sampling---while adapting them to the cross-modal setting.
It uses CLIP~\cite{radford2021learning} as the backbone, preserving HiLo's feature hierarchy by extracting domain features from early blocks and semantic features from late blocks of the CLIP vision encoder, with MI minimization applied in the shared vision-language embedding space.
Compared with prior CLIP-based GCD methods that focus on single-domain settings (e.g., CLIP-GCD~\cite{ouldnoughi2023clipgcd}, CPT~\cite{yang2025cpt}), VLPrompt additionally targets unseen domain shifts in the unlabelled pool.

\subsection{Architecture}

We instantiate VLPrompt with CLIP's dual-encoder architecture: vision encoder $\Phi^{(v)}: \mathcal{X} \rightarrow \mathbb{R}^{D}$ and text encoder $\Phi^{(t)}: \mathcal{T} \rightarrow \mathbb{R}^{D}$ project to a shared $D$-dimensional embedding space. The vision branch is $\Phi^{(v)}=\mathbf{W}\circ\mathcal{F}\circ\phi$, where $\phi$ is the patchify operation, $\mathcal{F}$ is the multi-head attention stack, and $\mathbf{W}$ projects from vision to text space.
For image $\mathbf{x}$ and text descriptions $\mathbf{u}$, the model computes $\ell_2$-normalized embeddings:
\begin{equation}
\begin{aligned}
\boldsymbol{\pi}^{(v)}(\mathbf{x})
= \frac{\Phi^{(v)}(\mathbf{x})}{\|\Phi^{(v)}(\mathbf{x})\|_2},\quad
\boldsymbol{\pi}^{(t)}(\mathbf{u})
= \frac{\Phi^{(t)}(\mathbf{u})}{\|\Phi^{(t)}(\mathbf{u})\|_2},
\end{aligned}
\end{equation}
where $\phi$ is the patchify operation in $\Phi^{(v)}$. The cross-modal matching score is: 
\begin{equation}
\label{eq:match_score}
    \rho(\mathbf{x}, \mathbf{u}) = \boldsymbol{\pi}^{(v)}(\mathbf{x})^\top \boldsymbol{\pi}^{(t)}(\mathbf{u}).
\end{equation}

\subsection{Adaptive Prompt Learning with Factorized Representations}

Instead of hand-crafted text templates, we learn continuous prompt representations through a factorized embedding structure.

\subsubsection{Factorized Prompt Embeddings}
\label{subsec:vlprompt_factorized_prompts}

We introduce a factorized prompt representation where each class $k$ is described by combining context embeddings shared across all categories with category-specific embeddings.
Let $\boldsymbol{\Theta}_{\text{ctx}} = [\boldsymbol{\theta}_1, \boldsymbol{\theta}_2, \ldots, \boldsymbol{\theta}_N] \in \mathbb{R}^{D \times N}$ denote a set of $N$ learnable context embeddings shared globally, encoding task-agnostic contextual information.
Let $\boldsymbol{\Gamma} = [\boldsymbol{\gamma}_1, \boldsymbol{\gamma}_2, \ldots, \boldsymbol{\gamma}_K] \in \mathbb{R}^{D \times K}$ denote $K$ category-specific embeddings, where $\boldsymbol{\gamma}_k \in \mathbb{R}^D$ captures discriminative features for class $k$.

For category $k$, we construct the composite textual representation as a sequence:
\begin{equation}
\begin{aligned}
\mathbf{S}_k
&= \text{concat}\big([\boldsymbol{\Theta}_{\text{ctx}}, \boldsymbol{\gamma}_k]\big)
\end{aligned}
\end{equation}

This sequence is processed by the frozen text encoder to obtain the category embedding:
\begin{equation}
\boldsymbol{\xi}_k = \Phi^{(t)}(\mathbf{S}_k)
\end{equation}

For categories embeddings $\boldsymbol{\gamma}_k$, we initialize them randomly from a Gaussian distribution $\mathcal{N}(0, \sigma \mathbf{I}_D)$ where $\sigma$ controls initialization scale.
Context embeddings $\boldsymbol{\Theta}_{\text{ctx}}$ are initialized from embeddings of generic phrases like ``a photo of a'' and then optimized end-to-end.

\subsubsection{Analysis of Factorization Benefits}

Factorization reduces the number of free parameters from $K(N+1)D$ (independent prompts) to $(N+K)D$ (shared context + category embeddings), which improves sample efficiency by sharing task-level context across categories---especially in the common regime $K\!\gg\!N$.

\subsection{Semi-supervised Classification}

We adopt a similar semi-supervised classification objective as Eq.~\ref{eq:cls_loss}. The only difference is that instead of learning prototypes $\mathbf{W}$ in a parametric way, we use the cross-modal matching score $\rho(\cdot,\cdot)$ in Eq.~\ref{eq:match_score}. The classification objective is:
\begin{equation}
\label{eq:clip_cls_loss}
    \mathcal{L}^{cls}(\mathbf{x}) = - \sum_{k=1}^{K} \mathbf{q} \log \sigma(\rho(\mathbf{x}, t_k);\tau),
\end{equation}
where $t_k\in T=[t_i]_{i=1}^K$ is a set of $K$ text descriptions, $\sigma(\cdot;\tau)$ is the softmax operation with temperature $\tau$, and $\mathbf{q}$ is the one-hot ground-truth label for labelled $\mathbf{x}$ and is the pseudo-label from a sharpened prediction (cross-modal matching score) of a different view of the same image for unlabelled $\mathbf{x}$.

\subsection{Vision-Language Alignment}

To further fit CLIP into the GCD with domain shift task by factorized text representations, the model must learn semantically meaningful context embeddings $\Theta_{\text{ctx}}$ and category-specific embeddings $\Gamma$, which should also be consistent with visual representations. Therefore, given a sampled mini-batch $\mathcal{B} = \mathcal{B}^l \cup \mathcal{B}^u$ of size $N$, we generate M augmented views for each image, resulting in a total batch size of $MN$. Let $\mathbf{V}=\pi^{(v)}(\mathbf{X})$ be the batch of normalized visual embeddings extracted by the image encoder, where $\mathbf{X}$ is the mini-batch of images.

We aim to derive a corresponding batch of aligned text representations $\mathbf{T}$ from a bank of category text embeddings. Let $\mathbf{E} = [\boldsymbol{\xi}_k]_{k=1}^K \in \mathbb{R}^{K\times D}$ denote the embedding bank, where each row is a $\ell_2$-normalized category embedding $\boldsymbol{\xi}_k$.
For a labelled sample $\mathbf{x}_i \in \mathcal{B}^l$, we directly select the text representation $\mathbf{t}_i \in \mathbf{E}$ according to its class label. For an unlabelled sample $\mathbf{x}_i \in \mathcal{B}^u$, we utilize the pseudo-label $\mathbf{q}_i$ derived from the sharpened cross-modal matching score of a different view of the image: $\mathbf{q}_i = \sigma(\rho(\mathbf{x}'_i, \mathbf{E});\tau_{text})$, where $\sigma(\cdot;\tau)$ is the softmax operation with temperature $\tau$. We then adopt a soft assignment scheme to compute the text representation as the probability-weighted sum of all category embeddings: $\mathbf{t}_i = \mathbf{q}_i \cdot \mathbf{E}$.

With the constructed visual batch $\mathbf{V}$ and text batch $\mathbf{T}$, we compute the image-text similarity logits as $\mathcal{S} = \mathbf{V} \mathbf{T}^\top$. To enforce alignment, we design a vision-language consistency objective $\mathcal{L}^{\text{vl}}$ that minimizes the symmetric Kullback-Leibler (KL) divergence between the image-to-text and text-to-image matching distributions:

\begin{equation}
\label{eq:vl_loss}
    \mathcal{L}^{\text{vl}} = \frac{1}{2} \left( \mathcal{L}_{i2t} + \mathcal{L}_{t2i} \right),
\end{equation}
where
\begin{equation}
\label{eq:vl_loss_sub}
\begin{split}
    \mathcal{L}_{i2t}
    &= \frac{1}{MN} \sum_{i=1}^{MN}
    D_{\text{KL}}(\boldsymbol{p}_i^{t2i} \| \boldsymbol{p}_i^{i2t}), \\
    \mathcal{L}_{t2i}
    &= \frac{1}{MN} \sum_{i=1}^{MN}
    D_{\text{KL}}(\boldsymbol{p}_i^{i2t} \| \boldsymbol{p}_i^{t2i}),
\end{split}
\end{equation}
where $p_i^{i2t} = \sigma(\mathcal{S}_i)$ and $p_i^{t2i} = \sigma((\mathcal{S}^\top)_i)$.

\subsection{Mutual Information Minimization}

We integrate HiLo's mutual information minimization (Section~\ref{sec:hilo}) into the VLPrompt architecture.
The motivation is the same---encouraging domain-invariant semantic features---but the adaptation is non-trivial: since CLIP aligns vision and text in a shared embedding space, we perform MI minimization \emph{after} the projection layer $\mathbf{W}$ so that the disentangled semantic features can be directly matched with text features, which are purely semantic.
Let $\mathcal{F}$ denote the transformer stack in the vision encoder, and let $\mathcal{F}_\ell$ denote the composition of its first $\ell$ transformer blocks (so $\mathcal{F}_L=\mathcal{F}$ for the final block index $L$).
For an input image $\mathbf{x}$, we extract the first-block and last-block outputs of $\mathcal{F}$ (after the projection $\mathbf{W}$) as the domain and semantic features, respectively:
\begin{equation}
\mathbf{r}_{\text{dom}}=\mathbf{W}\cdot\mathcal{F}_1(\phi(\mathbf{x})),
\quad\mathbf{r}_{\text{sem}}=\mathbf{W}\cdot\mathcal{F}_L(\phi(\mathbf{x})),
\end{equation}
and estimate mutual information using Eq.~\ref{eq:mi_estimator} with a discriminator $D:\mathbb{R}^D\times\mathbb{R}^D\to\mathbb{R}$.
The MI minimization loss is:
\begin{equation}
    \mathcal{L}_{MI}=\hat I(\mathbf{r}_{\text{dom}},\mathbf{r}_{\text{sem}})
\end{equation}
according to Eq.~\ref{eq:mi_loss}.

\subsection{Cross-modal PatchMix}

We extend HiLo's PatchMix augmentation (Section~\ref{sec:hilo}) to the vision-language setting.
The key new challenge is that, unlike the pure-vision case where only visual features are mixed, we must simultaneously construct consistent mixed representations in \emph{both} the visual and text modalities to maintain cross-modal alignment.
When we construct a mixed representation of $\mathbf{x}$ and $\mathbf{x}'$ by mixing their patch token embeddings as in Eq.~\ref{eq:patchmix_op}:
\begin{equation}
    \tilde{\phi}(\mathbf{x})_j = \beta_j \odot \phi(\mathbf{x})_j + (1-\beta_j) \odot \phi(\mathbf{x}')_j,
\end{equation}
where $\phi$ is the patchify operator in $\Phi^{(v)}$. We still need to align the mixed visual tokens with text features according to Eq.~\ref{eq:vl_loss}. Therefore, we construct mixed text features as well. Let $t_s$ and $t_t$ be the text features constructed for $\mathbf{x}\in\mathcal{B}^l$ and $\mathbf{x}'\in\mathcal{B}^u$ using the factorized prompts in Section~\ref{subsec:vlprompt_factorized_prompts}. Naturally, the text mixing coefficient should correspond to the visual mixing coefficients. Since $\beta_j$ acts on patches while text features are global, we take their mean:
\begin{equation}
\label{text_patchmix}
    \tilde t = \tilde \beta\cdot t_s+(1-\tilde \beta)\cdot t_t,
\end{equation}
where $\tilde \beta=\frac{1}{P}\sum_{j=1}^P\beta_j$.
Importantly, PatchMix can mix patches from different semantic categories, and our formulation treats the resulting view as a \emph{mixed-label} training instance rather than forcing a single ``clean'' text description.
We view the prompt-induced text embeddings $\{\boldsymbol{\xi}_k\}$ as prototype-like class anchors in the shared embedding space, so a patch-mixed visual representation naturally corresponds to a \emph{soft semantic target} over anchors.
Eq.~\ref{text_patchmix} makes this explicit by constructing a mixed text feature using the same mixing weight as the visual mixing, which ensures that both the PatchMix classification loss (with soft targets) and the symmetric vision-language consistency loss (Eq.~\ref{eq:vl_loss}) optimize a consistent cross-modal mixup objective.
Therefore, even when PatchMix blends heterogeneous semantics, the learning signal remains well-defined: the mixed image is trained to match a mixed distribution over class anchors, which stabilizes alignment under strong, domain-interpolating perturbations.
Therefore, $\mathcal{\tilde S}=\mathbf{\tilde V}\mathbf{\tilde T}^\top$, where $\mathbf{\tilde V}=\text{concat}(\mathcal{F}(\tilde\phi(\mathbf{x})))$ and $\mathbf{\tilde T}=\text{concat}(\tilde t)$. The resulting PatchMix vision-language alignment loss is derived according to Eq.~\ref{eq:vl_loss} and Eq.~\ref{eq:vl_loss_sub}.

The PatchMix classification loss $\mathcal{L}_{PM}^{cls}$ is the combination of Eq.~\ref{eq:patchmix_cls_loss} and Eq.~\ref{eq:clip_cls_loss}:
\begin{equation}
\begin{aligned}
\mathcal{L}^{cls}_{PM}(\mathbf{x})
&= - \sum_{k=1}^{K} \tilde{\mathbf{q}}\,
\log \sigma(\tilde \rho(\mathbf{x}, t_k);\tau),
\end{aligned}
\end{equation}
where $\tilde\rho(\mathbf{x},\cdot)=\left(\frac{\mathbf{W}\cdot\mathcal{F}(\tilde\phi(\mathbf{x}))}{\|\mathbf{W}\cdot\mathcal{F}(\tilde\phi(\mathbf{x}))\|_2}\right)^\top\cdot\boldsymbol{\pi}^{(t)}(\cdot)$ is the modified cross-modal matching score with the mixed token embedding as the visual input.

The resulting PatchMix objective for VLPrompt is:
\begin{equation}
\label{eq:clip_pm_loss}
  \mathcal{L}_{PM}=\mathcal{L}^{cls}_{PM}+\mathcal{L}^{vl}_{PM}.
\end{equation}

\subsection{Training Objective}

The complete VLPrompt training objective integrates semi-supervised classification, mutual information minimization, vision-language alignment, and cross-modal PatchMix:
\begin{equation}
\label{eq:vlprompt_full}
\begin{aligned}
\mathcal{L}_{\text{VLPrompt}}
&= \mathcal{L}^{cls} + \epsilon\Delta
 + \beta_1 \mathcal{L}_{MI} + \beta_2 \mathcal{L}^{vl} + \mathcal{L}_{PM},
\end{aligned}
\end{equation}
where $\Delta$ is the mean-entropy maximization regularizer in Eq.~\ref{simgcd_loss} to prevent pseudo-label collapse. 
The coefficients $\epsilon,\beta_1, \beta_2 > 0$ balance the different objective components. The optimization procedure and data sampling schedule are described in Section~\ref{sec:vlprompt_optim} below.

\subsection{Spatial Prompt Tuning}
Both HLPrompt and VLPrompt employ spatial prompt tuning (SPT)~\cite{wang2024sptnet}, but with different masking strategies.
HLPrompt uses NCut-based foreground masking (Section~\ref{subsec:ncut_spt}), which is reliable for the DINO backbone whose attention concentrates on foreground objects~\cite{caron2021emerging}.
VLPrompt instead adopts \emph{boundary-based} SPT: a learnable prompt $\mathbf{Q}_s \in \mathbb{R}^{3 \times H \times W}$ is injected at patch boundaries via a fixed binary mask $\mathbf{M}$, yielding $\tilde{\mathbf{x}} = \mathbf{x} + (\mathbf{Q}_s \odot \mathbf{M})$.
Boundary masking is preferred for CLIP because its attention distributes uniformly across patches for cross-modal alignment, making NCut masks less stable; boundary prompting modulates inter-patch interactions while preserving patch centroids, better respecting CLIP's pretrained geometry (verified in Appendix~\ref{app:spt_analysis} and Table~\ref{tab:vlprompt_ablation}).
Mask construction details are provided in Appendix~\ref{app:boundary_spt_details}.

\subsection{Two-Stage Alternating Optimization}
\label{sec:vlprompt_optim}

The VLPrompt framework jointly optimizes three groups of parameters: the vision encoder and discriminator, the textual prompt embeddings ($\boldsymbol{\Theta}_{\text{ctx}}, \boldsymbol{\Gamma}$), and the spatial prompts $\mathbf{Q}_s$.
We adopt a two-stage alternating training procedure, following the EM-inspired scheme from SPTNet~\cite{wang2024sptnet} (also used in HLPrompt, Section~\ref{sec:hlprompt}), to decouple visual prompt learning from the remaining parameters.

\textbf{Stage 1: Vision Encoder and Textual Prompt Update.}
With spatial prompts fixed, we update CLIP's encoders $\Phi$,  textual context embeddings $\boldsymbol{\Theta}_{\text{ctx}}$ and category embeddings $\boldsymbol{\Gamma}$ using the full VLPrompt objective from Eq.~\ref{eq:vlprompt_full}. Importantly, the spatial prompts provide an augmented view of the input, serving as a learned domain-invariant transformation.

\textbf{Stage 2: Visual Prompt Learning.}
With the visual encoder and textual prompt embeddings fixed, we optimize the spatial prompts $\mathbf{Q}_s$.

The two stages alternate every $k$ iterations until convergence:
\begin{equation}
\label{eq:optimization}
\begin{aligned}
\textbf{Stage 1:}\quad
& \boldsymbol{\Theta}_{\Phi}^{(t+1)}
\leftarrow \boldsymbol{\Theta}_{\Phi}^{(t)} - s\eta \nabla \mathcal{L}_{\text{VLPrompt}} \\
& \boldsymbol{\Theta}_{\text{ctx}}^{(t+1)}
\leftarrow \boldsymbol{\Theta}_{\text{ctx}}^{(t)} - \eta \nabla \mathcal{L}_{\text{VLPrompt}} \\
& \boldsymbol{\Gamma}^{(t+1)}
\leftarrow \boldsymbol{\Gamma}^{(t)} - \eta \nabla \mathcal{L}_{\text{VLPrompt}} \\
\textbf{Stage 2:}\quad
& \mathbf{Q}_s^{(t+1)}
\leftarrow \mathbf{Q}_s^{(t)} - \eta_p \nabla_{\mathbf{Q}_s} \mathcal{L}_{\text{VLPrompt}}
\end{aligned}
\end{equation}
where $\eta$ and $\eta_p$ are learning rates for textual and spatial prompts respectively, and $s$ is a scaling factor that reduces the learning rate for CLIP vision encoder $\Phi^{(v)}$.
Samples are drawn according to the curriculum sampling weight $p_{cs}(\mathbf{x}|t)$ in Eq.~\ref{eq:sampling_prob}, following the same schedule as HiLo (Section~\ref{sec:hilo}).

Note that, unlike what HLPrompt does in Eq.~\ref{eq:hilo_iter}, we first optimize the model parameters, i.e., textual prompt and category embeddings, and then visual prompts. This is because if we optimize visual prompts first with text descriptions randomly initialized and fixed, the model cannot output meaningful predictions with cross-modal matching score, and thus fails to provide any meaningful updates to the visual prompts as well.

During inference, we apply the learned spatial prompts and compute category probabilities via the cross-modal matching score with the learned textual descriptions.
The full training procedure is provided in~\Cref{alg:vlprompt}, and a detailed complexity breakdown in~\Cref{app:complexity}.

\begin{table*}[t]
\centering
\caption{Evaluation on the DomainNet-GCD Dataset. The model is trained on the labelled domain `Real' (i.e., $\Omega^{a}$) together with one new domain (i.e., $\Omega^{b}$) from `Painting'/`Sketch'/`Quickdraw'/`Clipart'/`Infograph' in turn.\label{tab:table1}}
\resizebox{\textwidth}{!}{
\setlength{\tabcolsep}{2pt}
\begin{tabular}{lcccccccccccccccccccccccccccccc}
\toprule
 & \multicolumn{6}{c}{\textbf{Real+Painting}} & \multicolumn{6}{c}{\textbf{Real+Sketch}} & \multicolumn{6}{c}{\textbf{Real+Quickdraw}} & \multicolumn{6}{c}{\textbf{Real+Clipart}} & \multicolumn{6}{c}{\textbf{Real+Infograph}} \\
\cmidrule(lr){2-7} \cmidrule(lr){8-13} \cmidrule(lr){14-19} \cmidrule(lr){20-25} \cmidrule(lr){26-31}
 & \multicolumn{3}{c}{Real} & \multicolumn{3}{c}{Painting} & \multicolumn{3}{c}{Real} & \multicolumn{3}{c}{Sketch} & \multicolumn{3}{c}{Real} & \multicolumn{3}{c}{Quickdraw} & \multicolumn{3}{c}{Real} & \multicolumn{3}{c}{Clipart} & \multicolumn{3}{c}{Real} & \multicolumn{3}{c}{Infograph} \\
\cmidrule(lr){2-4} \cmidrule(lr){5-7} \cmidrule(lr){8-10} \cmidrule(lr){11-13} \cmidrule(lr){14-16} \cmidrule(lr){17-19} \cmidrule(lr){20-22} \cmidrule(lr){23-25} \cmidrule(lr){26-28} \cmidrule(lr){29-31}
\textbf{Methods} & All & Old & New & All & Old & New & All & Old & New & All & Old & New & All & Old & New & All & Old & New & All & Old & New & All & Old & New & All & Old & New & All & Old & New \\
\midrule
RankStats+ & 34.1 & 62.0 & 19.7 & 29.7 & 49.7 & 9.6 & 34.2 & 62.0 & 19.8 & 17.1 & 31.1 & 6.8 & 34.1 & 62.5 & 19.5 & 4.1 & 4.4 & 3.9 & 34.0 & 62.4 & 19.4 & 24.1 & 45.1 & 6.2 & 34.2 & 62.4 & 19.6 & 12.5 & 21.9 & 6.3 \\
UNO+ & 44.2 & 72.2 & 29.7 & 30.1 & 45.1 & 17.2 & 43.7 & 72.5 & 28.9 & 12.5 & 17.0 & 9.2 & 31.1 & 60.0 & 16.1 & 6.3 & 5.8 & 6.8 & 44.5 & 66.1 & 33.3 & 21.9 & 35.6 & 10.1 & 42.8 & 69.4 & 29.0 & 10.9 & 15.2 & 8.0 \\
ORCA & 31.9 & 49.8 & 23.5 & 28.7 & 38.5 & 7.1 & 32.5 & 50.0 & 23.9 & 11.4 & 14.5 & 7.2 & 19.2 & 39.1 & 15.3 & 3.4 & 3.5 & 3.2 & 32.0 & 49.7 & 23.9 & 19.1 & 31.8 & 4.3 & 29.1 & 47.7 & 20.1 & 8.6 & 13.7 & 7.1 \\
GCD & 47.3 & 53.6 & 44.1 & 32.9 & 41.8 & 23.0 & 48.0 & 53.8 & 45.3 & 16.6 & 22.4 & 11.1 & 37.6 & 41.0 & 35.2 & 5.7 & 4.2 & 6.9 & 47.7 & 53.8 & 44.3 & 22.4 & 34.4 & 16.0 & 41.9 & 46.1 & 39.0 & 10.9 & 17.1 & 8.8 \\
SimGCD & 61.3 & 77.8 & 52.9 & 34.5 & 35.6 & 33.5 & 62.4 & 77.6 & 54.6 & 16.4 & 20.2 & 13.6 & 47.4 & 64.5 & 37.4 & 6.6 & 5.8 & 7.5 & 61.6 & 77.2 & 53.6 & 23.9 & 31.5 & 17.3 & 52.7 & 67.0 & 44.8 & 11.6 & 15.4 & 9.1 \\
FREE & 67.7 & 78.1 & 61.2 & 45.6 & 46.1 & 44.8 & 67.8 & 78.2 & 61.6 & 22.5 & 25.8 & 20.9 & 61.4 & \underline{78.1} & 55.1 & \underline{8.9} & 7.8 & \underline{9.0} & 66.4 & 78.1 & 60.1 & 29.3 & 37.2 & 26.3 & 68.1 & 78.9 & 60.2 & 16.1 & 18.6 & 13.4 \\
\midrule
\rowcolor{demogray}
HiLo (Ours) & 64.4 & 77.6 & 57.5 & 42.1 & 42.9 & 41.3 & 63.3 & 77.9 & 55.9 & 19.4 & 22.4 & 17.1 & 58.6 & 76.4 & 52.5 & 7.4 & 6.9 & 8.0 & 63.8 & 77.6 & 56.6 & 27.7 & 34.6 & 21.7 & 64.2 & 78.1 & 57.0 & 13.7 & 16.4 & 11.9 \\
\rowcolor{demogray}
HLPrompt (Ours) & \underline{70.0} & \underline{81.7} & \underline{64.0} & \underline{51.8} & \underline{54.2} & \underline{49.5} & \underline{69.2} & \underline{82.5} & \underline{62.3} & \underline{43.2} & \underline{47.8} & \underline{39.9} & \underline{62.9} & 75.8 & \underline{56.2} & 8.2 & \underline{8.0} & 8.4 & \underline{70.5} & \underline{82.8} & \underline{64.1} & \underline{53.7} & \underline{60.9} & \underline{47.4} & \underline{68.6} & \underline{82.1} & \underline{61.6} & \underline{21.8} & \underline{26.7} & \underline{18.5}\\
\rowcolor{demogray}
VLPrompt (Ours) & \textbf{77.6} & \textbf{83.5} & \textbf{74.5} & \textbf{61.7} & \textbf{62.5} & \textbf{60.9} & \textbf{77.1} & \textbf{85.1} & \textbf{73.0} & \textbf{55.1} & \textbf{61.4} & \textbf{50.5} & \textbf{74.1} & \textbf{81.3} & \textbf{70.3} & \textbf{12.3} & \textbf{12.3} & \textbf{12.3} & \textbf{78.8} & \textbf{84.4} & \textbf{76.0} & \textbf{64.2} & \textbf{68.5} & \textbf{60.4} & \textbf{77.9} & \textbf{84.4} & \textbf{74.5} & \textbf{35.8} & \textbf{43.0} & \textbf{31.1}\\
\bottomrule
\end{tabular}
}
\end{table*}

\begin{table*}[t]
\centering
\caption{Evaluation on SSB-C datasets. We report baseline results on the labelled/original domain (i.e., $\Omega^{a}$) and overall performance under corruptions (i.e., new domains $\Omega^{b}$).\label{tab:table2}}
\resizebox{0.9\textwidth}{!}{
\setlength{\tabcolsep}{3.5pt}
\begin{tabular}{lcccccccccccccccccc}
\toprule
 & \multicolumn{6}{c}{CUB-C} & \multicolumn{6}{c}{Scars-C} & \multicolumn{6}{c}{FGVC-C} \\
\cmidrule(lr){2-7} \cmidrule(lr){8-13} \cmidrule(lr){14-19}
 & \multicolumn{3}{c}{Original} & \multicolumn{3}{c}{Corrupted} & \multicolumn{3}{c}{Original} & \multicolumn{3}{c}{Corrupted} & \multicolumn{3}{c}{Original} & \multicolumn{3}{c}{Corrupted} \\
\cmidrule(lr){2-4} \cmidrule(lr){5-7} \cmidrule(lr){8-10} \cmidrule(lr){11-13} \cmidrule(lr){14-16} \cmidrule(lr){17-19}
Methods & All & Old & New & All & Old & New & All & Old & New & All & Old & New & All & Old & New & All & Old & New \\
\midrule
RankStats+ & 19.3 & 22.0 & 15.4 & 13.6 & 23.9 & 4.5 & 14.8 & 20.8 & 7.8 & 11.5 & 22.6 & 1.0 & 14.4 & 16.4 & 14.5 & 8.3 & 15.6 & 5.0 \\
UNO+ & 25.9 & 40.1 & 21.3 & 21.5 & 33.4 & 8.6 & 22.0 & 41.8 & 7.0 & 16.9 & 29.8 & 4.5 & 22.0 & 33.4 & 15.8 & 16.5 & 25.2 & 8.8 \\
ORCA & 18.2 & 22.8 & 14.5 & 21.5 & 23.1 & 18.9 & 19.1 & 28.7 & 11.2 & 15.0 & 22.4 & 8.3 & 17.6 & 19.3 & 16.1 & 13.9 & 17.3 & 10.1 \\
GCD & 26.6 & 27.5 & 25.7 & 25.1 & 28.7 & 22.0 & 22.1 & 35.2 & 20.5 & 21.6 & 29.2 & 10.5 & 25.2 & 28.7 & 23.0 & 21.0 & 23.1 & 17.3 \\
SimGCD & 31.9 & 33.9 & 29.0 & 28.8 & 31.6 & 25.0 & 26.7 & 39.6 & 25.6 & 22.1 & 30.5 & 14.1 & 26.1 & 28.9 & 25.1 & 22.3 & 23.2 & 21.4 \\
UniOT & 27.5 & 29.3 & 26.8 & 27.3 & 33.2 & 22.5 & 24.3 & 37.5 & 22.3 & 22.9 & 31.4 & 13.7 & 27.3 & 29.8 & 22.5 & 21.6 & 23.5 & 19.6 \\
FREE & \underline{60.4} & 58.5 & \textbf{63.2} & \underline{55.7} & 57.1 & \underline{53.7} & 43.6 & 48.1 & 40.8 & 38.9 & 46.1 & 32.6 & \underline{48.5} & \underline{54.9} & \underline{51.2} & 35.0 & 32.4 & \underline{38.9} \\
\midrule
\rowcolor{demogray}
HiLo (Ours) & 56.8 & 54.0 & \underline{60.3} & 52.0 & 53.6 & 50.5 & 39.5 & 44.8 & 37.0 & 35.6 & 42.9 & 28.4 & 44.2 & 50.6 & 47.4 & 31.2 & 29.0 & 33.4 \\
\rowcolor{demogray}
HLPrompt (Ours) & \textbf{60.4} & \textbf{72.8} & 54.2 & \textbf{59.1} & \textbf{60.8} & \textbf{57.4} & \textbf{67.5} & \textbf{82.3} & \textbf{60.3} & \textbf{63.3} & \underline{70.7} & \textbf{56.2} & \textbf{59.1} & \textbf{56.7} & \textbf{60.3} & \textbf{50.9} & \textbf{44.1} & \textbf{57.8} \\
\rowcolor{demogray}
VLPrompt (Ours) & 59.7 & \underline{68.5} & 55.3 & 55.6 & \underline{60.3} & 51.0 & \underline{62.4} & \underline{81.6} & \underline{53.0} & \underline{59.0} & \textbf{71.6} & \underline{46.9} & 41.5 & 42.3 & 41.0 & \underline{35.5} & \underline{34.1} & 36.8 \\
\bottomrule
\end{tabular}
}
\end{table*}

\section{Experiments}
\label{sec:experiments}

We present a comprehensive experimental evaluation of our proposed suite of methods: HiLo, HLPrompt, and VLPrompt.

\subsection{Experimental Setup}

\textbf{Datasets.} We evaluate our methods on \textbf{SSB-C}~\cite{vaze2021open}, which augments standard GCD benchmarks with ImageNet-C-style corruptions~\cite{hendrycks2019benchmarking} on CUB-200~\cite{wah2011caltech}, Stanford Cars~\cite{krause20133d}, and FGVC-Aircraft~\cite{maji2013fine}, and on \textbf{DomainNet-GCD} constructed from DomainNet~\cite{peng2019moment} (345 categories across six domains).
In both benchmarks, labelled data comes from a single source domain covering 50\% base classes, while unlabelled data spans all classes and multiple (possibly unseen) domains.

\textbf{Evaluation protocols.} Following standard GCD protocols~\cite{vaze2022generalized}, we report clustering accuracy on three splits after Hungarian matching: \textit{All} (all classes), \textit{Known} (base classes), and \textit{Novel} (novel classes only).

Full implementation details, including datasets statistics, hyperparameter settings, augmentation strategies, corruption specifications, and curriculum sampling configurations, are provided in Appendix~\ref{app:impl_details}.

\subsection{Main Results}
We compare against representative baselines and state-of-the-art methods, including GCD~\cite{vaze2022generalized}, ORCA~\cite{cao2021open}, SimGCD~\cite{wen2022parametric}, RankStats+~\cite{han2021autonovel}, UNO+~\cite{fini2021unified}, and FREE~\cite{feng2025free}. FREE is the only concurrent method that also considers GCD with domain shifts.
Among our methods, HiLo establishes the core framework; HLPrompt and VLPrompt are two extensions built upon HiLo with prompt-based adaptation, leveraging vision and vision-language foundation models respectively.
Results on DomainNet-GCD and SSB-C are reported in Tables~\ref{tab:table1} and~\ref{tab:table2}, respectively, where all three of our methods are highlighted.

Tables~\ref{tab:table1} and~\ref{tab:table2} reveal a clear progression across our three methods.
\textbf{HiLo} substantially outperforms all prior baselines that do not explicitly model domain shifts, confirming the effectiveness of the core disentanglement, PatchMix, and curriculum sampling framework.
\textbf{HLPrompt} builds upon HiLo by introducing spatial prompt tuning and alternating optimization; on DomainNet-GCD, it obtains consistent gains on the shifted target domains $\Omega^b$ (e.g., $+9.7\%$ on Painting `All', $+23.8\%$ on Sketch `All' compared with HiLo), and on SSB-C it achieves particularly strong improvements on fine-grained benchmarks (e.g., $+28.0\%$ on Scars-C `All' and $+14.9\%$ on FGVC-C `All').
These gains confirm that NCut-guided foreground prompting effectively suppresses background domain cues that HiLo's feature-level disentanglement alone cannot fully address.
\textbf{VLPrompt} extends HiLo's principles to the vision-language setting by inheriting MI minimization, PatchMix, and curriculum sampling while adding factorized textual prompts and cross-modal consistency regularization to leverage CLIP's aligned visual--textual representations.
On DomainNet-GCD, VLPrompt achieves the best results across nearly all domain pairs---for example, $+9.9\%$ over HLPrompt on Painting `All' and $+11.9\%$ on Sketch `All'---demonstrating that cross-modal semantic priors provide a powerful complement to visual disentanglement for handling extreme real-world domain gaps.
On SSB-C, however, HLPrompt's pure-vision path is more effective for fine-grained corruption-type shifts (e.g., HLPrompt outperforms VLPrompt by $5.1\%$ on Scars-C `All' and by $17.6\%$ on FGVC-C `All').
This asymmetry reflects complementary strengths: the corruption-type domain shifts in SSB-C (noise, blur, weather) are well handled by explicit domain-semantic disentanglement and foreground prompting, while the large visual style gaps in DomainNet (photos vs.\ sketches vs.\ infographics) are better bridged by CLIP's cross-modal pretraining.

\begin{table}[h]
\centering
\caption{HiLo Ablation Study on DomainNet-GCD (Real $\to$ Painting). We incrementally add each HiLo component on top of SimGCD, and also ablate the feature hierarchy design. \cinc{} and \cdec{} indicate absolute gain/drop relative to SimGCD.}
\label{tab:hilo_ablation}
\resizebox{\columnwidth}{!}{
\setlength{\tabcolsep}{4pt}
\begin{tabular}{cl ccc ccc}
\toprule
 & & \multicolumn{3}{c}{Real ($\Omega^a$)} & \multicolumn{3}{c}{Painting ($\Omega^b$)} \\
\cmidrule(lr){3-5} \cmidrule(lr){6-8}
 & Method & All & Old & New & All & Old & New \\
\midrule
Ref. & SimGCD & 61.3 & 77.8 & 52.9 & 34.5 & 35.6 & 33.5 \\
(1) & + PatchMix (original~\cite{zhu2023patch}) & 62.5 \cinc{1.2} & 76.3 \cdec{1.5} & 54.2 \cinc{1.3} & 34.8 \cinc{0.3} & 36.0 \cinc{0.4} & 33.8 \cinc{0.3} \\
(2) & + PatchMix for GCD & 63.5 \cinc{2.2} & 75.0 \cdec{2.8} & 57.6 \cinc{4.7} & 36.6 \cinc{2.1} & 39.6 \cinc{4.0} & 33.6 \cinc{0.1} \\
(3) & + MI Minimization & 66.4 \cinc{5.1} & 79.2 \cinc{1.4} & 59.8 \cinc{6.9} & 35.6 \cinc{1.1} & 36.7 \cinc{1.1} & 34.2 \cinc{0.7} \\
\rowcolor{demogray}
(4) & + Curriculum Sampling (HiLo) & 64.4 & 77.6 & 57.5 & 42.1 & 42.9 & 41.3 \\
\midrule
(5) & Both features from deep layers only & 28.2 \cdec{33.1} & 40.3 \cdec{37.5} & 22.7 \cdec{30.2} & 13.6 \cdec{20.9} & 20.0 \cdec{15.6} & 11.0 \cdec{22.5} \\
(6) & Both features from shallow layers only & 10.1 \cdec{51.2} & 18.1 \cdec{59.7} & ~6.4 \cdec{46.5} & ~5.7 \cdec{28.8} & ~9.2 \cdec{26.4} & ~5.7 \cdec{27.8} \\
\bottomrule
\end{tabular}
}
\end{table}

\begin{table*}[b]
\centering
\newcommand{\cg}{\cellcolor{demogray}}
\caption{HLPrompt ablation study. Each component is removed in turn. The left side shows results on DomainNet-GCD (Real $\to$ Painting) and the right side shows results on CUB-C (Original $\to$ Corrupted). \cinc{} and \cdec{} indicate absolute performance gain and drop relative to the reference, respectively.}
\label{tab:hlprompt_ablation}
\resizebox{\textwidth}{!}{
\setlength{\tabcolsep}{4pt}
\begin{tabular}{cl ccc ccc c ccc ccc}
\toprule
 & & \multicolumn{6}{c}{\textbf{DomainNet}} & & \multicolumn{6}{c}{\textbf{CUB-C}} \\
\cmidrule(lr){3-8} \cmidrule(lr){10-15}
 & & \multicolumn{3}{c}{Real} & \multicolumn{3}{c}{Painting} & & \multicolumn{3}{c}{Original} & \multicolumn{3}{c}{Corrupted} \\
\cmidrule(lr){3-5} \cmidrule(lr){6-8} \cmidrule(lr){10-12} \cmidrule(lr){13-15}
 & Methods & All & Old & New & All & Old & New & & All & Old & New & All & Old & New \\
\midrule
\rowcolor{demogray}
Ref. & HLPrompt & 70.0 & 81.7 & 64.0 & 51.8 & 54.2 & 49.5 & & 60.4 & 72.8 & 54.2 & 59.1 & 60.8 & 57.4 \\
(1) & w/o MI Minimization & 68.7 \cdec{1.3} & 82.2 \cinc{0.5} & 61.8 \cdec{2.2} & 50.1 \cdec{1.7} & 53.4 \cdec{0.8} & 46.9 \cdec{2.6} & & 55.5 \cdec{8.4} & 64.3 \cdec{7.4} & 51.1 \cdec{8.9} & 51.2 \cdec{9.2} & 52.6 \cdec{8.7} & 49.8 \cdec{9.7} \\
(2) & w/o PatchMix & 68.0 \cdec{2.0} & 80.4 \cdec{1.3} & 61.5 \cdec{2.5} & 47.5 \cdec{4.3} & 50.4 \cdec{3.8} & 44.6 \cdec{4.9} & & 59.9 \cdec{4.0} & 67.8 \cdec{3.9} & 55.9 \cdec{4.1} & 55.8 \cdec{4.6} & 58.2 \cdec{3.1} & 53.4 \cdec{6.1}\\
(3) & w/o Alt. Training & 69.8 \cdec{0.2} & 81.5 \cdec{0.2} & 63.8 \cdec{0.2} & 45.8 \cdec{6.0} & 47.8 \cdec{6.4} & 43.7 \cdec{6.0} & & 55.8 \cdec{8.1} & 63.7 \cdec{8.0} & 51.9 \cdec{8.1} & 52.1 \cdec{8.3} & 54.1 \cdec{7.2} & 50.2 \cdec{9.3} \\
(4) & w/o SPT & 71.1 \cinc{1.1} & 82.8 \cinc{1.1} & 65.0 \cinc{1.0} & 45.3 \cdec{6.5} & 47.5 \cdec{6.7} & 43.1 \cdec{6.4} & & 62.3 \cdec{1.6} & 71.4 \cdec{0.3} & 57.7 \cdec{2.3} & 58.4 \cdec{2.0} & 59.2 \cdec{2.1} & 57.6 \cdec{1.9}\\
\midrule
\rowcolor{demogray}
(5) & w/ Semantic-aware SPT (HLPrompt) & 70.0 & 81.7 & 64.0 & 51.8 & 54.2 & 49.5 & & 60.4 & 72.8 & 54.2 & 59.1 & 60.8 & 57.4 \\
\rowcolor{demogray}
(6) & w/ Boundary SPT & 68.9 & 79.3 & 63.5 & 52.4 & 56.1 & 48.8 & & 63.9 & 71.7 & 60.0 & 60.4 & 61.3 & 59.5 \\
\bottomrule
\end{tabular}
}
\end{table*}

\begin{table*}[t]
\centering
\caption{VLPrompt ablation study. Each component is removed in turn. The left side shows results on DomainNet-GCD (Real $\to$ Painting) and the right side shows results on CUB-C (Original $\to$ Corrupted). \cinc{} and \cdec{} indicate absolute performance gain and drop relative to the reference, respectively.}
\label{tab:vlprompt_ablation}
\resizebox{\textwidth}{!}{
\setlength{\tabcolsep}{4pt}
\begin{tabular}{cl ccc ccc c ccc ccc}
\toprule
 & & \multicolumn{6}{c}{\textbf{DomainNet}} & & \multicolumn{6}{c}{\textbf{CUB-C}} \\
\cmidrule(lr){3-8} \cmidrule(lr){10-15}
 & & \multicolumn{3}{c}{Real} & \multicolumn{3}{c}{Painting} & & \multicolumn{3}{c}{Original} & \multicolumn{3}{c}{Corrupted} \\
\cmidrule(lr){3-5} \cmidrule(lr){6-8} \cmidrule(lr){10-12} \cmidrule(lr){13-15}
 & Methods & All & Old & New & All & Old & New & & All & Old & New & All & Old & New \\
\midrule
\rowcolor{demogray}
Ref. & VLPrompt & 77.6 & 83.5 & 74.5 & 61.7 & 62.5 & 60.9 & & 59.7 & 68.5 & 55.3 & 55.6 & 60.3 & 51.0 \\
(1) & w/o MI Minimization & 78.1 \cinc{0.5} & 84.8 \cinc{1.3} & 74.6 \cinc{0.1} & 61.2 \cdec{0.5} & 62.5 & 59.9 \cdec{1.0} & & 59.3 \cdec{0.4} & 67.2 \cdec{1.3} & 55.3 & 55.3 \cdec{0.3} & 59.9 \cdec{0.4} & 50.7 \cdec{0.3} \\
(2) & w/o Vision-language Loss & 72.3 \cdec{5.3} & 84.5 \cinc{1.0} & 65.9 \cdec{8.6} & 55.1 \cdec{6.6} & 58.0 \cdec{4.5} & 52.2 \cdec{8.7} & & 46.9 \cdec{12.8} & 61.0 \cdec{7.5} & 39.9 \cdec{15.4} & 42.0 \cdec{13.6} & 43.0 \cdec{17.3} & 41.1 \cdec{9.9} \\
(3) & w/o Cross-modal PatchMix & 78.6 \cinc{1.0} & 84.1 \cinc{0.6} & 75.7 \cinc{1.2} & 60.7 \cdec{1.0} & 61.2 \cdec{1.3} & 60.1 \cdec{0.8} & & 58.9 \cdec{0.8} & 66.8 \cdec{1.7} & 55.0 \cdec{0.3} & 54.8 \cdec{0.8} & 59.7 \cdec{0.6} & 50.0 \cdec{1.0} \\
(4) & w/o Alt. Training & 77.8 \cinc{0.2} & 84.7 \cinc{1.2} & 74.3 \cdec{0.2} & 58.1 \cdec{3.6} & 58.8 \cdec{3.7} & 57.5 \cdec{3.4} & & 51.0 \cdec{8.7} & 63.0 \cdec{5.5} & 45.1 \cdec{10.2} & 47.4 \cdec{8.2} & 53.4 \cdec{6.9} & 41.4 \cdec{9.6} \\
(5) & w/o SPT & 77.6 & 85.3 \cinc{1.8} & 73.7 \cdec{0.8} & 56.1 \cdec{5.6} & 58.8 \cdec{3.7} & 53.4 \cdec{7.5} & & 54.0 \cdec{5.7} & 64.0 \cdec{4.5} & 48.9 \cdec{6.4} & 45.8 \cdec{9.8} & 52.0 \cdec{8.3} & 39.7 \cdec{11.3} \\
\midrule
\rowcolor{demogray}
(6) & w/ Boundary SPT (VLPrompt) & 77.6 & 83.5 & 74.5 & 61.7 & 62.5 & 60.9 & & 59.7 & 68.5 & 55.3 & 55.6 & 60.3 & 51.0 \\
\rowcolor{demogray}
(7) & w/ Semantic-aware SPT & 77.1 & 84.1 & 73.6 & 57.8 & 58.8 & 56.9 & & 45.0 & 54.3 & 40.4 & 39.4 & 45.4 & 33.4 \\
\bottomrule
\end{tabular}
}
\end{table*}

\subsection{Ablation Studies}
\label{subsec:ablation}

We conduct comprehensive ablation studies for all three methods to quantify the contribution of each design choice. We first ablate HiLo's core modules, which are shared across all three frameworks, to demonstrate their effectiveness. We then examine each module in HLPrompt and VLPrompt, including new modules introduced by them, respectively. Component ablations are presented on DomainNet-GCD (Real$\to$Painting) and SSB-C (CUB-C). Further analysis (classifier feature choices, backbone strength, and training dynamics) is deferred to Appendix~\ref{app:additional_ablations}.

\noindent\textbf{HiLo.} \
We validate the contribution of HiLo's core components in Table~\ref{tab:hilo_ablation}, starting from SimGCD as the baseline.
Adapting PatchMix for contrastive learning in the GCD setting (row~2) yields a noticeable boost on the shifted domain $\Omega^b$, confirming the benefit of intermediate-domain augmentation.
Adding the domain-semantic feature disentanglement via MI minimization (row~3) significantly improves performance on both $\Omega^a$ and $\Omega^b$, demonstrating that dissociating spurious domain-semantic correlations is essential.
Curriculum sampling (row~4) further improves performance on $\Omega^b$ by gradually exposing the model to harder new-domain samples, leading to the full HiLo model.
The engineering ablations (rows~5--6) confirm that the hi-lo feature hierarchy is critical: using only deep or only shallow features for both branches leads to drastic performance drops, validating the design choice of extracting domain features from early layers and semantic features from late layers.

\noindent\textbf{HLPrompt.} \ 
We ablate each HLPrompt component in Table~\ref{tab:hlprompt_ablation}.
\textit{MI minimization} and \textit{PatchMix} yield consistent gains on the shifted target domains $\Omega^b$, validating their role in domain robustness.
\textit{Alternating training} is crucial for making spatial prompt tuning effective; without it, performance drops sharply on $\Omega^b$.
The preferred SPT variant depends on the dataset: semantic-aware prompts work better on DomainNet-GCD, while boundary prompts are more reliable on CUB-C.
Appendix Figs.~\ref{fig:feature_choice} and~\ref{fig:backbone_comparison} further show that HLPrompt's gains are not solely due to a stronger backbone. A full metric-by-metric analysis is provided in Appendix~\ref{app:ablation_discussion}.

\noindent\textbf{VLPrompt.} \
We ablate each VLPrompt component in Table~\ref{tab:vlprompt_ablation}.
The \textit{vision-language alignment loss} is the most critical component for adapting CLIP to GCD under domain shifts; removing it causes the largest drops on both domains.
\textit{SPT} with \textit{alternating training} is critical for robustness on $\Omega^b$.
\textit{Boundary SPT} is consistently more reliable than semantic-aware SPT for VLPrompt, consistent with the attention-pattern analysis in Appendix~\ref{app:spt_analysis}.
CLIP fine-tuning choices and textual context embedding updates are further ablated in Appendix Figs.~\ref{fig:backbone_parameters} and~\ref{fig:train_ctx_or_not}.

\section{Conclusion}
\label{sec:conclusion}

We study Generalized Category Discovery under domain shifts and propose a suite of three frameworks grounded in DINO-pretrained vision backbones and the CLIP vision-language model.
\textbf{HiLo} achieves domain-semantic disentanglement through multi-level feature extraction and MI minimization, enhanced with PatchMix augmentation and curriculum learning, with provably tighter generalization bounds (Theorem~\ref{thm:mi_bound}).
\textbf{HLPrompt} extends HiLo with NCut~\cite{shi2000normalized}-based spatial prompt tuning and alternating optimization to suppress input-level background and domain noise, completing the pure-vision solution.
\textbf{VLPrompt} adapts vision-language models through factorized textual prompts and cross-modal consistency regularization, inheriting HiLo's core components in the cross-modal setting.
We hope this work encourages the community to consider more realistic open-world assumptions where domain shifts are the norm rather than the exception.
Several limitations point to productive future directions.
All three methods rely on Assumption~\ref{assume:independence} (domain-class independence), which may be violated when certain categories appear predominantly in specific domains; future work could relax this via joint domain-label modeling~\cite{long2018conditional} or causal disentanglement.
Our methods also assume that the number of categories $K$ is known in advance; relaxing this would require principled model selection or non-parametric category estimation~\cite{han2019learning}.
Furthermore, extremely large domain gaps (e.g., RGB to depth or thermal imagery) remain challenging, suggesting multi-modal pretraining as a promising avenue.

An expanded discussion of complementary strengths, VLM generalization versus memorization~\cite{mayilvahanan2024forgotten,mayilvahanan2023clip}, and potential future directions are provided in Appendix~\ref{app:discussion}.

\section*{Acknowledgements}
This work is supported by Hong Kong Research Grant Council - Early Career Scheme (Grant No. 27208022), Hong Kong Research Grant Council - General Research Fund (Grant No. 17211024), 
and HKU Seed Fund for Basic Research.

\FloatBarrier
\bibliographystyle{IEEEtran}
\bibliography{ref}

\clearpage
\appendix

\renewcommand{\thefigure}{A\arabic{figure}}
\renewcommand{\thetable}{A\arabic{table}}
\renewcommand{\theHfigure}{appendix.\arabic{figure}}
\renewcommand{\theHtable}{appendix.\arabic{table}}
\setcounter{figure}{0}
\setcounter{table}{0}

This appendix is organized as follows. Section~\ref{app:discussion} discusses the complementary strengths of the three frameworks, the role of VLM generalization, and the remaining limitations. Section~\ref{app:impl_details} summarizes implementation details and training algorithms. Section~\ref{app:additional_ablations} reports additional results, including component ablations and the analysis of semantic-aware SPT. Section~\ref{app:formal_defs} states the formal definitions and assumptions. Section~\ref{app:proofs} provides proofs and technical details, including the complexity analysis and the NCut formulation.

\subsection{Extended Discussion}
\label{app:discussion}

\subsubsection{Complementary Strengths of the Three Frameworks}

We propose a suite of three complementary frameworks that share the same core ingredients, including disentanglement, PatchMix, and curriculum sampling, while operating on different foundation backbones and introducing different adaptation mechanisms. As a result, they are suitable for different deployment constraints and domain-shift regimes.

HiLo establishes the foundational pure-vision formulation of our suite. It disentangles domain-specific and semantic features through multi-level extraction and MI minimization, and it further combines PatchMix augmentation with curriculum sampling for stable optimization. This design is especially attractive when one prefers a pure-vision pipeline with explicit control over feature disentanglement.

HLPrompt extends this pure-vision path from the perspective of prompt-based adaptation. By selectively injecting prompts into foreground object regions, it complements HiLo's feature-level disentanglement with input-level adaptation. This makes HLPrompt a natural choice when foreground localization and prompt-based specialization are beneficial.

VLPrompt approaches the same problem from the perspective of vision-language foundation models. It leverages aligned visual and textual representations to provide richer semantic structure for category discovery, while retaining the domain-aware principles introduced by HiLo. This setting is particularly appealing when cross-modal priors and parameter-efficient adaptation are available.

These complementary characteristics suggest that method selection should be guided by application requirements. The pure-vision path of HiLo and HLPrompt is well suited to settings that prioritize architectural flexibility, interpretability, or deployment without a text branch. VLPrompt is well suited to settings that can benefit from textual semantics and broader cross-modal prior knowledge.

\subsubsection{On VLM Generalization versus Data Memorization}

A natural question for VLM-based methods is whether strong performance under domain shift reflects genuine transfer or simple memorization of similar web examples. Recent systematic studies~\cite{mayilvahanan2024forgotten,mayilvahanan2023clip} provide two useful observations. First, CLIP retains substantial performance even after aggressive similarity-based pruning of its training data, which indicates that it learns transferable representations beyond direct pattern matching. Second, genuine OOD generalization remains challenging, which makes careful evaluation essential when claiming robustness.

VLPrompt is designed with this perspective in mind. Rather than relying on zero-shot predictions alone, it adapts the pretrained model through task-specific prompt learning, spatial prompting, and domain-aware regularization. Our benchmarks also introduce controlled domain shifts whose structure differs from the natural diversity of web data. Together with the ablation results, this supports the view that VLPrompt's gains come from the proposed adaptation mechanisms rather than from pretraining alone.

\subsubsection{Limitations and Future Work}

Current limitations include: the domain-class independence assumption (Assumption~\ref{assume:independence}), which may be violated when categories correlate with specific domains---a problem that joint domain-label modeling or causal approaches could address; the need to know the total number of categories $K$ in advance, which could be relaxed via non-parametric category estimation; the difficulty of handling extremely large domain gaps such as RGB to depth imagery, where multi-modal pretraining may help; and the computational cost of training on large-scale datasets.

Promising directions for future work include: (i) a unified framework combining pure-vision disentanglement with VLM-based cross-modal alignment; (ii) automatic estimation of the number of novel classes; (iii) extension to continual settings with streaming domains and categories; (iv) incorporation of additional modalities (audio, depth, text metadata); and (v) tighter theoretical guarantees under weaker assumptions beyond domain-class independence.

\subsection{Implementation Details}
\label{app:impl_details}

We summarize the implementation details for HLPrompt and VLPrompt in this section.

\subsubsection{Datasets}

SSB-C contains three fine-grained recognition datasets: CUB-200~\cite{wah2011caltech} with 200 bird species and 11,788 images, Stanford Cars~\cite{krause20133d} with 196 car models and 16,185 images, and FGVC-Aircraft~\cite{maji2013fine} with 100 aircraft variants and 10,000 images. For each dataset, the labelled set consists of clean images from 50\% of the base classes, while the unlabelled set contains a mixture of clean and corrupted images spanning all classes. We consider five severity levels and a diverse set of noise, blur, weather, and digital corruptions.

DomainNet-GCD contains 345 object categories across six visually distinct domains, namely Real, Clipart, Infograph, Painting, Quickdraw, and Sketch. The protocol uses labelled data from the Real domain for 50\% of the base classes, while the unlabelled set spans all six domains and all classes.

Statistics of the datasets are recorded in Table~\ref{tab:dataset_stats}.

\begin{table}[t]
\centering
\caption{Statistics of the datasets.}
\label{tab:dataset_stats}
\setlength{\tabcolsep}{3.5pt}
\resizebox{\columnwidth}{!}{
\begin{tabular}{lcccccc}
\toprule
 & \multicolumn{3}{c}{Labelled} & \multicolumn{3}{c}{Unlabelled} \\
\cmidrule(lr){2-4} \cmidrule(lr){5-7}
Dataset & \#Image & \begin{tabular}[c]{@{}c@{}}\#Class\\ $|\mathcal{Y}_{base}|$\end{tabular} & \begin{tabular}[c]{@{}c@{}}\#Domain\\ $|\Omega^a|$\end{tabular} & \#Image & \begin{tabular}[c]{@{}c@{}}\#Class\\ $|\mathcal{Y}_{novel}|$\end{tabular} & \begin{tabular}[c]{@{}c@{}}\#Domain\\ $|\Omega|$\end{tabular} \\
\midrule
DomainNet & 39.1K & 172 & 1 & 547.5K & 345 & 6 \\
CUB-C     & 1.5K  & 100 & 1 & 45K    & 200 & 10 \\
Scars-C   & 2.0K  & 98  & 1 & 61K    & 196 & 10 \\
FGVC-C    & 1.7K  & 50  & 1 & 50K    & 100 & 10 \\
\bottomrule
\end{tabular}
}
\end{table}

\subsubsection{HLPrompt}

We employ ViT-B/16 pretrained with DINOv3~\cite{simeoni2025dinov3} as the backbone. The projection heads are 3-layer MLPs (768-2048-2048-256), and the MI discriminator is a 3-layer MLP (256-2048-2048-1). We generate $M=2$ augmented views per image. For DomainNet-GCD, we use random crop, random horizontal flip, color jitter, and random grayscale. For SSB-C, we use random crop, random horizontal flip, and color jitter.

For curriculum sampling on SSB-C, we compute the sampling weight $p_{cs}(\mathbf{x}|t)$ inspired by FREE~\cite{feng2025free}: we apply 2D FFT, use the amplitude spectrum as domain representation, downsample with average pooling, and run semi-supervised K-Means to partition samples into $\Omega^a$ and $\Omega^b$.

Detailed optimization and hyperparameter settings of HLPrompt are recorded in Table~\ref{tab:hyperparameters} (a).

\subsubsection{VLPrompt}

VLPrompt builds on pretrained CLIP~\cite{radford2021learning} ViT-B/16. The MI discriminator is a 3-layer MLP (512-2048-2048-1). We learn 4 context tokens initialized from ``a photo of a'' and $K$ category embeddings with 4 tokens each, where $K$ equals the total number of classes in each dataset. Category embeddings are initialized from $\mathcal{N}(0, 0.02^2 \mathbf{I})$.

Augmentation and curriculum sampling settings follow those of HLPrompt. Detailed optimization and hyperparameter settings of VLPrompt are recorded in Table~\ref{tab:hyperparameters} (b).

\begin{table}[t]
\centering
\caption{Optimization details and hyperparameters.}
\label{tab:hyperparameters}

\vspace{2pt}
\textbf{(a) HLPrompt} \\
\vspace{2pt}
\resizebox{\columnwidth}{!}{
\begin{tabular}{lc}
\toprule
\textbf{Setting} & \textbf{Value} \\
\midrule
\textit{Optimization} & \\
Optimizer & SGD \\
Epochs & 200 \\
Batch Size (DomainNet / SSB-C) & 128 / 150 \\
Weight Decay & $5\times 10^{-5}$ \\
\midrule
\textit{Model-specific} & \\
\quad $\epsilon_s$ & 2 \\
\quad $\epsilon_d$ & 0.1 \\
\quad $\tau$ & 0.07 \\
\quad $k$ & 20 \\
\midrule
\textit{Curriculum sampling} & \\
DomainNet & $r'=1, t'=80$ \\
SSB-C & $r_0=0, r'=0.05, t'=80$ \\
\bottomrule
\end{tabular}
}

\vspace{12pt}

\textbf{(b) VLPrompt} \\
\vspace{2pt}
\resizebox{\columnwidth}{!}{
\begin{tabular}{lc}
\toprule
\textbf{Setting} & \textbf{Value} \\
\midrule
\textit{Optimization} & \\
Optimizer & SGD \\
Epochs & 200 \\
Batch Size (DomainNet / SSB-C) & 128 / 150 \\
Weight Decay & $5\times 10^{-4}$ \\
\midrule
\textit{Model-specific} & \\
\quad $\epsilon$ & 2 \\
\quad $\beta_1$ & 0.5 \\
\quad $\beta_2$ & 0.4 \\
\quad $\tau$ & 0.007 \\
\quad $k$ & 20 \\
\midrule
\textit{Curriculum sampling} & \\
DomainNet & $r'=1, t'=80$ \\
SSB-C & $r_0=0, r'=0.05, t'=80$ \\
\bottomrule
\end{tabular}
}
\end{table}

\subsubsection{Training Algorithms}
\label{app:algorithms}

We next summarize the full training procedures of HLPrompt and VLPrompt for completeness. Both algorithms follow the same high-level principle of curriculum-based mini-batch sampling and alternating optimization, while differing in the underlying backbone, prompt parameterization, and cross-modal alignment components. The pseudo-code below makes explicit how the proposed modules are integrated into each training iteration.

\begin{styledalg}[float=t, label={alg:hilo}]{HLPrompt Training Algorithm}
\begin{algorithmic}[1]
\STATE \textbf{Input:} Labeled data $\mathcal{D}_\ell$, unlabeled data $\mathcal{D}_u$, hyperparameters $\epsilon_s, \epsilon_d, \tau, r_0, r', t', k$
\STATE \textbf{Initialize:} ViT encoder $\mathcal{T}$, projection heads $\psi_{\text{dom}}, \psi_{\text{sem}}$, discriminator $D$, spatial prompts $Q_s$.
\STATE Define parameter sets: $\Theta_{\text{model}} = \{ \mathcal{T}, \psi_{\text{dom}}, \psi_{\text{sem}}, D \}$, $\Theta_{\text{prompt}} = \{ Q_s \}$
\STATE Set optimization phase flag: $\text{phase} \leftarrow 1$

\FOR{epoch $t = 0$ to $T_{\text{max}}$}
    \FOR{mini-batch $b = 0$ to $B_{max}$}
        \IF{$(b + 1) \pmod k == 0$}
            \STATE $\text{phase} \leftarrow 1 - \text{phase}$
        \ENDIF

        \IF{$\text{phase} == 0$}
            \STATE $\Theta_{\text{active}} \leftarrow \Theta_{\text{model}}$; Freeze $\Theta_{\text{prompt}}$
        \ELSE
            \STATE $\Theta_{\text{active}} \leftarrow \Theta_{\text{prompt}}$; Freeze $\Theta_{\text{model}}$
        \ENDIF

        \STATE \textbf{Data Sampling:}
        \STATE Sample $\mathcal{B}_\ell \sim \mathcal{D}_\ell$ uniformly; $\mathcal{B}_u \sim \mathcal{D}_u$ with weights $p(\mathbf{x}|t)$

        \STATE \textbf{Forward Propagation:}
        \FOR{each $\mathbf{x} \in \mathcal{B}_\ell \cup \mathcal{B}_u$}
            \STATE Compute foreground mask $\mathbf{M}$ and apply prompts:
            \STATE \quad $\tilde{\mathbf{x}} = \mathbf{x} + \mathbf{M} \odot Q_s$
            \STATE Apply PatchMix augmentation to $\tilde{\mathbf{x}}$
            \STATE Extract features: $\mathbf{r}_{\text{dom}} = \mathcal{T}_1(\phi(\tilde{\mathbf{x}}))$, $\mathbf{r}_{\text{sem}} = \mathcal{T}_L(\phi(\tilde{\mathbf{x}}))$
            \STATE Project: $\mathbf{h}_{\text{dom}} = \psi_{\text{dom}}(\mathbf{r}_{\text{dom}})$, $\mathbf{h}_{\text{sem}} = \psi_{\text{sem}}(\mathbf{r}_{\text{sem}})$
        \ENDFOR

        \STATE \textbf{Optimization:}
        \STATE Compute loss: $\mathcal{L}_{\text{total}} = \mathcal{L}_{\text{s}} + \epsilon_{s} \Delta_{s} + \mathcal{L}_{\text{d}}+ \epsilon_{d} \Delta_{d} + \mathcal{L}_{\text{MI}} + \mathcal{L}_{\text{PM}}$
        \STATE Update $\Theta_{\text{active}}$ via gradient descent on $\nabla_{\Theta_{\text{active}}} \mathcal{L}_{\text{total}}$
    \ENDFOR
\ENDFOR
\STATE \textbf{Return:} Trained $\Theta_{\text{prompt}}$ and $\Theta_{\text{model}}$
\end{algorithmic}
\end{styledalg}

\begin{styledalg}[float=t, label={alg:vlprompt}]{VLPrompt Training Algorithm}
\begin{algorithmic}[1]
\STATE \textbf{Input:} Labeled data $\mathcal{D}_\ell$, unlabeled data $\mathcal{D}_u$, hyperparameters $\epsilon, \beta_1, \beta_2, \tau, r_0, r', t', k$
\STATE \textbf{Initialize:} CLIP's image encoder $\Phi^{(v)}$ and text encoder $\Phi^{(t)}$; Context $\Theta_{\text{ctx}}$, Category $\Gamma$; Discriminator $D$; Foreground prompt $Q_s$.
\STATE \textbf{Define Parameter Sets:}
\STATE \quad $\Theta_{\text{model}} = \{ \Phi, \Theta_{\text{ctx}}, \Gamma, D \}$
\STATE \quad $\Theta_{\text{prompt}} = \{ Q_s \}$
\STATE Set optimization phase flag: $\text{phase} \leftarrow 0$

\FOR{epoch $t = 0$ to $T_{\text{max}}$}
    \FOR{mini-batch $b = 0$ to $B_{max}$}
        \IF{$(b + 1) \pmod k == 0$}
            \STATE $\text{phase} \leftarrow 1 - \text{phase}$
        \ENDIF

        \IF{$\text{phase} == 0$}
            \STATE $\Theta_{\text{active}} \leftarrow \Theta_{\text{model}}$; Freeze $\Theta_{\text{prompt}}$
        \ELSE
            \STATE $\Theta_{\text{active}} \leftarrow \Theta_{\text{prompt}}$; Freeze $\Theta_{\text{model}}$
        \ENDIF

        \STATE \textbf{Data Sampling:}
        \STATE Sample $\mathcal{B}_\ell \sim \mathcal{D}_\ell$ uniformly; $\mathcal{B}_u \sim \mathcal{D}_u$ with weights $p_{cs}(\mathbf{x}|t)$

        \STATE \textbf{Forward Propagation:}
        \FOR{each $\mathbf{x} \in \mathcal{B}_\ell \cup \mathcal{B}_u$}
            \STATE Apply foreground mask $\mathbf{M}$ and inject prompts:
            \STATE \quad $\tilde{\phi}(\mathbf{x}) = \phi(\mathbf{x}) + \mathbf{M} \odot Q_s$
            \STATE Apply PatchMix augmentation to the prompted tokens $\tilde{\phi}(\mathbf{x})$
            \STATE Extract features: $\mathbf{r}_{\text{dom}} = \mathbf{W}\cdot\mathcal{F}_1(\tilde{\phi}(\mathbf{x}))$, $\mathbf{r}_{\text{sem}} = \mathbf{W}\cdot\mathcal{F}_L(\tilde{\phi}(\mathbf{x}))$
            \STATE Construct text embeddings: $\mathbf{S} = \text{concat}([\Theta_{\text{ctx}}, \Gamma])$
            \STATE Normalize: $\pi^{(v)} = \mathbf{r}_{\text{sem}} / \|\mathbf{r}_{\text{sem}}\|_2$, $\quad \pi^{(t)} = \mathbf{S} / \|\mathbf{S}\|_2$
            \STATE Compute matching score: $\rho = \pi^{(v)} \cdot \pi^{(t)}$
        \ENDFOR

        \STATE \textbf{Optimization:}
        \STATE Compute loss: $\mathcal{L}_{\text{total}} = \mathcal{L}^{cls} + \epsilon\Delta + \beta_1\mathcal{L}_{\text{MI}} + \beta_2\mathcal{L}^{vl} + \mathcal{L}_{\text{PM}}$
        \STATE Update $\Theta_{\text{active}}$ via gradient descent on $\nabla_{\Theta_{\text{active}}} \mathcal{L}_{\text{total}}$
    \ENDFOR
\ENDFOR
\STATE \textbf{Return:} Trained prompts $Q_s$, vision encoder $\Phi^{(v)}$, and embeddings $\Theta_{\text{ctx}}, \Gamma$
\end{algorithmic}
\end{styledalg}

\subsection{Additional Results}
\label{app:additional_ablations}

\subsubsection{Detailed Discussion of Component Ablations}
\label{app:ablation_discussion}

For HLPrompt on DomainNet-GCD Real$\to$Painting (Table~\ref{tab:hlprompt_ablation}, left), removing MI minimization reduces performance on most metrics, with the largest degradation on domain $\Omega^b$ (`Painting'), where the `New' metric drops by $2.6\%$. Removing PatchMix augmentation also consistently degrades performance, including a $4.3\%$ drop on `All' and a $4.9\%$ drop on `New' in $\Omega^b$, which highlights the benefit of intermediate-domain augmentation. Removing alternating training causes only minor changes on $\Omega^a$ (`Real') but large drops on $\Omega^b$, showing that it is important for effective spatial prompt adaptation. Removing SPT slightly improves $\Omega^a$ while substantially harming robustness on $\Omega^b$. Among SPT variants, semantic-aware prompts are preferable to boundary prompts on DomainNet-GCD.

For HLPrompt on CUB-C (Table~\ref{tab:hlprompt_ablation}, right), removing MI minimization causes severe degradation on all metrics. During training, performance on the `Original' domain can initially match the reference, but it deteriorates after corrupted samples are introduced, which indicates that semantic alignment collapses without the disentanglement constraint. Removing PatchMix leads to large drops across metrics, especially on corrupted samples. Removing alternating training also leads to severe degradation, and using spatial prompts without alternating training performs worse than removing spatial prompts entirely. On this benchmark, boundary prompts are more reliable than semantic-aware prompts.

For VLPrompt on DomainNet-GCD Real$\to$Painting (Table~\ref{tab:vlprompt_ablation}, left), MI minimization improves generalization to $\Omega^b$ with a $1.0\%$ gain on $\Omega^b$ `New'. The vision-language loss is the most important component for adapting VLMs to GCD, and removing it causes the largest drops on both domains. Cross-modal PatchMix improves generalization to $\Omega^b$ while causing only minor trade-offs on $\Omega^a$. Removing alternating training reduces performance on $\Omega^b$ by 3.6\% on average, and removing SPT produces even larger drops. Boundary SPT is consistently more reliable than semantic-aware SPT for VLPrompt.

For VLPrompt on CUB-C (Table~\ref{tab:vlprompt_ablation}, right), the overall pattern is similar. The vision-language loss remains the most critical component, cross-modal PatchMix improves performance across metrics, and both alternating training and SPT are essential for robustness on the shifted domain. Boundary SPT again outperforms semantic-aware SPT by a large margin, which is consistent with the supplemental figures below and the mask visualizations analyzed in Section~\ref{app:spt_analysis}.

\begin{figure}[!t]
    \centering
    \includegraphics[width=\linewidth]{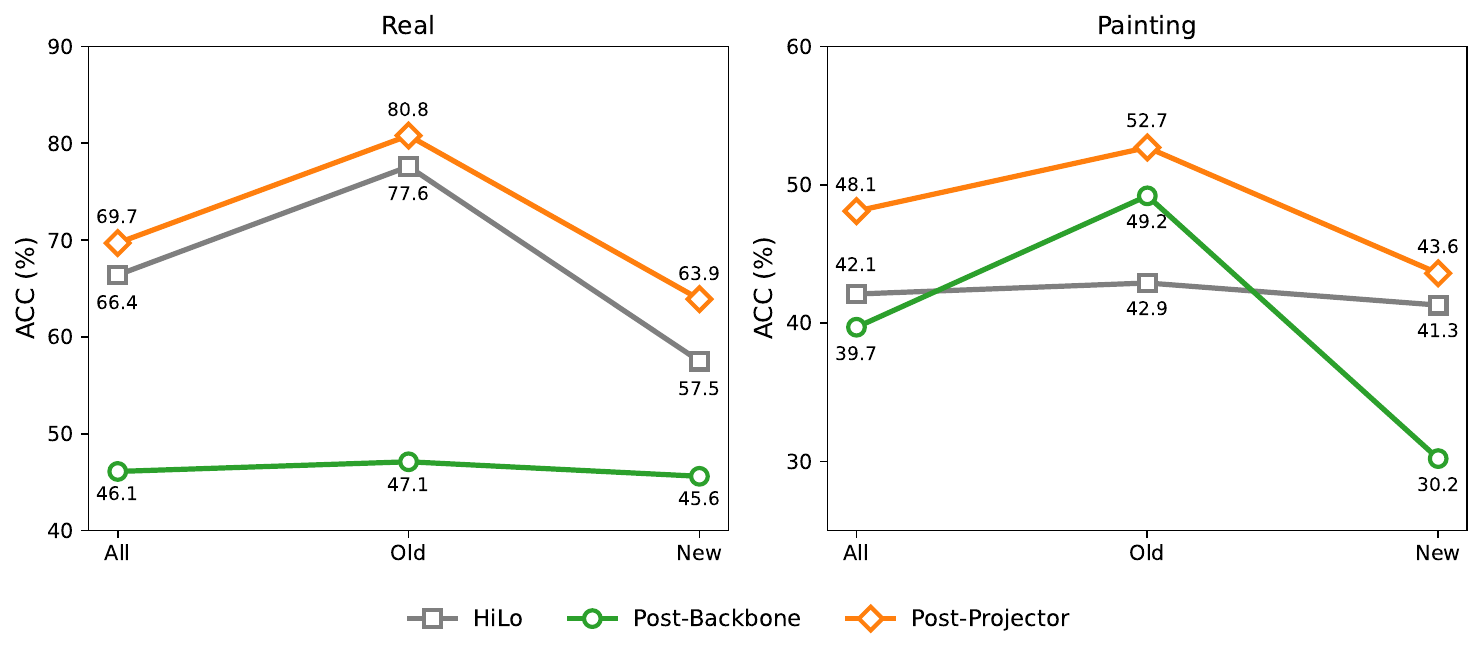}
    \caption{Ablation study on which feature representation for classification. We compare the performance of using features after the backbone (Post-Backbone) versus after the projection head (Post-Projector) on DomainNet-GCD 'Real' and 'Painting' domains. HiLo is included as a baseline.}
    \label{fig:feature_choice}
\end{figure}

\begin{figure}[!t]
    \centering
    \includegraphics[width=\linewidth]{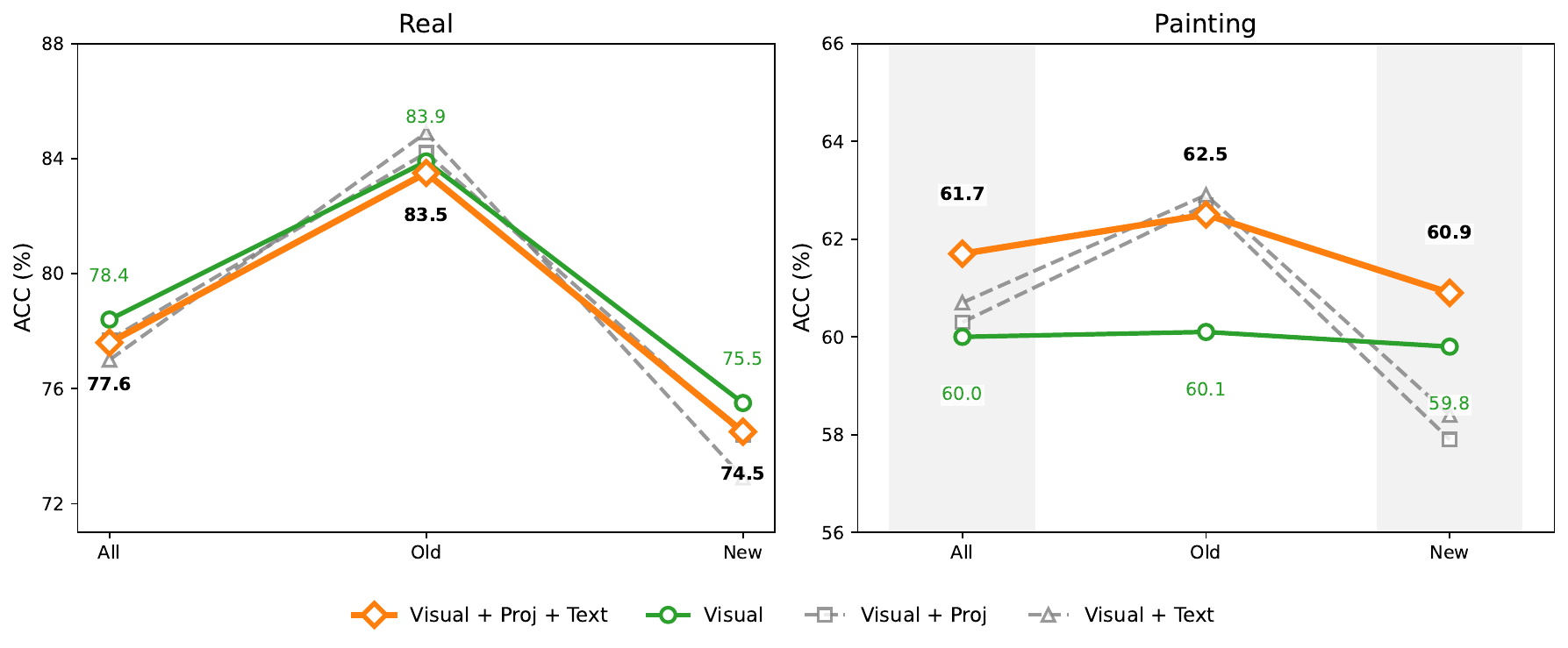}
    \caption{Comparison between which backbone parameters of the CLIP backbone to fine-tune on DomainNet-GCD 'Real' and 'Painting' domains.}
    \label{fig:backbone_parameters}
\end{figure}

\begin{figure*}[!t]
    \centering
    \includegraphics[width=\linewidth]{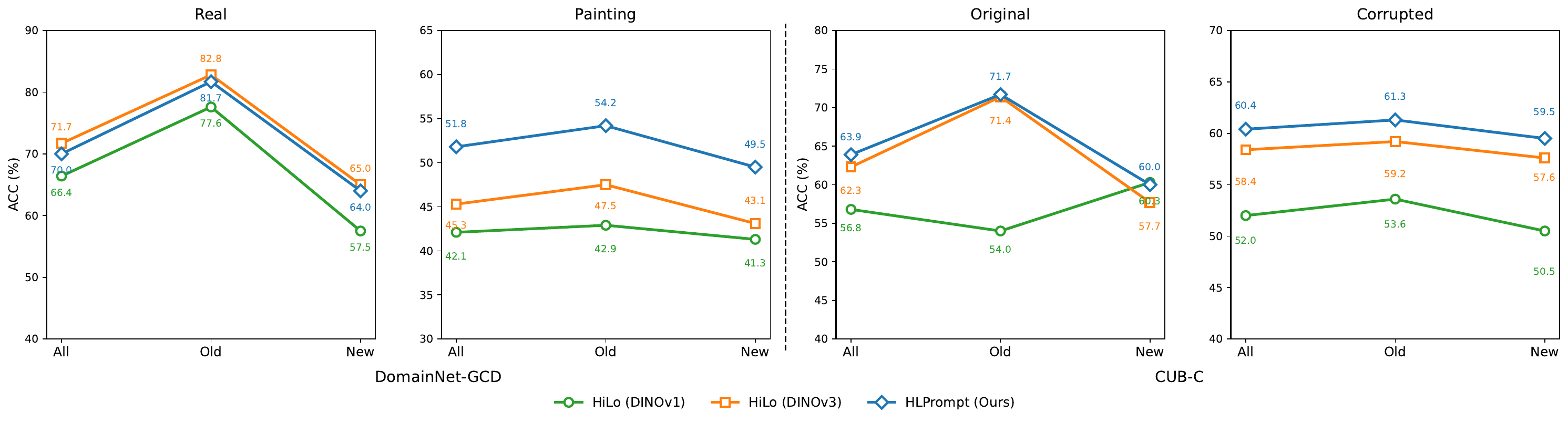}
    \caption{Comparison between DINOv1 and DINOv3 backbone. We compare the performance gain from direct replacement of a stronger backbone on the HiLo baseline on DomainNet-GCD Real-Painting and CUB-C. We include HLPrompt results to present further performance gain.}
    \label{fig:backbone_comparison}
\end{figure*}

\begin{figure}[!t]
    \centering
    \includegraphics[width=\linewidth]{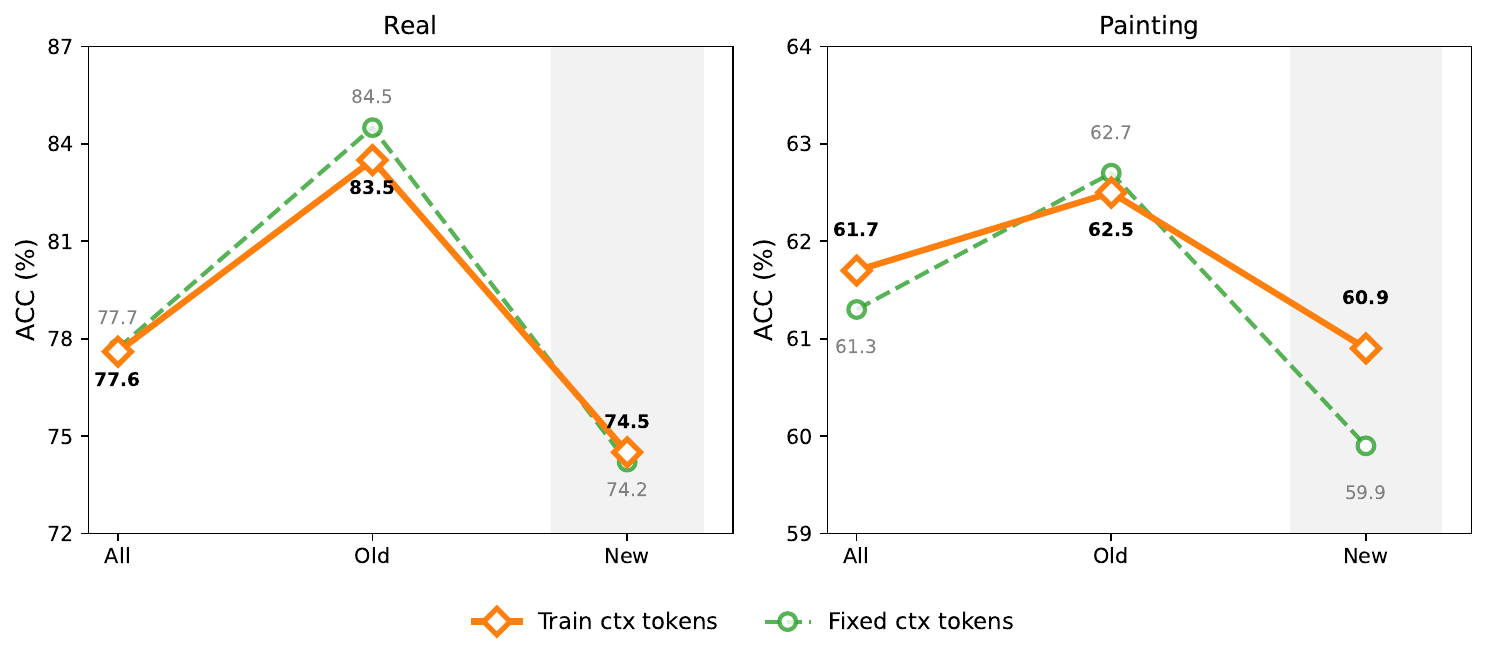}
    \caption{Comparison between fine-tuning and fixing textual context embeddings $\Theta_{\text{ctx}}$ on DomainNet-GCD 'Real' and 'Painting' domains.}
    \label{fig:train_ctx_or_not}
\end{figure}

\subsubsection{Analysis on the Effectiveness of Semantic-aware SPT}
\label{app:spt_analysis}

To understand the different behavior of semantic-aware SPT in HLPrompt and VLPrompt, we visualize the NCut~\cite{shi2000normalized}-based masks produced by DINOv3 ViT and CLIP ViT together with the final-layer [CLS] attention maps. DomainNet-GCD examples are shown in Fig.~\ref{fig:domainnet_mask_vis}, and CUB-C examples are shown in Fig.~\ref{fig:cubc_mask_vis}.

\begin{figure*}[!t]
    \centering
    \includegraphics[width=\linewidth]{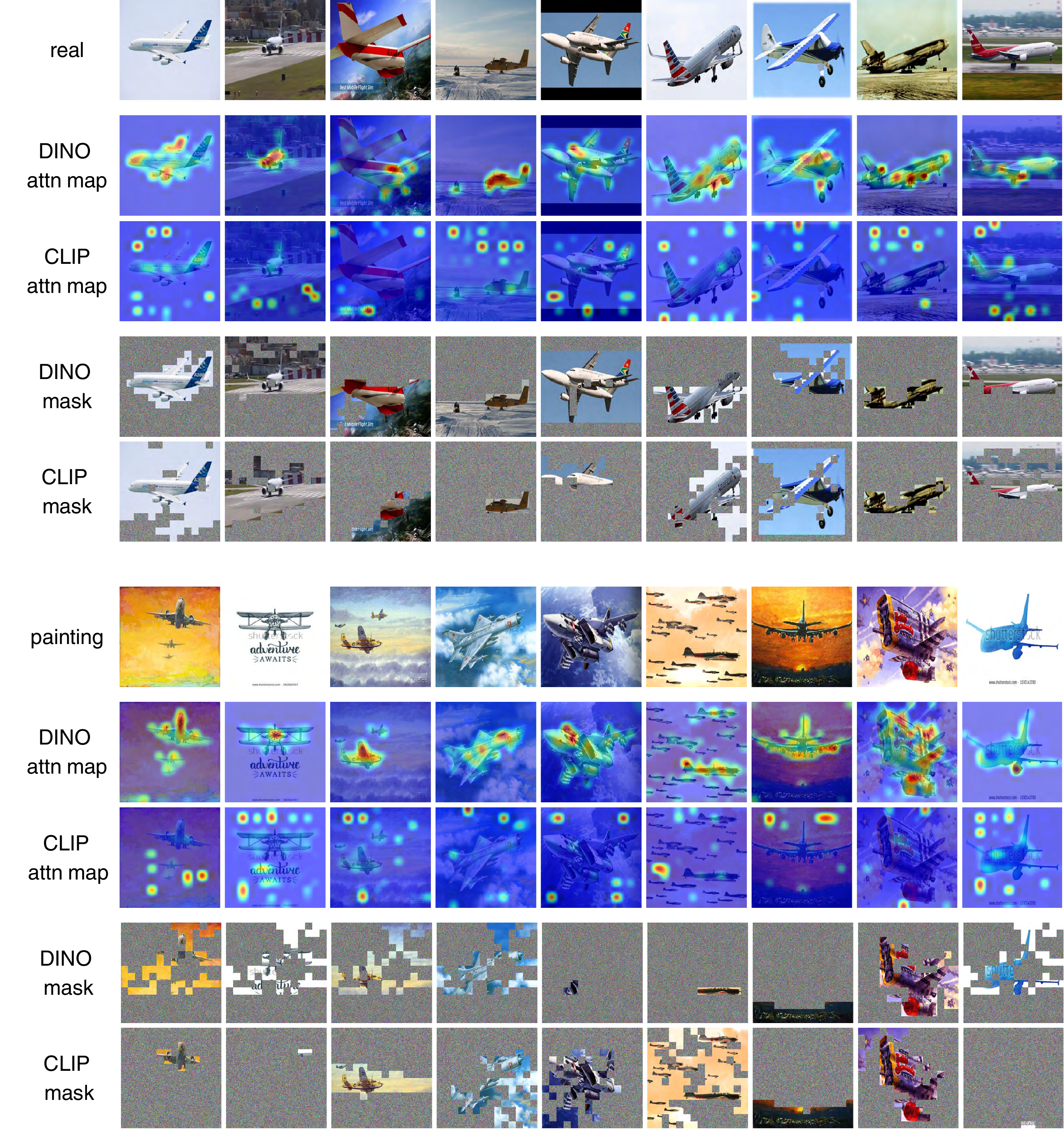}
    \caption{Visualization of last-layer [CLS] attention maps and NCut~\cite{shi2000normalized}-based foreground masks produced by DINOv3 ViT and CLIP ViT. Samples are randomly drawn from DomainNet-GCD with the \texttt{Real} domain (top row) and the \texttt{Painting} domain (bottom row).}
    \label{fig:domainnet_mask_vis}
\end{figure*}

\begin{figure*}[!t]
    \centering
    \includegraphics[width=\linewidth]{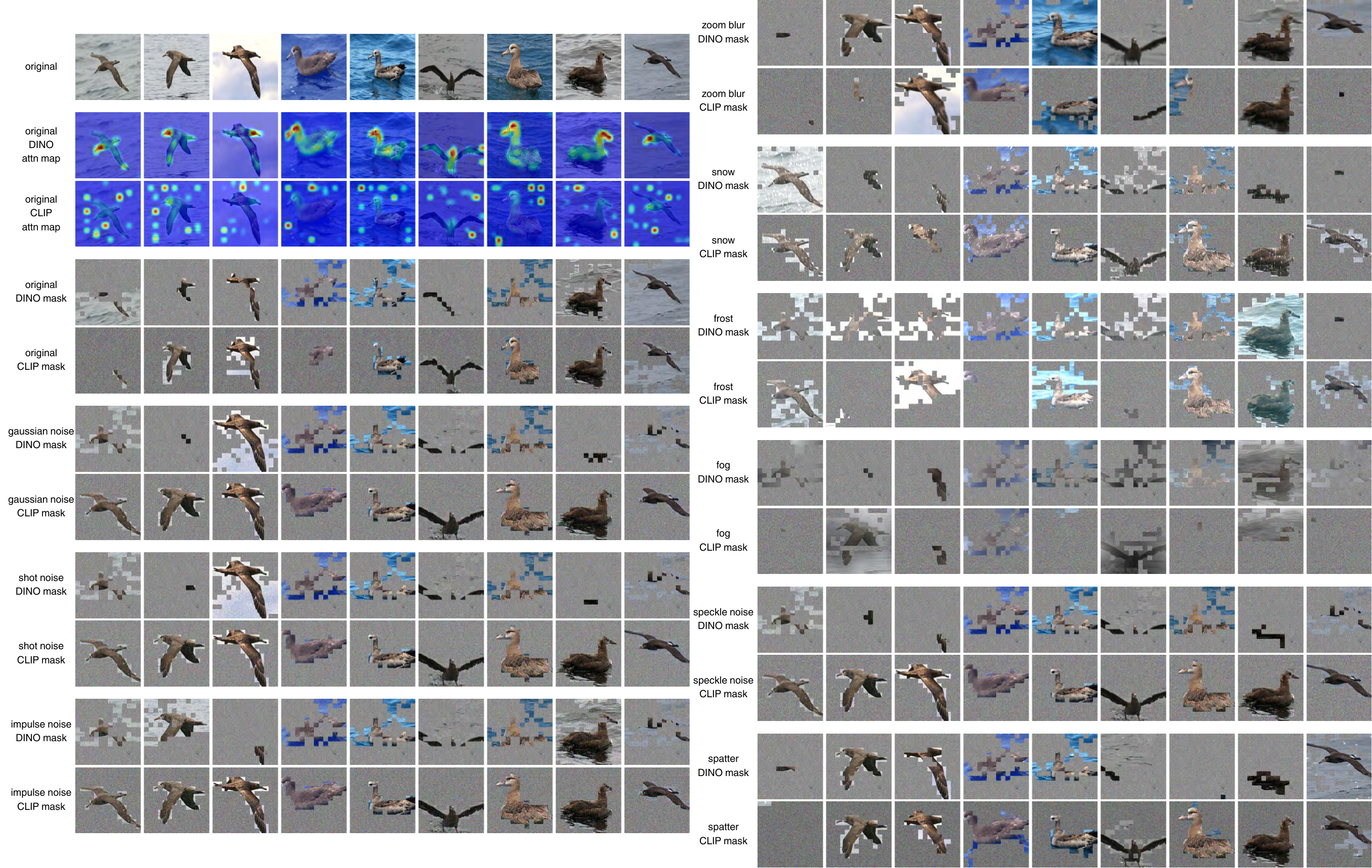}
    \caption{Visualization of last-layer [CLS] attention maps and NCut~\cite{shi2000normalized}-based foreground masks produced by DINOv3 ViT and CLIP ViT. Samples are randomly drawn from CUB-C and cover all nine corruption types.}
    \label{fig:cubc_mask_vis}
\end{figure*}

On DomainNet-GCD, both backbones produce masks of comparable visual quality, but their attention patterns differ. DINO concentrates attention on the foreground object, which makes foreground-focused prompting compatible with the pretrained representation. CLIP distributes attention more broadly across the image in order to preserve global information for cross-modal alignment. As a result, masking and prompting large foreground or background regions can disturb the image-text geometry more strongly in VLPrompt, which explains why boundary prompting is more reliable in that setting.

On CUB-C, CLIP often produces visually cleaner masks than DINO under severe corruptions. However, this does not make semantic-aware SPT preferable for VLPrompt. The same strong masks can remove or alter image regions that contribute to CLIP's global alignment behavior. For HLPrompt, the results also indicate that semantic-aware prompting becomes less reliable when corruptions distort the object itself. These observations explain why boundary prompting becomes the more robust choice on CUB-C for both methods.

\subsection{Formal Definitions and Assumptions}
\label{app:formal_defs}

Before presenting our theoretical results, we formalize the problem setup and state the assumptions used throughout the paper.
These definitions precisely specify the GCD problem with domain shifts and the conditions under which our theoretical guarantees hold.

\begin{definition}[Domain]
A domain $\omega \in \Omega$ is characterized by a conditional distribution $P(\mathbf{x}|y, \omega)$ over inputs given labels. Two domains $\omega, \omega' \in \Omega$ differ if $P(\mathbf{x}|y, \omega) \neq P(\mathbf{x}|y, \omega')$ for some $y \in \mathcal{Y}$.
\end{definition}

\begin{definition}[GCD with Domain Shifts]
We are given a labelled dataset $\mathcal{D}_{\ell} = \{(\mathbf{x}_i, y_i)\}_{i=1}^{N_\ell}$ where each image $\mathbf{x}_i$ is drawn from $P(\mathbf{x}|y_i, \omega^a)$ for a fixed labelled domain $\omega^a \in \Omega^{a}$, with labels $y_i$ belonging to the base category set $\mathcal{Y}_{\text{base}} \subset \mathcal{Y}$.
We are also given an unlabelled mixed dataset $\mathcal{D}_{u} = \{\mathbf{x}_i\}_{i=1}^{N_u}$ where images $\mathbf{x}_i$ are drawn from $P(\mathbf{x}|y_i, \omega_i)$ with domains $\omega_i$ coming from either the labelled domain set $\Omega^{a}$ or a disjoint set of \emph{new domains} $\Omega^{b}$ (i.e., $\Omega^{a} \cap \Omega^{b} = \emptyset$), and with labels $y_i$ from either $\mathcal{Y}_{\text{base}}$ or $\mathcal{Y}_{\text{novel}}$ where $\mathcal{Y}_{\text{base}} \cap \mathcal{Y}_{\text{novel}} = \emptyset$.
The objective is to learn a classifier $f: \mathcal{X} \rightarrow \mathcal{Y}_{\text{base}} \cup \mathcal{Y}_{\text{novel}}$ that minimizes the expected error on unlabelled data:
$\min_f \mathbb{E}_{(\mathbf{x}, y, \omega) \sim P_{u}}[\mathds{1}(f(\mathbf{x}) \neq y)]$,
where $P_u$ denotes the joint distribution over unlabelled samples, their true labels, and their domains.
\end{definition}

\begin{assumption}[Domain-Class Independence]
\label{assume:independence}
The class label distribution is domain-invariant: $P(y|\omega) = P(y)$ for all $\omega \in \Omega$.
This ensures that domain shifts are purely covariate shifts in $P(\mathbf{x}|y, \omega)$ rather than semantic shifts.
\end{assumption}

\begin{assumption}[Semantic Transferability]
\label{assume:transfer}
Features useful for discriminating classes in $\mathcal{Y}_{\text{base}}$ are also useful for discriminating classes in $\mathcal{Y}_{\text{novel}}$.
Formally, there exists a feature extractor $\phi: \mathcal{X} \rightarrow \mathcal{Z}$ such that if $\phi$ separates $\mathcal{Y}_{\text{base}}$ in feature space $\mathcal{Z}$, it also separates $\mathcal{Y}_{\text{novel}}$.
\end{assumption}

\begin{assumption}[Sufficient Diversity]
\label{assume:diversity}
$|\mathcal{Y}_{\text{base}}| \geq K_{\min}$ for some minimum $K_{\min}$ to learn generalizable features, and unlabelled data contains sufficient samples per novel class for clustering: $\min_{y \in \mathcal{Y}_{\text{novel}}} |\{\mathbf{x} : y(\mathbf{x}) = y, \mathbf{x} \in \mathcal{D}_u\}| \geq M_{\min}$.
\end{assumption}

\subsection{Proofs and Technical Details}
\label{app:proofs}

This section provides the detailed proofs and technical analyses referenced in the main paper.
Each subsection begins with a brief motivation that explains how the result connects to the method design.

\subsubsection{Pure-Vision HiLo: Target-Risk Bounds via \texorpdfstring{$\mathcal{A}$}{A}-Distance}
\label{app:pv_hilo_bounds}

A core premise of HiLo is that disentangling domain-specific from semantic features reduces the distributional gap between labelled and unlabelled data, thereby improving classification on the unlabelled (target) set.
To make this argument precise, we appeal to the $\mathcal{A}$-distance framework of Ben-David et al.~\cite{ben2006analysis,ben2010theory}, which provides upper bounds on target risk in terms of source risk and a measurable domain divergence.
The results below formalize the intuition that smaller domain divergence---achieved by HiLo's disentanglement---leads to tighter target-risk guarantees.

Let $\mathcal{D}$ be an open-world dataset consisting of a labelled set
$\mathcal{D}^l=\{ (\boldsymbol{x}_i, y_i)\}_{i=1}^{N_l} \subset \mathcal{X}^l \times \mathcal{Y}^l$
and an unlabelled set $\mathcal{D}^u=\{\boldsymbol{x}_i\}_{i=1}^{N_u} \subset \mathcal{X}^u$.
We view $\mathcal{D}^l$ as samples drawn from a \emph{source-domain} distribution with density $p_s(X,Y)$,
and $\mathcal{D}^u$ as drawn from an \emph{unlabelled mixture} that includes both $p_s(X)$ and a target-domain marginal $p_t(X)$.
Let $g$ be a hypothesis (classifier) from a hypothesis class $\mathbb{G}$ and let $\ell$ denote the $0$--$1$ loss.
The \emph{target error} is defined by
\begin{equation}
e_{\mathcal{D}^u}(g) \coloneqq \mathbb{E}_{(x, y) \sim \mathcal{D}^u}[\ell(g(x), y)].
\label{eq:pv_targetrisk}
\end{equation}
More generally, if labels are induced by a (possibly unknown) labelling function $g'$, we define
\begin{equation}
e_{\mathcal{D}^u}(g,g') \coloneqq \mathbb{E}_{x \sim \mathcal{D}^u}[\ell(g(x), g'(x))].
\label{eq:pv_targetrisk_generalized}
\end{equation}
Similarly, the corresponding source errors are $e_{\mathcal{D}^l}(g)$ and $e_{\mathcal{D}^l}(g,g')$.

Following Ben-David et al.~\cite{ben2006analysis}, the $\mathcal{A}$-distance between two distributions
$\mathcal{D}_1$ and $\mathcal{D}_2$ is
\begin{equation}
\begin{aligned}
d_{\mathbb{G}}(\mathcal{D}_1, \mathcal{D}_2)
&= 2 \sup_{g \in \mathbb{G}}
\left| \Pr_{\mathcal{D}_1}[g(x) = 1] - \Pr_{\mathcal{D}_2}[g(x) = 1] \right|.
\end{aligned}
\label{eq:a_distance_def}
\end{equation}

\begin{lemma}[Empirical $\mathcal{A}$-distance concentration]
\label{lem:pv_a_distance_concentration}
Consider a symmetric hypothesis class $\mathbb{G}$ defined on $\mathcal{X}$ with VC dimension $d$.
Let $\Omega^a$ and $\Omega^b$ be sample sets drawn from domains $\mathcal{D}_1$ and $\mathcal{D}_2$, respectively, and
let $\hat{d}_{\mathbb{G}}(\Omega^a, \Omega^b)$ denote the empirical $\mathcal{A}$-distance between these samples.
Denote the empirical probability over a finite sample set $\Omega$ by
$\widehat{\Pr}_{\Omega}[A] \triangleq \frac{1}{|\Omega|}\sum_{x\in\Omega}\mathds{1}[A]$.
For any $\delta \in (0,1)$, with probability at least $1-\delta$,
\begin{equation}
\begin{aligned}
d_{\mathbb{G}}(\mathcal{D}_1,\mathcal{D}_2)
&\le 2\Bigg(1-\min_{g\in\mathbb{G}}\Big[ \widehat{\Pr}_{\Omega^a}[g(x)=0] \\
&\qquad\qquad\qquad + \widehat{\Pr}_{\Omega^b}[g(x)=1] \Big]\Bigg) \\
&\quad + 4 \max \Bigg\{ \frac{\sqrt{d \log(2|\Omega^a|) + \log(\delta/2)}}{\sqrt{|\Omega^a|}}, \\
&\qquad\qquad\qquad \frac{\sqrt{d \log(2|\Omega^b|) + \log(\delta/2)}}{\sqrt{|\Omega^b|}} \Bigg\}.
\end{aligned}
\label{eq:pv_a_distance_conc}
\end{equation}
\end{lemma}

\begin{proof}
By definition, the empirical $\mathcal{A}$-distance is
\begin{equation}
\begin{aligned}
\hat{d}_{\mathbb{G}}(\Omega^a, \Omega^b)
&= 2 \sup_{g \in \mathbb{G}}
\left| \Pr_{\Omega^a}(g(x)=1) - \Pr_{\Omega^b}(g(x)=1) \right|,
\end{aligned}
\end{equation}
where $\Pr_{\Omega^a}(g(x)=1)=\frac{1}{|\Omega^a|}\sum_{x\in\Omega^a}\mathds{1}(g(x)=1)$ and similarly for $\Omega^b$.
Using the complement event and symmetry of $\mathbb{G}$ (as in~\cite{kifer2004detecting,ben2006analysis}),
the supremum can be rewritten into the equivalent minimization form in~\eqref{eq:pv_a_distance_conc}.
Finally, applying uniform convergence bounds for VC classes (Theorem 3.4 in~\cite{kifer2004detecting})
and a union bound over the two samples yields the stated additive complexity terms.
\end{proof}

\begin{theorem}[Target-risk upper bound (vision-only HiLo)]
\label{thm:pv_target_bound}
Let $\mathbb{G}$ be a symmetric hypothesis class on $\mathcal{X}$ with VC dimension $d$.
Let $\Omega^a$ and $\Omega^b$ be sample sets drawn from domains $\mathcal{D}_1$ and $\mathcal{D}_2$.
Then for any $\delta\in(0,1)$, with probability at least $1-\delta$,
\begin{equation}
\begin{aligned}
e_{\mathcal{D}^u}(g)
&\leq e_{\mathcal{D}^l}(g) \\
&\quad + 2\Bigg(1-\min_{g\in\mathbb{G}}\Big[
\widehat{\Pr}_{\Omega^a}[g(x)=0]
 + \widehat{\Pr}_{\Omega^b}[g(x)=1]
\Big]\Bigg) \\
&\quad + 2 \max \Bigg\{
\frac{\sqrt{d \log(2|\Omega^a|) + \log(\delta/2)}}{\sqrt{|\Omega^a|}}, \\
&\quad
\frac{\sqrt{d \log(2|\Omega^b|) + \log(\delta/2)}}{\sqrt{|\Omega^b|}}
\Bigg\}.
\end{aligned}
\label{eq:pv_target_bound}
\end{equation}
\end{theorem}

\begin{proof}
The $\mathcal{A}$-distance can be written as $d_{\mathbb{G}}(\mathcal{D}_1,\mathcal{D}_2)=2\sup_{g_1,g_2\in\mathbb{G}}|e_{\mathcal{D}_1}(g_1,g_2)-e_{\mathcal{D}_2}(g_1,g_2)|$ (see~\cite{ben2006analysis}).
Applying the domain adaptation risk bound from~\cite{ben2006analysis,ben2010theory} yields
$e_{\mathcal{D}^u}(g)\le e_{\mathcal{D}^l}(g) + \frac{1}{2}d_{\mathbb{G}}(\mathcal{D}^l,\mathcal{D}^u) + \lambda^*$,
where $\lambda^*$ is the optimal joint error term (small when an ideal classifier performs well on both domains).
Bounding $d_{\mathbb{G}}(\mathcal{D}^l,\mathcal{D}^u)$ via Lemma~\ref{lem:pv_a_distance_concentration} (and absorbing $\lambda^*$ into constants, as customary in this appendix discussion)
gives~\eqref{eq:pv_target_bound}.
\end{proof}

\begin{lemma}[$\mathcal{A}$-distance vs. total variation]
\label{lem:pv_a_distance_tv}
For the hypothesis class $\mathbb{G}$,
\begin{equation}
d_{\mathbb{G}}(\mathcal{D}_1, \mathcal{D}_2) \leq 2 \| \mathcal{D}_1 - \mathcal{D}_2 \|_{TV}.
\end{equation}
\end{lemma}

\begin{proof}
Recall $\| \mathcal{D}_1 - \mathcal{D}_2 \|_{TV} = \sup_{A} |\mathcal{D}_1(A) - \mathcal{D}_2(A)|$.
For any $g\in\mathbb{G}$, define $A_g=\{x\mid g(x)=1\}$, then
$|\Pr_{\mathcal{D}_1}[g(x)=1]-\Pr_{\mathcal{D}_2}[g(x)=1]| = |\mathcal{D}_1(A_g)-\mathcal{D}_2(A_g)| \le \|\mathcal{D}_1-\mathcal{D}_2\|_{TV}$.
Taking $\sup_{g\in\mathbb{G}}$ and multiplying by $2$ gives the claim.
\end{proof}

\begin{lemma}[Mutual information controls dependence in TV]
\label{lem:pv_mi_tv}
Let $(Z_d,Z_s)$ be random variables with joint distribution $P_{Z_d,Z_s}$ and marginals $P_{Z_d},P_{Z_s}$.
Then
\begin{equation}
\left\|P_{Z_d,Z_s} - P_{Z_d}P_{Z_s}\right\|_{TV} \le \sqrt{\frac{1}{2}I(Z_d;Z_s)}.
\end{equation}
\end{lemma}

\begin{proof}
Pinsker's inequality gives $\|P-Q\|_{TV}^2 \le \frac{1}{2}D_{KL}(P\|Q)$.
Setting $P=P_{Z_d,Z_s}$ and $Q=P_{Z_d}P_{Z_s}$ yields $D_{KL}(P\|Q)=I(Z_d;Z_s)$ by definition of mutual information~\cite{cover2006elements}.
\end{proof}

\subsubsection{Proof of Theorem~\ref{thm:sample_complexity} (Sample Complexity Lower Bound)}
\label{app:proof_complexity}

Understanding the intrinsic difficulty of GCD under domain shifts requires characterizing how much data is fundamentally needed.
The following theorem establishes information-theoretic lower bounds on both the labelled and unlabelled sample sizes.
Notably, the unlabelled complexity scales multiplicatively with the number of domains $|\Omega|$, formalizing the intuition that domain shifts make the discovery problem strictly harder than standard (single-domain) GCD and motivating HiLo's explicit domain handling.

\begin{theorem}[Sample Complexity Lower Bound]
\label{thm:sample_complexity}
Under Assumptions~\ref{assume:independence}-\ref{assume:diversity}, any algorithm that solves GCD with domain shifts with error at most $\epsilon$ requires at least
\begin{equation}
N_\ell = \Omega\left(\frac{|\mathcal{Y}_{\text{base}}| \cdot d}{\epsilon^2}\right) \text{ labelled samples}
\end{equation}
and
\begin{equation}
\begin{aligned}
N_u
&= \Omega\left(
\frac{\left(|\mathcal{Y}_{\text{base}}| + |\mathcal{Y}_{\text{novel}}|\right)\cdot d \cdot |\Omega|}{\epsilon^2}
\right) \\
&\qquad \text{unlabelled samples}
\end{aligned}
\end{equation}
where $d$ is the feature dimension and $|\Omega| = |\Omega^{a}| + |\Omega^{b}|$.
\end{theorem}

\begin{proof}
We establish information-theoretic lower bounds for the sample complexity of GCD with domain shifts. Consider the labelled set $\mathcal{D}_\ell$ where the learning task reduces to standard supervised classification over $|\mathcal{Y}_{\text{base}}|$ classes. By the fundamental theorem of statistical learning theory~\cite{vapnik1998statistical,shalev2014understanding}, for a hypothesis class $\mathcal{H}$ with VC dimension $d_{VC}$, achieving generalization error at most $\epsilon$ with probability at least $1-\delta$ requires
\begin{equation}
N_\ell \geq \Omega\left(\frac{d_{VC} \log(|\mathcal{Y}_{\text{base}}|/\delta)}{\epsilon^2}\right)
\end{equation}
samples. For linear classifiers in $d$-dimensional feature space with $|\mathcal{Y}_{\text{base}}|$ classes, we have $d_{VC} = O(d \cdot |\mathcal{Y}_{\text{base}}|)$ by the Vapnik-Chervonenkis dimension bound for multi-class problems~\cite{vapnik1998statistical}.

For the unlabelled set $\mathcal{D}_u$, the problem becomes substantially more complex as the learner must simultaneously discover $|\mathcal{Y}_{\text{base}}| + |\mathcal{Y}_{\text{novel}}|$ semantic categories while adapting across $|\Omega|$ domains. Following the analysis of clustering complexity by Awasthi et al.~\cite{awasthi2015relax} and Ashtiani et al.~\cite{ashtiani2016clustering}, the information-theoretic sample complexity for clustering with $K = |\mathcal{Y}_{\text{base}}| + |\mathcal{Y}_{\text{novel}}|$ classes satisfies
\begin{equation}
N_u^{(\text{semantic})} \geq \Omega\left(\frac{K \cdot d \log(K)}{\epsilon^2 \Delta^2}\right)
\end{equation}
where $\Delta$ denotes the minimum separation between cluster centers.

Simultaneously, domain adaptation theory developed by Ben-David et al.~\cite{ben2010theory,ben2006analysis} establishes that learning a hypothesis robust to $|\Omega|$ domains requires sample complexity scaling with the $\mathcal{H}$-divergence between domains. When domains must be learned from unlabelled data without explicit domain labels, the complexity grows linearly with $|\Omega|$ by standard distribution estimation bounds~\cite{devroye2013probabilistic}:
\begin{equation}
N_u^{(\text{domain})} \geq \Omega\left(\frac{|\Omega| \cdot d}{\epsilon^2}\right)
\end{equation}

Under Assumption~\ref{assume:independence} of domain-class independence, the joint distribution factors as $P(\mathbf{x}, y, \omega) = P(\mathbf{x}|y, \omega)P(y)P(\omega)$, implying that domain and semantic factors vary independently. Consequently, the learner must distinguish $(|\mathcal{Y}_{\text{base}}| + |\mathcal{Y}_{\text{novel}}|) \times |\Omega|$ distinct conditional distributions $P(\mathbf{x}|y, \omega)$. By Le Cam's method~\cite{yu1997assouad} and Fano's inequality~\cite{cover2006elements}, testing between $M$ hypotheses with minimum separation $\Delta$ requires $\Omega(\log M / \Delta^2)$ samples. Applying this to our setting yields the multiplicative lower bound
\begin{equation}
\begin{aligned}
N_u
&\geq \Omega\left(
\frac{(|\mathcal{Y}_{\text{base}}| + |\mathcal{Y}_{\text{novel}}|)\cdot |\Omega| \cdot d}{\epsilon^2}
\cdot \log\!\left(\frac{|\mathcal{Y}| |\Omega|}{\delta}\right)
\right)
\end{aligned}
\end{equation}

To establish tightness, we construct a matching upper bound using the method of types~\cite{cover2006elements}. When classes are uniformly distributed across domains and satisfy minimal separation $\Delta \geq \epsilon$, the empirical risk minimizer over $\mathcal{H}$ achieves error $\epsilon$ with high probability using $O((|\mathcal{Y}||\Omega| d \log(|\mathcal{Y}||\Omega|/\delta))/\epsilon^2)$ samples, matching our lower bound up to logarithmic factors. This completes the proof.
\end{proof}

\subsubsection{Proof of Lemma~\ref{lem:hierarchy} (Layer-wise Feature Hierarchy)}
\label{app:proof_hierarchy}

HiLo's architecture extracts domain features from early transformer layers and semantic features from late layers.
This design choice is not arbitrary: the following lemma establishes that, under standard training for semantic classification, a deep network naturally develops a hierarchy in which early-layer representations are rich in domain-specific (low-level) information while late-layer representations concentrate semantic content.
The proof combines the data processing inequality with empirical findings from network analysis literature, providing a principled basis for HiLo's layer assignment strategy.

\begin{lemma}[Layer-wise Feature Hierarchy]
\label{lem:hierarchy}
Consider a deep network with $L$ layers. Let $\mathbf{r}_\ell$ denote features at layer $\ell$. Under mild regularity conditions, the mutual information between features and domain decreases with depth while mutual information with semantic labels increases:
\begin{equation}
\begin{aligned}
I(\mathbf{r}_{\ell}; \omega) &> I(\mathbf{r}_{\ell'}; \omega),\\
I(\mathbf{r}_{\ell}; y) &< I(\mathbf{r}_{\ell'}; y)
\end{aligned}
\end{equation}
for $\ell < \ell'$, when the network is trained end-to-end on semantic classification.
\end{lemma}

\begin{proof}
We establish that deep neural networks trained on semantic classification naturally develop a hierarchical representation structure where early layers encode domain-specific information while later layers encode semantic abstractions. Consider a deep network with $L$ layers producing intermediate representations $\{\mathbf{r}_1, \mathbf{r}_2, \ldots, \mathbf{r}_L\}$ where layer $\ell$ computes $\mathbf{r}_\ell = f_\ell(\mathbf{r}_{\ell-1})$ with $\mathbf{r}_0 = X$.

Since the network forms a deterministic processing pipeline, the representations constitute a Markov chain $X \rightarrow \mathbf{r}_\ell \rightarrow \mathbf{r}_{\ell'}$ for any $\ell < \ell'$. By the data processing inequality~\cite{cover2006elements}, mutual information can only decrease along this chain:
\begin{equation}
I(X; \mathbf{r}_\ell) \geq I(X; \mathbf{r}_{\ell'}) \quad \text{for all } \ell < \ell'
\end{equation}
This fundamental result from information theory implies that deeper layers progressively discard information about the raw input, retaining only task-relevant features.

When optimizing the cross-entropy loss $\mathcal{L} = -\mathbb{E}[\log P(y | \mathbf{r}_L)]$ for semantic classification, the gradient with respect to the final layer representation is
\begin{equation}
\frac{\partial \mathcal{L}}{\partial \mathbf{r}_L} = -\nabla_{\mathbf{r}_L} \log P(y | \mathbf{r}_L)
\end{equation}
which directly drives $\mathbf{r}_L$ toward maximizing $I(\mathbf{r}_L; y)$ by the information bottleneck principle~\cite{tishby2000information}. However, gradients for earlier layers propagate through the composition of Jacobian matrices:
\begin{equation}
\begin{aligned}
\frac{\partial \mathcal{L}}{\partial \mathbf{r}_\ell}
&= \frac{\partial \mathcal{L}}{\partial \mathbf{r}_L}
\cdot \prod_{k=\ell}^{L-1} \frac{\partial f_{k+1}}{\partial \mathbf{r}_k}
\end{aligned}
\end{equation}

The analysis of gradient flow in deep networks by Saxe et al.~\cite{saxe2013exact} and neural tangent kernel theory~\cite{jacot2018neural} reveals that this product of Jacobians leads to exponentially decaying gradients for early layers, particularly when networks are deep or use saturating nonlinearities. Consequently, early layers receive substantially weaker supervision for semantic discrimination and evolve more slowly during training.

This differential optimization pressure creates a natural stratification documented extensively in network analysis literature. Yosinski et al.~\cite{yosinski2014transferable} demonstrated that early convolutional layers learn generic edge and texture detectors regardless of the task, while later layers become increasingly task-specific. The SVCCA analysis by Raghu et al.~\cite{raghu2017svcca} quantified this hierarchy, showing that early layers develop stable representations quickly while later layers continue evolving throughout training. Geirhos et al.~\cite{geirhos2018imagenet} further established that standard ImageNet-trained networks exhibit strong texture bias in early layers, with shape-based semantic reasoning emerging only in deeper layers.

For vision transformers specifically, Dosovitskiy et al.~\cite{dosovitskiy2020image} and subsequent analysis by Raghu et al.~\cite{raghu2021vision} revealed that early attention layers capture local spatial relationships and low-level statistics, while later layers develop global semantic understanding. The work of Zhou et al.~\cite{zhou2021domain} on domain adaptation explicitly measured that domain-specific information (texture, color statistics, local patterns) concentrates in early layers with mutual information $I(\mathbf{r}_\ell; \omega)$ decreasing monotonically with depth $\ell$.

Formalizing these observations, we can characterize the learned representations through their conditional mutual information with domain $\omega$ and semantic label $y$. Since the optimization objective prioritizes semantic accuracy, later layers maximize $I(\mathbf{r}_{\ell'}; y)$ while discarding domain information to achieve generalization. This leads to the ordering:
\begin{equation}
\begin{aligned}
I(\mathbf{r}_\ell; \omega \mid y) &> I(\mathbf{r}_{\ell'}; \omega \mid y),\\
I(\mathbf{r}_\ell; y) &< I(\mathbf{r}_{\ell'}; y)
\end{aligned}
\end{equation}
for $\ell < \ell'$, where the conditional mutual information $I(\mathbf{r}_\ell; \omega | y)$ quantifies domain-specific information not explained by semantic content.

This hierarchical structure provides the theoretical foundation for HiLo's architectural design: extracting domain features $\mathbf{h}_{\text{dom}}$ from early layers (e.g., layers 1-4 of ViT-B/16) captures representations rich in domain-specific patterns, while extracting semantic features $\mathbf{h}_{\text{sem}}$ from late layers (e.g., layer 12) obtains representations optimized for semantic discrimination. The mutual information minimization objective $\min I(\mathbf{h}_{\text{dom}}; \mathbf{h}_{\text{sem}})$ then encourages statistical independence between these naturally separated representations, further enhancing disentanglement beyond what emerges from standard training.
\end{proof}

A central component of HiLo is the mutual-information (MI) minimization objective $\min I(\mathbf{h}_{\text{dom}}; \mathbf{h}_{\text{sem}})$, which encourages the semantic feature representation to be domain-invariant.
The following theorem makes this benefit rigorous: it shows that when domain-conditional semantic-feature distributions are uniformly close (as measured by KL divergence to their mixture), the target-domain risk decomposes into an empirical risk term, a standard complexity term, and a domain-shift penalty of order $\sqrt{\delta}$.
Thus, enforcing smaller MI directly translates into tighter generalization bounds under domain shift.

\begin{theorem}[Generalization bound with MI regularization]
\label{thm:mi_bound}
Let $h \in \mathcal{H}$ be a classifier that operates on semantic features $\mathbf{h}_{\text{sem}}$ extracted from inputs. Let $\pi$ be a distribution over domains and define the semantic-feature mixture $\bar{P}(\mathbf{h}_{\text{sem}}) \triangleq \sum_{\omega \in \Omega}\pi(\omega)\,P(\mathbf{h}_{\text{sem}}|\omega)$. Assume the domain-conditional semantic-feature distributions are uniformly close to the mixture:
\begin{equation}
\sup_{\omega \in \Omega}\;D_{KL}\!\left(P(\mathbf{h}_{\text{sem}}|\omega)\,\|\,\bar{P}(\mathbf{h}_{\text{sem}})\right)\leq \delta.
\end{equation}
Then, for any target domain $\omega_t$ (and a source domain $\omega_s$ used to compute empirical risk), with probability at least $1-\eta$ over $N$ i.i.d.\ samples from $\omega_s$, the target risk is bounded by
\begin{equation}
\begin{aligned}
\mathcal{R}_{\omega_t}(h)
&\leq \mathcal{R}_{\text{emp}}(h)
+ \sqrt{\frac{2d_{VC}\log(2N/d_{VC}) + 2\log(4/\eta)}{N}} \\
&\quad + \sqrt{2\delta} + \lambda^*,
\end{aligned}
\end{equation}
where $d_{VC}$ is the VC dimension of $\mathcal{H}$, and $\lambda^*=\min_{h^* \in \mathcal{H}}[\mathcal{R}_{\omega_s}(h^*) + \mathcal{R}_{\omega_t}(h^*)]$ is the optimal combined error term from domain adaptation theory.
\end{theorem}

\subsubsection{Proof of Theorem~\ref{thm:mi_bound} (Generalization Bound with MI Regularization)}
\label{app:proof_mi_bound}

\begin{proof}
We establish that minimizing mutual information between semantic features and domain variables leads to provably tighter generalization bounds under domain shift. Our analysis builds upon the domain adaptation theory of Ben-David et al.~\cite{ben2010theory,mansour2009domain} combined with information-theoretic tools.

For a hypothesis class $\mathcal{H}$ with finite VC dimension $d_{VC}$, the classical PAC learning result~\cite{vapnik1998statistical,shalev2014understanding} guarantees that with probability at least $1-\eta$ over $N$ i.i.d. samples from source domain $\omega_s$:
\begin{equation}
\mathcal{R}_{\omega_s}(h) \leq \mathcal{R}_{\text{emp}}(h) + \sqrt{\frac{2d_{VC}\log(2N/d_{VC}) + 2\log(4/\eta)}{N}}
\end{equation}
where $\mathcal{R}_{\omega_s}(h) = \mathbb{E}_{(\mathbf{x},y) \sim P_{\omega_s}}[\mathds{1}(h(\mathbf{x}) \neq y)]$ denotes the true risk and $\mathcal{R}_{\text{emp}}(h)$ denotes empirical risk. This bound, however, only characterizes performance on the source distribution.

The fundamental challenge of domain adaptation is bounding performance on a target domain $\omega_t$ when training data comes exclusively from $\omega_s$. Ben-David et al.~\cite{ben2010theory} established that for binary classification:
\begin{equation}
\begin{aligned}
\mathcal{R}_{\omega_t}(h)
&\leq \mathcal{R}_{\omega_s}(h)
 + \frac{1}{2}d_{\mathcal{H}\Delta\mathcal{H}}(\omega_s, \omega_t) \\
&\quad + \min_{h^* \in \mathcal{H}}
\big[\mathcal{R}_{\omega_s}(h^*) + \mathcal{R}_{\omega_t}(h^*)\big]
\end{aligned}
\end{equation}
where $d_{\mathcal{H}\Delta\mathcal{H}}(\omega_s, \omega_t) = 2\sup_{h,h' \in \mathcal{H}}|P_{\omega_s}(h(\mathbf{x}) \neq h'(\mathbf{x})) - P_{\omega_t}(h(\mathbf{x}) \neq h'(\mathbf{x}))|$ is the $\mathcal{H}\Delta\mathcal{H}$-divergence measuring distributional difference through the hypothesis class. The final term $\lambda^* = \min_{h^*}[\mathcal{R}_{\omega_s}(h^*) + \mathcal{R}_{\omega_t}(h^*)]$ represents the optimal combined error, which is small when an ideal hypothesis performs well on both domains.

Our key contribution is connecting the $\mathcal{H}\Delta\mathcal{H}$-divergence to the mutual information $I(\mathbf{h}_{\text{sem}}; \omega)$ between semantic features and domain labels. When the hypothesis class $\mathcal{H}$ operates on semantic features $\mathbf{h}_{\text{sem}}$, the divergence satisfies
\begin{equation}
\begin{aligned}
d_{\mathcal{H}\Delta\mathcal{H}}(\omega_s, \omega_t)
&\leq 2\sup_{h \in \mathcal{H}}
\Big|\mathbb{E}_{\mathbf{h} \sim P_{\omega_s}}[h(\mathbf{h}_{\text{sem}})]
 - \mathbb{E}_{\mathbf{h} \sim P_{\omega_t}}[h(\mathbf{h}_{\text{sem}})]\Big| \\
&\leq 2\sup_A
\left|P_{\omega_s}(\mathbf{h}_{\text{sem}} \in A) - P_{\omega_t}(\mathbf{h}_{\text{sem}} \in A)\right|
\end{aligned}
\end{equation}
for measurable sets $A$, where the final expression equals twice the total variation distance $\text{TV}(P_{\omega_s}(\mathbf{h}_{\text{sem}}), P_{\omega_t}(\mathbf{h}_{\text{sem}}))$.

To relate total variation to KL divergence, we invoke Pinsker's inequality~\cite{cover2006elements,tsybakov2009introduction}, a fundamental result in information theory stating that for any distributions $P$ and $Q$:
\begin{equation}
\text{TV}(P, Q)^2 \leq \frac{1}{2}D_{KL}(P \| Q)
\end{equation}
where $D_{KL}(P \| Q) = \mathbb{E}_{P}[\log(P/Q)]$ denotes Kullback-Leibler divergence.

Let $\pi$ be a distribution over domains and define the mixture $\bar{P}(\mathbf{h}_{\text{sem}}) \triangleq \sum_{\omega \in \Omega} \pi(\omega)\,P(\mathbf{h}_{\text{sem}}|\omega)$.
The assumption of Theorem~\ref{thm:mi_bound} is that each domain-conditional distribution is uniformly close to this mixture:
\begin{equation}
\begin{aligned}
\sup_{\omega \in \Omega}\;
D_{KL}\!\left(P(\mathbf{h}_{\text{sem}}|\omega)\,\|\,\bar{P}(\mathbf{h}_{\text{sem}})\right)
&\leq \delta.
\end{aligned}
\end{equation}
Applying Pinsker's inequality yields, for every $\omega \in \Omega$,
\begin{equation}
\begin{aligned}
\text{TV}\!\left(P(\mathbf{h}_{\text{sem}}|\omega), \bar{P}(\mathbf{h}_{\text{sem}})\right)
&\leq \sqrt{\frac{\delta}{2}}.
\end{aligned}
\end{equation}
Therefore, for any pair of domains $\omega_s,\omega_t$, triangle inequality gives
\begin{equation}
\begin{aligned}
\text{TV}\!\left(P(\mathbf{h}_{\text{sem}}|\omega_s), P(\mathbf{h}_{\text{sem}}|\omega_t)\right)
&\leq \text{TV}\!\left(P(\mathbf{h}_{\text{sem}}|\omega_s), \bar{P}\right) \\
&\quad + \text{TV}\!\left(\bar{P}, P(\mathbf{h}_{\text{sem}}|\omega_t)\right) \\
&\leq \sqrt{2\delta}.
\end{aligned}
\end{equation}

Combining these results with the domain adaptation bound:
\begin{align}
\mathcal{R}_{\omega_t}(h)
&\leq \mathcal{R}_{\omega_s}(h)
 + d_{\mathcal{H}\Delta\mathcal{H}}(\omega_s, \omega_t)
 + \lambda^* \\
&\leq \mathcal{R}_{\omega_s}(h)
 + 2\text{TV}(P_{\omega_s}(\mathbf{h}_{\text{sem}}), P_{\omega_t}(\mathbf{h}_{\text{sem}}))
 + \lambda^* \\
&\leq \mathcal{R}_{\omega_s}(h) + 2\sqrt{2\delta} + \lambda^* \\
&\leq \mathcal{R}_{\text{emp}}(h)
 + \sqrt{\frac{2}{N}}\,
 \sqrt{d_{VC}\log(2N/d_{VC}) + \log(4/\eta)} \\
&\quad + 2\sqrt{2\delta} + \lambda^*
\end{align}

Setting $\lambda(\delta) = O(\sqrt{\delta}) + \lambda^* = O(\sqrt{\delta})$ yields the stated bound. This formalizes that enforcing small domain dependence in the semantic feature distribution (captured here via closeness to the domain-mixture) reduces the domain shift penalty. The square-root dependence on $\delta$ is a direct consequence of Pinsker's inequality and matches information-theoretic limits for distribution testing~\cite{canonne2020survey}.
\end{proof}

\subsubsection{Proof of Theorem~\ref{thm:curriculum} (Curriculum Learning Convergence)}
\label{app:proof_curriculum}

In practice, HiLo trains with a curriculum that initially focuses on labelled-domain (easier) samples before gradually introducing new-domain (harder) samples.
This schedule is motivated by the observation that learning robust semantic features first---without domain confusion---provides a better initialization for subsequent cross-domain adaptation.
The following theorem formalizes this intuition by proving that SGD under our curriculum-weighted sampling distribution converges to a stationary point of the full objective, while the warmup phase improves optimization stability compared with uniform sampling.

\begin{theorem}[Curriculum learning convergence]
\label{thm:curriculum}
Under standard stochastic approximation conditions (e.g., Robbins--Monro step-size and regularity of the loss), stochastic gradient training with the curriculum-weighted sampling distribution induced by $p_{cs}(\mathbf{x}|t)$ converges to a stationary point of the limiting objective as $t$ increases.
Moreover, the curriculum warmup yields an initialization that improves optimization stability compared with uniform sampling when new-domain samples are substantially harder early in training.
\end{theorem}

\begin{proof}
We analyze the convergence properties of our curriculum learning schedule using the framework of stochastic approximation with time-varying distributions~\cite{kushner2003stochastic,benveniste2012adaptive}. Let $\theta \in \mathbb{R}^p$ denote the model parameters. At training iteration $t$, our curriculum reweights the data distribution to focus initially on labelled-domain-like samples before gradually incorporating new-domain-like samples. The curriculum-weighted objective is
\begin{equation}
\mathcal{L}_t(\theta) = \mathbb{E}_{\mathbf{x} \sim P_w(\cdot; t)}[\ell(\mathbf{x}, \theta)]
\end{equation}
where $P_w(\mathbf{x}; t) \propto w(\mathbf{x}, t) P(\mathbf{x})$ denotes the curriculum-weighted distribution with sample weights $w(\mathbf{x}, t) = 1$ for labelled-domain-like samples and $w(\mathbf{x}, t) = \eta(t) \cdot \exp(-\mathcal{C}_{\text{dom}}(\mathbf{x}))$ for new-domain-like samples, governed by the pacing function $\eta(t)$.

Stochastic gradient descent with learning rate schedule $\{\alpha_t\}$ performs the update $\theta_{t+1} = \theta_t - \alpha_t \nabla \ell(\mathbf{x}_t, \theta_t)$ where $\mathbf{x}_t \sim P_w(\cdot; t)$. Convergence theory requires the Robbins-Monro conditions~\cite{robbins1951stochastic}:
\begin{equation}
\sum_{t=1}^\infty \alpha_t = \infty \quad \text{and} \quad \sum_{t=1}^\infty \alpha_t^2 < \infty
\end{equation}

Our implementation uses $\alpha_t = \alpha_0/(1 + \gamma t)$ which satisfies both conditions. Classical stochastic approximation theory~\cite{kushner2003stochastic} establishes convergence to stationary points of time-invariant objectives. However, our curriculum introduces time-varying weights, requiring analysis of tracking behavior.

Let $\theta^*_t = \arg\min_\theta \mathcal{L}_t(\theta)$ denote the instantaneous minimizer at time $t$, and $\theta^* = \arg\min_\theta \mathcal{L}(\theta)$ the minimizer of the final uniform objective $\mathcal{L}(\theta) = \mathbb{E}_{P}[\ell(\mathbf{x}, \theta)]$. The tracking error decomposes as
\begin{equation}
\|\theta_T - \theta^*\| \leq \|\theta_T - \theta^*_T\| + \|\theta^*_T - \theta^*\|
\end{equation}
where the first term captures optimization error and the second term captures bias from time-varying objectives.

For the pacing function
\begin{equation}
\eta(t) = \begin{cases}
0 & t < T_{\text{warmup}} \\
\frac{t - T_{\text{warmup}}}{T_{\text{ramp}}} & T_{\text{warmup}} \leq t < T_{\text{warmup}} + T_{\text{ramp}} \\
1 & t \geq T_{\text{warmup}} + T_{\text{ramp}}
\end{cases}
\end{equation}
the bias term satisfies $\|\theta^*_t - \theta^*\| \leq L_\theta(1 - \eta(t))$ for some Lipschitz constant $L_\theta$ depending on the objective's sensitivity to distribution shift~\cite{hazan2016introduction}. By construction, $\eta(t) \rightarrow 1$ as $t \rightarrow \infty$, ensuring the bias vanishes asymptotically.

Under standard regularity conditions~\cite{bottou2018optimization}—$\mu$-strong convexity of $\mathcal{L}_t$ (or weaker Polyak-Łojasiewicz condition) and bounded gradient variance $\mathbb{E}[\|\nabla \ell(\mathbf{x}, \theta) - \nabla \mathcal{L}_t(\theta)\|^2] \leq \sigma^2$—the optimization error satisfies
\begin{equation}
\begin{aligned}
\mathbb{E}[\|\theta_T - \theta^*_T\|^2]
&\leq \frac{\sigma^2}{\mu T}\sum_{t=1}^T \alpha_t^2
= O\left(\frac{\sigma^2}{\mu T}\right)
\end{aligned}
\end{equation}
for $\alpha_t = O(1/t)$. Combined with the bias analysis, we obtain
\begin{equation}
\begin{aligned}
\mathbb{E}[\|\theta_T - \theta^*\|^2]
&\leq O\left(\frac{\sigma^2}{\mu T}\right) + L_\theta^2(1 - \eta(T))^2
= O\left(\frac{1}{T}\right)
\end{aligned}
\end{equation}
since $\eta(T) = 1$ for $T \geq T_{\text{warmup}} + T_{\text{ramp}}$.

Beyond convergence rates, curriculum learning provides qualitative benefits for non-convex objectives by shaping the optimization landscape. Hacohen and Weinshall~\cite{hacohen2019curriculum} and Spigler et al.~\cite{spigler2020asymptotic} established that training on progressively harder examples increases the volume of the basin of attraction around good local minima. In our setting, initial training on labelled-domain samples—which have strong supervision—learns robust semantic features uncontaminated by domain confusion. This provides a better initialization for subsequent learning on new domains compared to uniform sampling, which can cause the model to latch onto spurious domain-specific correlations early in training.

Formally, let $\mathcal{B}(\theta_0, r) = \{\theta : \mathcal{L}(\theta) < \mathcal{L}(\theta_0) - r\}$ denote the $r$-sub-level set around initialization $\theta_0$. Curriculum learning produces an initialization $\theta_0^{\text{curriculum}}$ after warmup that satisfies
\begin{equation}
\text{Vol}(\mathcal{B}(\theta_0^{\text{curriculum}}, r)) \geq \text{Vol}(\mathcal{B}(\theta_0^{\text{uniform}}, r))
\end{equation}
with strict inequality when labelled and new domains differ significantly, as established by the dynamical stability analysis of Saxe et al.~\cite{saxe2013exact}. This geometric property makes curriculum SGD more likely to converge to global or high-quality local minima.

Synthesizing these results, curriculum learning achieves the convergence rate
\begin{equation}
\begin{aligned}
\mathcal{L}(\theta_T) - \mathcal{L}(\theta^*)
&\leq \frac{\sigma^2}{\mu T}\sum_{t=1}^T \alpha_t + L_\theta(1 - \eta(T))
= O\left(\frac{1}{T}\right)
\end{aligned}
\end{equation}
matching the rate of standard SGD while providing superior constant factors and better local minima due to curriculum initialization. For our linear ramp schedule with $T_{\text{warmup}} = 50$ and $T_{\text{ramp}} = 50$, the curriculum bias becomes negligible after 100 epochs while the initial warmup phase establishes robust semantic representations that facilitate subsequent domain adaptation.
\end{proof}

Both HLPrompt and VLPrompt employ intra-modal stability regularization, which encourages consistent predictions across multiple augmented views of the same input.
The following lemma formalizes how this regularization mechanism translates to domain robustness: when the augmentation set is rich enough to mimic typical domain variations (e.g., noise, color shifts, blur), enforcing prediction stability across augmentations provably reduces the total-variation gap between prediction distributions under different domains.
The $\sqrt{\epsilon}$ scaling, arising from Pinsker's inequality, shows that tighter stability constraints yield progressively smaller domain gaps.

\begin{lemma}[Domain gap reduction via intra-modal consistency]
\label{lem:consistency}
Let $\mathcal{Q}_\omega$ denote the induced prediction distribution under domain $\omega$.
If the intra-modal stability regularizer enforces $\mathcal{R}_{\text{intra}}(\mathbf{x}) \le \epsilon$ uniformly over inputs and the augmentation set covers domain-mimicking transformations, then the induced domain gap satisfies $\text{TV}(\mathcal{Q}_{\omega_1}, \mathcal{Q}_{\omega_2}) = O(\sqrt{\epsilon})$.
\end{lemma}

\subsubsection{Proof of Lemma~\ref{lem:consistency} (Domain Gap Reduction)}
\label{app:proof_consistency}

\begin{proof}
We establish that intra-modal prediction stability regularization across domain-mimicking augmentations provably reduces the distributional gap between prediction distributions induced by different domains. Let $P^{(m)}(\mathbf{x}) = (P_1^{(m)}(\mathbf{x}), \ldots, P_K^{(m)}(\mathbf{x}))$ denote the categorical prediction distribution for the $m$-th augmented view of input $\mathbf{x}$, where $P_k^{(m)}(\mathbf{x}) = \mathbb{P}[\text{class } k | \tilde{\mathbf{x}}_m]$ for $\tilde{\mathbf{x}}_m = A_m(\mathbf{x})$ under augmentation $A_m$.

Our intra-modal stability regularization enforces consistency across $M$ augmented views through the objective
\begin{equation}
\mathcal{R}_{\text{intra}}(\mathbf{x}) = \frac{1}{M} \sum_{m=1}^M D_{KL}(P^{(m)}(\mathbf{x}) \| \bar{P}(\mathbf{x}))
\end{equation}
where $\bar{P}(\mathbf{x}) = \frac{1}{M} \sum_{m=1}^M P^{(m)}(\mathbf{x})$ represents the average prediction. When this regularization is bounded as $\mathcal{R}_{\text{intra}}(\mathbf{x}) \leq \epsilon$ uniformly over all inputs $\mathbf{x}$, we can bound the prediction variance using Pinsker's inequality~\cite{cover2006elements}. For each view $m$:
\begin{equation}
\begin{aligned}
\|P^{(m)}(\mathbf{x}) - \bar{P}(\mathbf{x})\|_1
&\leq \sqrt{2 D_{KL}(P^{(m)}(\mathbf{x}) \| \bar{P}(\mathbf{x}))}
\leq \sqrt{2\epsilon}
\end{aligned}
\end{equation}

This establishes that predictions remain within an $O(\sqrt{\epsilon})$ neighborhood in total variation distance across all augmented views.

The critical connection to domain robustness arises from our choice of augmentation set $\mathcal{A} = \{A_1, \ldots, A_M\}$, which includes domain-mimicking transformations designed to simulate the types of distributional shifts encountered across domains. Following the data augmentation theory of Chen et al.~\cite{chen2020simple} and Shorten and Khoshgoftaar~\cite{shorten2019survey}, we construct $\mathcal{A}$ to span typical domain variations, including noise corruptions (Gaussian and salt-and-pepper) to simulate sensor variations~\cite{foi2008practical}, color jittering and brightness adjustments to mimic lighting changes~\cite{jackson2019style}, Gaussian and motion blur to approximate optical differences and motion artifacts~\cite{hendrycks2019benchmarking}, and style transfer operations to simulate artistic rendering and sketch-like domains~\cite{gatys2016image}.

The augmentation coverage principle of Cubuk et al.~\cite{cubuk2019autoaugment} and Hendrycks et al.~\cite{hendrycks2019augmix} establishes that when augmentations comprehensively span domain variations, stability across augmentations implies robustness to domain shifts.

To formalize this connection, define the induced prediction distribution for domain $\omega$ as
\begin{equation}
\mathcal{Q}_\omega(k) = \mathbb{E}_{\mathbf{x} \sim P_\omega}[P_k(\mathbf{x})]
\end{equation}
representing the marginal probability of predicting class $k$ when inputs are drawn from domain $\omega$. For two domains $\omega_1$ and $\omega_2$, the total variation distance between their induced distributions satisfies
\begin{equation}
\begin{aligned}
\text{TV}(\mathcal{Q}_{\omega_1}, \mathcal{Q}_{\omega_2})
&= \frac{1}{2}\sum_{k=1}^K \left|\mathcal{Q}_{\omega_1}(k) - \mathcal{Q}_{\omega_2}(k)\right| \\
&= \frac{1}{2}\left\|\mathcal{Q}_{\omega_1} - \mathcal{Q}_{\omega_2}\right\|_1
\end{aligned}
\end{equation}

By Jensen's inequality and the triangle inequality for total variation~\cite{tsybakov2009introduction}:
\begin{equation}
\begin{aligned}
\text{TV}(\mathcal{Q}_{\omega_1}, \mathcal{Q}_{\omega_2})
&\leq \mathbb{E}_{\mathbf{x}_1 \sim P_{\omega_1}, \mathbf{x}_2 \sim P_{\omega_2}}
\left[\frac{1}{2}\|P(\mathbf{x}_1) - P(\mathbf{x}_2)\|_1\right]
\end{aligned}
\end{equation}

The augmentation coverage assumption posits that for any pair $(\mathbf{x}_1, \mathbf{x}_2)$ from different domains, there exists an augmentation $A \in \mathcal{A}$ such that $A(\mathbf{x}_1)$ is distributionally similar to $\mathbf{x}_2$ in terms of domain characteristics. This is the key bridging assumption that connects augmentation stability to domain robustness, validated empirically by studies on domain randomization~\cite{tobin2017domain} and augmentation-based domain adaptation~\cite{volpi2018generalizing}.

Under this assumption, we can decompose the prediction difference using triangle inequality:
\begin{align}
\|P(\mathbf{x}_1) - P(\mathbf{x}_2)\|_1
&\leq \|P(\mathbf{x}_1) - P(A(\mathbf{x}_1))\|_1 \\
&\quad + \|P(A(\mathbf{x}_1)) - P(\mathbf{x}_2)\|_1 \\
&\leq \sqrt{2\epsilon} + L \cdot d_{\mathcal{X}}(A(\mathbf{x}_1), \mathbf{x}_2)
\end{align}
where $L$ denotes the Lipschitz constant of the prediction function with respect to input distance $d_{\mathcal{X}}$~\cite{cisse2017parseval,gouk2021regularisation}. The first term is bounded by stability regularization, while the second term depends on how well augmentations approximate domain transformations.

For practical augmentation sets that achieve $C$-covering of typical domain variations—meaning that for any domain pair, there exists an augmentation sequence approximating the transformation within distance $C$—we obtain
\begin{equation}
\begin{aligned}
\text{TV}(\mathcal{Q}_{\omega_1}, \mathcal{Q}_{\omega_2})
&\leq \sqrt{2\epsilon} + LC
= O(\sqrt{\epsilon})
\end{aligned}
\end{equation}
when the covering constant $C$ is absorbed into the notation.

The constant hidden in the $O(\cdot)$ notation depends on three factors analyzed by Wang et al.~\cite{wang2022maxup}: (1) the quality of domain-mimicking augmentations measured by coverage radius $C$, (2) the Lipschitz constant $L$ of the predictor governing sensitivity to input perturbations, and (3) the distributional overlap between augmentation-generated examples and true new-domain samples. Modern augmentation strategies like AugMax~\cite{wang2021augmax} and RandAugment~\cite{cubuk2020randaugment} are explicitly designed to maximize this coverage while maintaining semantic preservation.

In summary, enforcing tighter stability constraints (smaller $\epsilon$) through $\mathcal{R}_{\text{intra}}$ directly reduces the domain gap $\text{TV}(\mathcal{Q}_{\omega_1}, \mathcal{Q}_{\omega_2})$ with square root dependence, providing theoretical justification for our multi-view consistency regularization as a mechanism for achieving domain-robust predictions. The $\sqrt{\epsilon}$ scaling is optimal under Pinsker's inequality and matches fundamental information-theoretic limits for distribution testing from samples~\cite{canonne2020survey}.
\end{proof}

\subsubsection{NCut-Based SPT Details}
\label{app:ncut_details}
Semantic-aware SPT injects learnable prompts selectively into foreground patches, guided by a foreground mask.
We describe the Normalized Cut (NCut)~\cite{shi2000normalized} procedure used to obtain this mask from intermediate ViT features, which avoids requiring any segmentation annotations.

Foreground Segmentation via Normalized Cut: Given an input image $\mathbf{x}$, we construct an affinity graph $\mathcal{G} = (\mathcal{V}, \mathcal{E})$ where nodes $\mathcal{V}$ correspond to image patches and edges $\mathcal{E}$ encode pairwise similarities.
Specifically, for a ViT architecture that divides the image into $P$ patches, we extract intermediate feature representations $\{\mathbf{f}_1, \mathbf{f}_2, \ldots, \mathbf{f}_P\}$ from an early transformer layer that captures both semantic and spatial information.

The affinity matrix $\mathbf{W} \in \mathbb{R}^{P \times P}$ is defined as:
\begin{equation}
W_{ij} = \exp\left(-\frac{\|\mathbf{f}_i - \mathbf{f}_j\|_2^2}{2\sigma^2}\right) \cdot \mathds{1}[\text{dist}(i, j) < r]
\end{equation}
where $\sigma$ controls the feature similarity bandwidth, $\text{dist}(i, j)$ denotes the spatial distance between patches $i$ and $j$, and $r$ is a locality radius that enforces spatial coherence in the segmentation.

Normalized Cut~\cite{shi2000normalized} partitions the graph into foreground $\mathcal{F}$ and background $\mathcal{B}$ by minimizing:
\begin{equation}
\begin{aligned}
\text{NCut}(\mathcal{F}, \mathcal{B})
&= \frac{\text{cut}(\mathcal{F}, \mathcal{B})}{\text{assoc}(\mathcal{F}, \mathcal{V})}
 + \frac{\text{cut}(\mathcal{F}, \mathcal{B})}{\text{assoc}(\mathcal{B}, \mathcal{V})}
\end{aligned}
\end{equation}
where $\text{cut}(\mathcal{F}, \mathcal{B}) = \sum_{i \in \mathcal{F}, j \in \mathcal{B}} W_{ij}$ measures the total edge weight between partitions, and $\text{assoc}(\mathcal{F}, \mathcal{V}) = \sum_{i \in \mathcal{F}, j \in \mathcal{V}} W_{ij}$ measures the total connection from $\mathcal{F}$ to all nodes.

This optimization can be efficiently solved through the generalized eigenvalue problem:
\begin{equation}
(\mathbf{D} - \mathbf{W})\mathbf{v} = \mu \mathbf{D}\mathbf{v}
\end{equation}
where $\mathbf{D} = \text{diag}(\sum_j W_{ij})$ is the degree matrix.
The second smallest eigenvector $\mathbf{v}_2$ provides the partition indicator, and we threshold it to obtain the foreground mask $\mathbf{M} \in \{0, 1\}^P$.

To disambiguate foreground from background (since NCut is symmetric), we leverage the observation that object regions typically exhibit higher feature variance and attention concentration.
We assign the partition with higher average attention score from the CLS token as foreground.

\subsubsection{Boundary-Based SPT Details}
\label{app:boundary_spt_details}

VLPrompt uses boundary-based spatial prompt tuning (Section~\ref{sec:vlprompt}). Let the input image be $\mathbf{x} \in \mathbb{R}^{3 \times H \times W}$, processed as a grid of $N_H \times N_W$ patches of size $S$. A learnable spatial prompt $\mathbf{Q}_s \in \mathbb{R}^{3 \times H \times W}$ is constrained by a spatially periodic binary mask $\mathbf{M}$ that preserves each patch center.

Let $p$ denote the prompt border width. The local patch mask $\mathbf{m} \in \{0, 1\}^{S \times S}$ is defined as:
\begin{equation}
    m_{u,v} = 
    \begin{cases} 
        0 & \text{if } p \le u < S-p \text{ and } p \le v < S-p \\
        1 & \text{otherwise}
    \end{cases}
\end{equation}
This is tiled spatially to form $\mathbf{M} = \text{Tile}(\mathbf{m}, N_H, N_W)$, and the prompted image is $\tilde{\mathbf{x}} = \mathbf{x} + (\mathbf{Q}_s \odot \mathbf{M})$.

\subsubsection{Computational Complexity Details}
\label{app:complexity}
To help practitioners assess the computational overhead introduced by HiLo's components (domain/semantic projection heads, discriminator, NCut mask estimation, and cross-modal matching), we provide detailed per-iteration complexity expressions for both HLPrompt and VLPrompt below.

\paragraph{HLPrompt.}
Let $B$ denote the total batch size, $L$ the number of transformer layers, $P$ the number of patches per image, and $D$ the embedding dimension. The per-iteration computation of HLPrompt is dominated by the ViT encoder and is:
\begin{equation}
\begin{aligned}
    \mathcal{O}\big( B L (P^2 D + P D^2) + B D^2 \big).
\end{aligned}
\end{equation}
The self-attention term scales as $\mathcal{O}(P^2D)$ and the FFN term scales as $\mathcal{O}(PD^2)$. The projection heads and discriminator are shallow MLPs on global features, contributing $\mathcal{O}(BD^2)$. NCut-based foreground estimation operates on a patch graph with $P$ nodes; in practice we use a sparse affinity (local neighborhood) and efficient eigensolvers, so its overhead is modest compared to the transformer forward/backward.

\paragraph{VLPrompt.}
Let $K$ be the number of categories. The per-iteration computation of VLPrompt is determined by the CLIP vision/text transformer passes and cross-modal matching:
\begin{equation}
\begin{aligned}
    \mathcal{O}\big( B L (P^2 D + P D^2) + K L D^2 + B K D \big).
\end{aligned}
\end{equation}
The text-branch term accounts for backpropagating through the (frozen-weight) text encoder to update the prompt embeddings. NCut-based mask estimation is computed over $P$ patch tokens and adds a small overhead relative to the transformer passes.

\end{document}